\documentclass[english]{article}

\usepackage{geometry}
\usepackage[T1]{fontenc}
\usepackage[utf8]{inputenc}
\usepackage{bm}
\usepackage{graphicx}
\graphicspath{{../imgs/}}
\usepackage{tikz}
\usetikzlibrary{positioning, arrows.meta, calc}
\tikzset{
  box/.style={
    draw,
    rounded corners,
    minimum width=2cm,
    minimum height=1cm,
    align=center,
  },
  arrow/.style={
    ->,
    >=Latex,
    thick,
    draw=blue,
  },
  dashed-line/.style={
    dashed,
    thick,
    draw=blue,
  }
}

\usepackage{caption}
\usepackage{subcaption}
\usepackage{amssymb}
\usepackage{mathrsfs}
\usepackage{amsmath}
\usepackage{dsfont}
\usepackage{extarrows}
\usepackage{enumitem}
\usepackage{booktabs}
\usepackage{makecell}
\usepackage{tablefootnote}
\usepackage[linesnumbered,ruled,vlined]{algorithm2e}
\usepackage{xcolor}

\SetCommentSty{mycommfont}

\SetKwInput{KwInput}{Input}                
\SetKwInput{KwOutput}{Output}              

\usepackage[unicode=true,
 bookmarks=false,
 breaklinks=false,pdfborder={0 0 1},colorlinks=false]
 {hyperref}
\hypersetup{
 colorlinks,citecolor=blue,filecolor=blue,linkcolor=blue,urlcolor=blue}

\makeatletter
\usepackage{amsthm}
\usepackage{comment}
\usepackage{natbib}

\graphicspath{ {./imgs/} }

\newtheorem{theorem}{Theorem}[section]
\newtheorem{proposition}[theorem]{Proposition}
\newtheorem{lemma}[theorem]{Lemma}
\newtheorem{corollary}[theorem]{Corollary}
\newtheorem{assumption}[theorem]{Assumption}

\newtheorem{remark}[theorem]{Remark}

\newcounter{theoremctr}
\newcounter{corollaryctr}
\newcounter{assumptionctr}
\newcounter{propositionctr}
\newcounter{lemmactr}

\renewcommand{\thetheoremctr}{\arabic{theoremctr}}     
\renewcommand{\thecorollaryctr}{\arabic{corollaryctr}}  
\renewcommand{\theassumptionctr}{\arabic{assumptionctr}} 
\renewcommand{\thepropositionctr}{\arabic{propositionctr}} 
\renewcommand{\thelemmactr}{\arabic{lemmactr}}

\newcommand{\tmix}{t_{\mathsf{mix}}} 
\newcommand{\mix}{\mathsf{mix}}

\allowdisplaybreaks

\let\bar\overline
\let\tilde\widetilde

\newcommand{\red}[1]{\textcolor{red}{#1}}
\newcommand{\blue}[1]{\textcolor{blue}{#1}}

\newcommand{\yuting}[1]{\red{Yuting: #1}}

\newcommand{\weichen}[1]{\blue{Weichen: #1}}

\makeatletter
\let\old@paragraph\paragraph
\renewcommand{\paragraph}[1]{\old@paragraph{#1.}}
\makeatother

\title{Uncertainty quantification for Markov chain induced martingales \\ with application to temporal difference learning}

\author{
  Weichen Wu\thanks{The Voleon Group, New York, NY 10010 USA.}\\
  \and
  Yuting Wei\thanks{Department of Statistics and Data Science, The Wharton School, University of Pennsylvania, Philadelphia, PA 19104, USA.}  \\
  \and 
  Alessandro Rinaldo\thanks{Department of Statistics and Data Sciences, University of Texas, Austin TX 78705, USA.}
  } 
\date{\today}

\begin{document}

\maketitle

\begin{abstract}
We establish novel and general high-dimensional concentration inequalities and Berry-Esseen bounds for vector-valued martingales induced by Markov chains. We apply these results to analyze the performance of the Temporal Difference (TD) learning algorithm with linear function approximations, a widely used method for policy evaluation in Reinforcement Learning (RL), obtaining a sharp high-probability consistency guarantee that matches the asymptotic variance up to logarithmic factors. Furthermore, we establish an $O(T^{-\frac{1}{4}}\log T)$ distributional convergence rate for the Gaussian approximation of the TD estimator, measured in convex distance. Our martingale bounds are of broad applicability, and our analysis of TD learning provides new insights into statistical inference for RL algorithms, bridging gaps between classical stochastic approximation theory and modern RL applications.
\end{abstract}

\setcounter{tocdepth}{2}
\tableofcontents

\section{Introduction}




Markov chains are important tools in statistical machine learning for modeling dependent data.
They provide a theoretical framework for analyzing sequential algorithms, are fundamental to MCMC sampling, Hidden Markov Models, and Reinforcement Learning (RL), and are widely used in a multitude of high-stakes applications, including NLP, finance, biology, and AI systems.

Given their widespread use, it is crucial to provide rigorous probabilistic guarantees on the convergence, stability, and error bounds of Markovian sequences. Indeed, uncertainty quantification is essential for assessing the reliability, robustness, and generalization power of machine learning models built on Markov chains. For example, in MCMC sampling, precise uncertainty quantification ensures reliable convergence diagnostics and variance estimation, preventing misleading inferences in Bayesian models. In RL, uncertainty estimation for value functions helps balance exploration and exploitation, leading to more stable decision-making. Similarly, understanding uncertainty in reverse stochastic processes improves sample quality and diversity in diffusion-based generative models.

In the first part of the paper, we derive novel concentration and high-dimensional Berry--Esseen bounds for a certain type of bounded vector-valued martingale difference sequences induced by (not necessarily reversible) Markov chains, and illustrate the broad applicability of this data dependency structure. While there is a rich literature on concentration bounds for vector- and matrix-valued additive functionals of Markov chains \citep{garg2018matrixexpanderchernoff,qiu2020matrix,neeman2024concentration}, we are in fact not aware of any high-dimensional Berry--Esseen bounds for data of this form. Recently,  \cite{heejong.m.dependence} have studied high-dimensional Gaussian approximations for $m$-dependent sequences over the class of hyper-rectangles, though those bounds and methods do not directly apply to martingale sequences.

In the second part of the paper, we turn to the study of uncertainty quantification for the temporal difference (TD) learning algorithm, a widely used method for value function estimation in RL \citep{sutton2018reinforcement}, assuming Markovian data. TD learning is an instance of stochastic approximation \citep{robbins1951stochastic}, designed to solve fixed-point equations via randomized approximations of residuals. In recent years, largely motivated by the diffusion and success of RL applications, there have been significant advancements in statistical inference techniques for Markov chain-based algorithms in RL; see, e.g., \cite{bhandari2018finite,mou2020linear,Fan2021Hoeffding,li2021sample,samsonov2023finitesample,samsonov2024gaussian,srikant2024rates}. In this paper we obtain novel and sharp bounds that improve on the current state of the art. 





\paragraph{Summary of the contributions}
We make two types of theoretical contributions. First, we derive new high-dimensional finite sample approximations to Markov data and martingale processes in discrete time that are of broad applicability. Specifically,
\begin{itemize}
    \setlength{\itemsep}{0pt}
    

\item In Section \ref{sec:settings}, we introduce the notion of {\it Markov chain induced martingales}, which we argue are naturally suited to analyze the concentration and Gaussian approximation properties of additive functionals of Markov chain data. Then in Section \ref{sec:MC-mtg}, we obtain a Bernstein-type inequality for bounded Markov chain induced martingales in Theorem \ref{thm:matrix-bernstein-mtg}, as well as provide a high-dimensional Gaussian approximation bound measured by convex distance in Theorem \ref{thm:Berry--Esseen-mtg}. To the best of our knowledge, both results are novel. 

\item In Section \ref{sec:MC-Berry--Esseen}, we present new concentration and Gaussian approximation bounds for bounded martingale difference sequences that we require to handle Markov chain induced martingales and may be of independent interest. In particular, in Theorem \ref{thm:Srikant-generalize} and Corollary \ref{cor:Wu}, we obtain novel high-dimensional  Berry--Esseen bounds in Wasserstein distance for normalized sums of discrete time vector-valued martingales with deterministic variance exhibiting dependence on the sample size $n$ of order $O(n^{-\frac{1}{2}}\log n)$. We establish this bound by refining and extending a proof strategy recently put forward by \cite{srikant2024rates}. 

\end{itemize}

Secondly, we establish rates of consistency and distributional approximations to the output of the TD learning algorithm (with Polyak-Ruppert averaging and polynomially vanishing step-sizes), arguably the most popular methodology for policy evaluation using linear approximations in reinforcement learning, with Markovian sequences of length $T$. Specifically, 
\begin{itemize}
\setlength{\itemsep}{0pt}
\item Theorem \ref{thm:TD-whp} gives a high-probability bound on the Euclidean norm of the estimation error of the TD learning algorithm featuring a sample dependence of order $T^{-1/2}$, up to log factors, and an optimal dependence on the asymptotic variance;
\item Theorem \ref{thm:TD-Berry--Esseen} provides a high-dimensional Berry--Esseen bound  for the TD estimator in the convex distance with a sample dependence of order $O(T^{-\frac{1}{4}}\log T)$. 
\end{itemize}
Both results are novel and do not have direct counterparts in the RL literature, which has, for the most part, focused on independent samples; see Section \ref{sec:TD} for a discussion of relevant work. 

\paragraph{Notation}
Throughout the paper, we use boldface small letters to denote vectors and boldface capital letters to denote matrices. For any vector $\bm{x}$, we let $\|\bm{x}\|_2$ be its $L_2$ norm; for any matrix $\bm{M}$, $\mathsf{Tr}(\bm{M})$ denotes its trace, $\mathsf{det}(\bm{M})$ its determinant, $\|\bm{M}\|$ its spectral norm (i.e., the largest singular value), and $\|\bm{M}\|_{\mathsf{F}}$ its Frobenius norm, i.e., $\|\bm{M}\|_{\mathsf{F}}=\sqrt{\mathsf{Tr}(\bm{M}^\top \bm{M})}$.  For any $\bm{M} \in \mathbb{S}^{d \times d}$, the set of all $d \times d$ real symmetric matrices, we use $\lambda_{\max}(\bm{M}) = \lambda_1(\bm{M}) \geq \lambda_2(\bm{M}) \geq \ldots \geq \lambda_d(\bm{M}) = \lambda_{\min}(\bm{M})$ to indicate its eigenvalues.
For sequences $\{f_t\}_{t \in \mathbb{N}}$ and $\{g_t\}_{t \in \mathbb{N}}$ of numbers, we write $f_t \lesssim g_t$, or $f_t = O(g_t)$, to signify that there exists a universal constant $C > 0$  such that $f_t \leq C g_t$ for all $t$. We will use the notation $\tilde{C}$ to express a quantity independent of $T$, the number of iterations/sample size, but possibly dependent on other problem-related parameters. 

For two measures $\mu,\nu$ on the same measurable space $(\mathcal{X},\mathscr{F})$, if $\nu$ is absolutely continuous with respect to $\mu$, we use $\frac{\mathrm{d}\nu}{\mathrm{d}\mu}$ to denote its Radon--Nykodym derivative with respect to $\mu$. 
For any measurable function $f: \mathcal{X} \to \mathbb{R}$ and any number $p > 1$, we set $\|f\|_{\mu,p}^p:=\int_{x \in \mathcal{X}}|f(x)|^p \mathrm{d}\mu(x)$.
This definition can also be extended to the case of $p = \infty$, by defining $\|f\|_{\mu,\infty} = \text{ess} \sup |f|,$ the essential supremum of $|f|$ with respect to $\mu$. To simplify the notation, when $p = 2$, we write $\|f\|_{\mu}:=\|f\|_{\mu,2}$.

We will consider the following measure of distance between two probability distributions $P,Q$ on $\mathbb{R}^d$: the \emph{total variation distance} $d_{\mathsf{TV}(P,Q)}:= \sup_{\mathcal{A} \in \mathscr{B}_d} |P(\mathcal{A}) - Q(\mathcal{A})|$, where $\mathscr{B}_d$ represents the class of all Borel sets in $\mathbb{R}^d$; the \emph{convex distance} $d_{\mathsf{C}}(P,Q) := \sup_{\mathcal{A} \in \mathscr{C}_d} |P(\mathcal{A}) - Q(\mathcal{A})|$, where  $\mathscr{C}_d$ represents the class of all convex subsets of $\mathbb{R}^d$; and the \emph{Wasserstein distance} $d_{\mathsf{W}}(P,Q) := \sup_{h \in \mathsf{Lip}_1}|\mathbb{E}_{\bm{x} \sim P}[h(\bm{x})] - \mathbb{E}_{\bm{x} \sim Q}[h(\bm{x})]|$, where $\mathsf{Lip}_1$ represents the class of all 1-Lipschitz functions from $\mathbb{R}^d$ to $\mathbb{R}$. 

\section{
Markov chain induced martingales}\label{sec:settings}

We study high-dimensional concentration and Gaussian approximations for certain martingale difference sequences that are naturally associated with Markovian sequences, which here we refer to as {\it Markov chain induced martingales,} defined next. Although our results provide novel tools of broad applicability to the analysis of Markov chain data, they are especially useful in TD learning modeling; see Section \ref{sec:TD}.
Below, we briefly introduce the problem setting and main assumptions and formalize the notion of Markov chain induced martingales.

Throughout this paper, we consider a discrete-time Markov chain $\{s_t\}_{t=0,1,2,\ldots}$ on some (sufficiently well-behaved) state space $\mathcal{S}$ satisfying the following assumptions. We refer the reader to Appendix \ref{app:MC-basics} for background and details on the quantities introduced below.
\begin{assumption}\label{as:nu}
\begin{enumerate}
\item The transition kernel $P$ admits a \emph{unique stationary distribution}, denoted as $\mu$.\footnote{It can be guaranteed that any positive recurrent, irreducible and aperiodic Markov chain fulfills this property.}
\item The transition kernel $P$ has a \emph{spectral gap} $1-\lambda > 0$.
\item The initial state $s_0$ is drawn from a probability distribution $\nu$ that is absolutely continuous with respect to the stationary distribution $\mu$; furthermore, there exists $p \in (1,\infty]$, such that $\left\|\frac{\mathrm{d}\nu}{\mathrm{d}\mu}\right\|_{\mu,p} < \infty$, where $\frac{\mathrm{d}\nu}{\mathrm{d}\mu}$ is the Radon--Nykodym derivative of $\nu$ with respect to $\mu$ and $\| \cdot \|_{\mu,p}$ denotes the $L_p(\mu)$ norm.
\end{enumerate}
\end{assumption}
Letting $q \in [1,\infty)$ denote the conjugate of $p$, i.e. $\frac{1}{p} + \frac{1}{q} = 1$ when $p < \infty$ and  $q = 1$ if $p = \infty$,  we say that $\nu$ \emph{satisfies Assumption \ref{as:nu} with parameters $(p,q)$} if these conditions hold true.

\begin{remark}
   In our analysis, we do \emph{not} require the Markov chain to be reversible.
\end{remark}

We will study functions $\bm{f}:\mathcal{S}^2 \to \mathbb{R}^{d}$ of Markov chains that satisfy the following condition. 
\begin{assumption}\label{as:markov-mtg}
For every $s \in \mathcal{S}$, $\mathbb{E}_{s' \sim P(\cdot \mid s)}\bm{f}(s,s') = \bm{0}$ and $\mathbb{E}_{s' \sim P(\cdot|s)}[\|\bm{f}(s,s')\|_2^2] < \infty$.
\end{assumption}

In this paper, we work with {\it bounded} measurable functions $\bm{f}$, which automatically satisfy the second moment condition in the above assumption. 
It is immediate to see that, for a sequence of functions (possibly the same one) $\{\bm{f}_i\}_{1 \leq i \leq n}$ each satisfying Assumption \ref{as:markov-mtg} and a random sequence $\{s_i\}_{i=0}^n$ drawn from the Markov chain, the $d$-dimensional random vectors 
\begin{equation}\label{eq:MC-martingale}
\bm{f}_1(s_{0},s_1), \bm{f}_2(s_{1},s_2), \ldots, \bm{f}_n(s_{n-1},s_n)
\end{equation}
form a martingale difference sequence with respect to the natural filtration of the data $\{ \mathcal{F}_{i} \}_{i=0}^n$, i.e.,
\begin{align*}
\mathbb{E}[\bm{f_i}(s_{i-1},s_i)| \mathcal{F}_{i-1}] = \bm{0}\quad \text{for all }i=1,\ldots,n.
\end{align*}
We call this type of martingale {\it a Markov chain induced martingale.}

We remark that Markov chain induced martingales are well-known quantities that arise naturally when studying concentration properties and Gaussian approximations for additive functionals of discrete Markov chains.
In detail, suppose that we are interested in analyzing the sum
\begin{align}\label{eq:partial-sum}
\sum_{i=1}^n \bm{g}(s_i), \qquad \text{for }
\bm{g} \colon \mathcal{S} \rightarrow \mathbb{R}^d,
\end{align}
where the random sequence $\{s_i\}_{i=1}^n$ is a draw from a Markov chain and $\bm{g}$ satisfies $\mathbb{E}_{s \sim \mu}\mathbb{E}[\bm{g}(s)] = \bm{0}$. 
We recall that the initial state $s_0$  of the Markov chain need not be a draw from the stationary distribution. A widely used approach \citep{
GlynnPeterW.2024SRfP,Makowski2002,li2023online,srikant2024rates} is to relate the above sum to a sum of terms from a martingale difference via {\it the Poisson equation} 
\begin{align}\label{eq:poisson.eq}
\bm{g}(s) = \bm{U}(s)- \mathcal{P}\bm{U}(s), \quad  s \in \mathcal{S},
\end{align}
holding for $\mu$-almost all $s \in \mathcal{S}$,
where the operator $\mathcal{P}$ maps the function $\bm{U} \colon \mathcal{S} \rightarrow \mathbb{R}^d$ to the function $s \in \mathcal{S} \mapsto \mathcal{P}\bm{U}(s) = \mathbb{E}_{s' \sim P(\cdot \mid s)}[\bm{U}(s')]$; see Appendix \ref{app:MC-basics} for details. 
Though there can be multiple solutions to the Poisson equation, it is well known that if the function 
$s \in \mathcal{S} \mapsto \sum_{k=0}^\infty \mathcal{P}^k \bm{g}(s)$  is integrable with respect to the stationary distribution, then it is a solution $\bm{U}$ to \eqref{eq:poisson.eq}; see, e.g., \cite{MC-book}. This is indeed the approach we will follow in Section \ref{sec:TD} when we analyze the properties of the TD learning algorithm. 


Given a solution $\bm{U}$ to the Poisson equation, the partial sum \eqref{eq:partial-sum} can then be expanded as
\begin{align}\label{eq:partial-sum-decompose}
\sum_{i=1}^n \bm{g}(s_i)&= \sum_{i=1}^n \{\bm{U}(s_i) - \mathbb{E}[\bm{U}(s_{i+1})\mid s_i]\} \nonumber \\ 
&= \sum_{i=1}^n \{\bm{U}(s_i) - \mathbb{E}[\bm{U}(s_{i})\mid s_{i-1}]\} + \sum_{i=1}^n \{\mathbb{E}[\bm{U}(s_{i})\mid s_{i-1}] - \mathbb{E}[\bm{U}(s_{i+1})\mid s_i]\} \nonumber \\ 
&= \sum_{i=1}^n \{\bm{U}(s_i) - \mathbb{E}[\bm{U}(s_{i})\mid s_{i-1}]\} + \mathbb{E}[\bm{U}(s_1)\mid s_0] - \mathbb{E}[\bm{U}(s_{n+1})\mid s_n]
\end{align}
where the second identity follows from a reordering of the summands and the last line from telescoping. Thus, by defining 
\begin{align*}
\bm{f}(s,s'):= \bm{U}(s') - \mathcal{P}\bm{U}(s), \quad  s,s' \in \mathcal{S},
\end{align*}
we can easily verify that $\bm{f}$ satisfies Assumption \eqref{as:markov-mtg} and that the first term of \eqref{eq:partial-sum-decompose} can be represented as
\begin{align*}
\sum_{i=1}^n \{\bm{U}(s_i) - \mathbb{E}[\bm{U}(s_{i})\mid s_{i-1}]\} = \sum_{i=1}^n \bm{f}(s_{i-1},s_i),
\end{align*}
which is a sum of the terms forming a Markov chain induced martingale \eqref{eq:MC-martingale}. 
Thus, the task of establishing concentration and Gaussian approximation properties of the sum \eqref{eq:partial-sum} largely reduces to that of analyzing the sum of a martingale difference sequence. This is the main task we undertake in the first half of the paper.

\section{Main results}
In this section, we describe our main results. 
In Section \ref{sec:MC-mtg} we present a novel Bernstein-style concentration inequality and a new Berry--Esseen bound over the class of convex sets in $\mathbb{R}^d$ for Markov chain induced  martingales.  Section \ref{sec:MC-Berry--Esseen} presents general Gaussian approximation and concentration bounds for vector-valued martingales (not necessarily Markov chain induced  martingales) that are needed in our analysis and that may be of independent interest.

\subsection{Bounds for Markov chain induced martingales}\label{sec:MC-mtg}
In our first result, we derive a novel Bernstein-style concentration inequality for the Euclidean norm of an average of a bounded Markov chain induced martingale \eqref{eq:MC-martingale}, assuming the settings and conditions introduced in Section \ref{sec:settings}. We denote with 
\begin{align}\label{eq:defn-Sigman}
\bm{\Sigma}_n = \frac{1}{n} \sum_{i=1}^n \mathbb{E}_{s\sim\mu,s' \sim P(\cdot \mid s)}\Big[\bm{f}_i(s,s')\bm{f}_i^\top(s,s')\Big]
\end{align} 
the average variance of the martingale initialized from the stationary distribution (i.e. when $s_0 \sim \mu$),
which is well defined provided that $ \mathbb{E}_{s \sim \mu, s' \sim P(\cdot|s)}[\|\bm{f}_i(s,s')\|_2^2] < \infty $ for all $i$.

\medskip
\begin{theorem}[Bernstein's inequality for Markov chain induced martingales]\label{thm:matrix-bernstein-mtg}
Consider a Markov chain $\{s_i\}_{i =0,1,\ldots,n}$ satisfying Assumption \ref{as:nu} and let $\{\bm{f}_i\}_{i \in [n]}$ be a sequence of functions 
as in Assumption \ref{as:markov-mtg} such that 
\begin{align*}
\|\bm{f}_i(s,s')\|_2 \leq F,\quad \forall i \in [n], \quad \text{and} \quad  s,s' \in \mathcal{S},
\end{align*}
for some constant $F>0$.
Then for any $\delta \in (0,1)$, 
\begin{align}\label{eq:matrix-Bernstein}
\left\|\frac{1}{n}\sum_{i=1}^n \bm{f}_i(s_{i-1},s_i)\right\|_2  \lesssim \sqrt{\frac{\mathsf{Tr}(\bm{\Sigma}_n)}{n}\log \frac{1}{\delta}}+\frac{\sqrt{q}F}{(1-\lambda)^{\frac{1}{4}} n^{\frac{3}{4}}} \log^{\frac{3}{4}} \left(\frac{1}{\delta}\left\|\frac{\mathrm{d}\nu}{\mathrm{d}\mu}\right\|_{\mu,p}\right) + \frac{F}{n} \log \frac{1}{\delta},
\end{align}
with probability at least $1-\delta$.
\end{theorem}
\medskip
The proof of this result is provided in Appendix~\ref{app:proof-matrix-bernstein-mtg}. 
To establish Theorem ~\ref{thm:matrix-bernstein-mtg}, we deploy a version of Freedman's inequality for martingales in Hilbert spaces due to \cite{peng2024advances} (see \ref{thm:Hilbert-Freedman} in the appendix), yielding an upper bound that is dependent on the \emph{conditional} covariance matrix
\begin{align*}
\bar{\bm{\Sigma}}_n := \frac{1}{n}\sum_{i=1}^n \mathbb{E}[\bm{f}_i\bm{f}_i^\top\mid \mathcal{F}_{i-1}],
\end{align*}
where $\{ \mathcal{F}_{i} \}_{i=0}^{n}$ is the natural filtration induced by the Markov chain. 
However, since this quantity is not measurable with respect to the trivial $\sigma$-field $\mathscr{F}_0$,
 we further control the difference between the traces of $\bar{\bm{\Sigma}}_n$  and $\bm{\Sigma}_n$ using a one-dimensional Hoeffding's inequality for Markov chains by \cite{Fan2021Hoeffding} (Theorem \ref{thm:markov-hoeffding-1d} in the appendix), which leads to the first two terms in the upper bound \eqref{eq:matrix-Bernstein}. Notice that, assuming for simplicity that only $n$ and $d$ vary in the bound, the second and third terms on the right hand side of \eqref{eq:matrix-Bernstein} both converge faster than $n^{-1/2}$, so Theorem \ref{thm:matrix-bernstein-mtg} shows that the sample mean of the sequence $\{\bm{f}_i\}$ converges to $\bm{0}$ in Euclidean norm at a rate determined by $\sqrt{ \mathsf{Tr}(\bm{\Sigma}_n)/n}$ with high probability, just like in the i.i.d. case. This bound plays a key role in 
our analysis of TD learning in Section~\ref{sec:TD} and can be potentially used to understand the concentration of other machine learning algorithms. 

In our next result, we establish a high-dimensional Berry--Esseen bound for bounded Markov chain induced martingales in the convex distance, which is amenable to constructing confidence regions/sets, a task that would be otherwise difficult in other metrics, e.g the Wasserstein distance. 

\medskip
\begin{theorem}[High-dimensional Berry--Esseen bound on martingales generated from Markov chains]\label{thm:Berry--Esseen-mtg}
Consider a Markov chain $\{s_i\}_{i =1,\ldots,n}$ satisfying Assumption \ref{as:nu} with parameters $(p,q)$ and let $\{\bm{f}_i\}_{i \in [n]}$ be a sequence of functions as in Assumption \ref{as:markov-mtg} such  that 
 $\|\bm{\Sigma}_n^{-\frac{1}{2}}\bm{f}_i(s,s')\|_2 \leq M_i \leq M$ for all $s,s' \in \mathcal{S}$ and $i \in [n]$. Then, letting
\begin{align*}
\bar{M} = \left(\frac{\sum_{i=1}^n M_i^4}{n}\right)^{\frac{1}{4}},
\end{align*}
it holds that
\begin{align}\label{eq:Berry--Esseen-mtg}
&d_{\mathsf{C}}\left(\frac{1}{\sqrt{n}}\sum_{i=1}^n \bm{f}_i(s_{i-1}, s_i), \mathcal{N}(\bm{0},\bm{\Sigma}_n)\right) \nonumber \\ 
&\lesssim \left\{\bar{M}\left(\frac{q}{1-\lambda}\right)^{\frac{1}{4}}d^{\frac{3}{4}}\log^{\frac{1}{4}}\left(d\left\|\frac{\mathrm{d}\nu}{\mathrm{d}\mu}\right\|_{\mu,p}\right)+ \sqrt{M} d^{\frac{5}{8}}\log^{\frac{1}{2}} d\right\}  \frac{\log n}{n^{\frac{1}{4}}} .
\end{align}
\end{theorem}
\medskip

To demonstrate how the Berry--Esseen bound in Theorem \ref{thm:Berry--Esseen-mtg} scales with the dimension and with other problem-related quantities, consider the scenario in which 
$\bm{f}_1 = \bm{f}_2 = \ldots = \bm{f}_n = \bm{f}$, and $M_1 = M_2 = \ldots = M_n = M$. Then, in this case,
\begin{align*}
\bm{\Sigma}_n = \frac{1}{n} \sum_{i=1}^n \mathbb{E}_{s_0 \sim \mu}\Big[\bm{f}_i(s_{i-1},s_i)\bm{f}_i^\top(s_{i-1},s_i)\Big] = \mathbb{E}_{s \sim \mu, s' \sim P(\cdot \mid s)}[\bm{f}(s,s')\bm{f}^\top(s,s')] =: \bm{\Sigma},
\end{align*}
and $\bar{M} = M$. Therefore, \eqref{eq:Berry--Esseen-mtg} implies that
\begin{align*}
d_{\mathsf{C}}\left(\frac{1}{\sqrt{n}}\sum_{i=1}^n \bm{f}(s_{i-1}, s_i), \mathcal{N}(\bm{0},\bm{\Sigma})\right)  \lesssim \left(\frac{q}{1-\lambda}\right)^{\frac{1}{4}}\log^{\frac{1}{4}}\left(d\left\|\frac{\mathrm{d}\nu}{\mathrm{d}\mu}\right\|_{\mu,p}\right)Md^{\frac{3}{4}}n^{-\frac{1}{4}}\log n.
\end{align*}
We remark that the dependence on $d$ appears both explicitly in the $d^{\frac{3}{4}}$ term, and implicitly in $M$ term. 

The proof of Theorem \ref{thm:Berry--Esseen-mtg}, presented in Appendix \ref{app:proof-Berry--Esseen-mtg}, relies on arguments used by \cite{rollin2018} to derive Gaussian approximations for univariate martingales, which can be traced back to \cite{Dvoretzky.clt.martingale:72}. In the recent literature on high-dimensional statistics, this type of technique has been extended to the multivariate setting in 
\citet[][Lemma B.8]{cattaneo2024yurinskiiscouplingmartingales} and \citet[][Theorem 2.1]{
belloni2018highdimensionalcentrallimit}.
Specifically, we construct an auxiliary martingale with deterministic final quadratic variation, i.e. satisfying \eqref{eq:as-P1} below, and then apply a novel high-dimensional Berry--Esseen bound for martingale difference sequences in Wasserstein distance whose proof borrows heavily from the recent work by \cite{srikant2024rates}; see Theorem \ref{thm:Srikant-generalize} and Corollary \ref{cor:Wu} below. The difference between the target martingale and this auxiliary martingale is handled using a novel Hoeffding inequality for Markov chain, given in Theorem \ref{thm:matrix-hoeffding}. Finally, we combine these bounds using the properties of Gaussian distributions, as well as the relationship between convex distance and Wasserstein distance. These supportive Gaussian approximation and concentration results are novel and non-trivial, and may be of independent interest. They are presented, along with detailed commentary and comparisons with the existing literature, in the next section, which may be read as a stand-alone portion of the paper. 
%



\subsection{General vector-valued martingale bounds}
\label{sec:MC-Berry--Esseen}
This section collects several auxiliary results needed for the proof of Theorem \ref{thm:Berry--Esseen-mtg}, which may be of independent interest. The following theorem provides a novel high-dimensional Berry--Esseen bound for vector-valued martingales with \emph{deterministic terminal quadratic variation} in terms of the  Wasserstein distance. See Appendix \ref{app:Srikant-generalize} for the proof.

\medskip
\begin{theorem}[Berry--Esseen bound on vector-valued martingales]\label{thm:Srikant-generalize} 
Let $\{\bm{x}_i\}_{i=1}^n$ be a martingale difference process in $\mathbb{R}^d$ with respect to the filtration $\{\mathscr{F}_i\}_{i=0}^n$.
For every $i \in [n]$, define 
\begin{align*}
&\bm{V}_i := \mathbb{E}[\bm{x}_i\bm{x}_i^\top \mid \mathscr{F}_{i-1}], \quad  \text{and}  \quad \bm{P}_i := \sum_{j=i}^n \bm{V}_i.
\end{align*}
Furthermore, define
\begin{align*}
\bm{\Sigma}_n := \frac{1}{n}\sum_{i=1}^n \mathbb{E}[\bm{x}_i \bm{x}_i^\top \mid \mathscr{F}_0], 
\end{align*}
and assume that
\begin{align}\label{eq:as-P1}
\bm{P}_1 = n\bm{\Sigma}_n \quad \text{almost surely.}
\end{align}
Then for any  $d$-dimensional symmetric positive semi-definite matrix $\bm{\Sigma}$, it can be guaranteed that
\begin{align}\label{eq:Srikant-Berry--Esseen}
d_{\mathsf{W}}\left(\frac{1}{\sqrt{n}}\sum_{i=1}^n \bm{x}_i,\mathcal{N}(\bm{0},\bm{\Sigma}_n)\right) &\lesssim  \frac{(2+\log(d\|(n\bm{\Sigma}_n + \bm{\Sigma})\|))^+}{\sqrt{n}} \sum_{i=1}^n \mathbb{E}\left[\|(\bm{P}_i + \bm{\Sigma})^{-\frac{1}{2}}\bm{x}_i\|_2^2 \|\bm{x}_i\|_2 \right] \nonumber \\ 
&\qquad+ \frac{1}{\sqrt{n}}\left[\mathsf{Tr}(\log(n\bm{\Sigma}_n+\bm{\Sigma})) - \log(\bm{\Sigma}))\right]+ \sqrt{\frac{\mathsf{Tr}(\bm{\Sigma})}{n}}.
\end{align}
\end{theorem}
\medskip

Let us compare this result with analogous ones in the literature.  Theorem \ref{thm:Srikant-generalize} may be be regarded as a multivariate generalization of the univariate bound of Theorem 2.1 of \cite{rollin2018}. When $d=1$, Theorem \ref{thm:Srikant-generalize} agrees with Theorem 2.1 of \cite{rollin2018}, aside from logarithmic  factors; indeed, translating our notation into that of \cite{rollin2018}, we have that
\begin{align*}
	\bm{\Sigma}_n = s_n^2/n, \quad \bm{P}_i = \rho_i^2, \quad \text{and} \quad \bm{\Sigma} = a^2.
\end{align*}
\textcolor{black}{Theorem \ref{thm:Srikant-generalize} may be compared with Theorem 1 of \cite{JMLR2019CLT}, offering point-wise Gaussian approximation bounds for twice-differentiable functions of multivariate martingale sequences with deterministic terminal quadratic variation, and with Theorem 2.1 in \cite{belloni2018highdimensionalcentrallimit}, which focuses on thrice-differentiable functions. In contrast, our bound delivers uniform guarantees, as it establishes convergence rates in the Wasserstein distance, which in turn has direct relation with the convex distance (see, e.g., Theorem \ref{thm:Gaussian-convex-Wass} in the appendix).} 

As alluded above, our strategy to prove Theorem \ref{thm:Srikant-generalize} closely follows the strategy put forward recently by \cite{srikant2024rates}, which deploys Stein's method and Lindeberg swapping.  We point out that approach by addressing a gap in the proof arising when condition \eqref{eq:as-P1} does not hold and by deriving a tighter bound on the smoothness of the solution to the multivariate Stein's equation, which may be of independent interest (see Proposition \ref{prop:Stein-smooth} in Appendix \ref{app:Srikant-generalize}  and compare it to Proposition 2.2 and 2.3 in \cite{gallouët2018regularity}).

Under an additional technical condition, the bound of Theorem \ref{thm:Srikant-generalize} can be further tightened, as illustrated in the next result, which is more directly applicable to constructing the Berry--Esseen bound for Markov chain-induced martingales, e.g., Theorem \ref{thm:Berry--Esseen-mtg}. The proof is in
Appendix \ref{app:proof-cor-Wu}. 

\medskip
\begin{corollary}
\label{cor:Wu}
Under the settings of Theorem \ref{thm:Srikant-generalize}, further assume there exists a uniform constant $M > 0$, such that for any matrix $\bm{A} \in \mathbb{R}^{d \times d}$ and any $i \in [n]$, it satisfies 
\begin{align}\label{eq:3rd-momentum-condition}
\mathbb{E}\left[\|\bm{Ax}_i\|_2^2 \|\bm{x}_i\|_2 \bigg|\mathscr{F}_{i-1}\right] \leq M \mathbb{E}\left[\|\bm{Ax}_i\|_2^2 \bigg|\mathscr{F}_{i-1}\right].
\end{align}
Then,
\begin{align*}
	d_{\mathsf{W}}\left(\frac{1}{\sqrt{n}}\sum_{i=1}^n \bm{x}_i,\mathcal{N}(\bm{0},\bm{\Sigma}_n)\right) \lesssim \left[M(2+\log(dn\|\bm{\Sigma}_n\|))^+ + 1\right]\frac{d \log n}{\sqrt{n}} +\sqrt{\frac{\mathsf{Tr}(\bm{\Sigma}_n)}{n}}.
\end{align*}
\end{corollary}
\medskip

The above results bear similarities with Corollary 2.3 in \cite{rollin2018} in the univariate case and  Corollary 3 in \cite{JMLR2019CLT}, valid in multivariate settings (but for a different metric). In establishing our bound, we have lifted the requirement on the conditional third momentum of $\bm{x}_i$ and fixed a flaw in the proof of Corollary 3 of \cite{JMLR2019CLT}.


Importantly, both Theorem \ref{thm:Srikant-generalize} and Corollary \ref{cor:Wu} require a deterministic terminal quadratic variation $\bm{P}_1$; see Equation \eqref{eq:as-P1}. We remark that this condition, while impractical, is fairly standard in the literature: see, e.g., \cite{Bolthausen.82,Haeusler.88,VanDung2014L1BF,Machkouri.Ouchti.2007,rollin2018,JMLR2019CLT,Kojevnikov2022}. 
{\color{black} Recently, generalizing arguments from \cite{belloni2018highdimensionalcentrallimit}, \cite{cattaneo2024yurinskiiscouplingmartingales} derived Berry--Esseen bounds for vector-valued martingales through a generalization of Yurinskii's coupling that does not require this condition. Specifically, \cite{cattaneo2024yurinskiiscouplingmartingales} showed that
\begin{align}
d_{\mathsf{c}}\left(\frac{1}{\sqrt{n}}\sum_{i=1}^n \bm{x}_i,\mathcal{N}(\bm{0},\bm{\Sigma}_n)\right) &\lesssim \inf_{\eta > 0} \left\{ \frac{(\beta_{2,2}d)^{\frac{1}{3}}}{\sqrt{n}\eta} + \left(\frac{\mathbb{E}\|\bm{P}_1 - n\bm{\Sigma}_n\|\cdot d}{n\eta^2}\right)^{\frac{1}{3}} + \eta \sqrt{\|\bm{\Sigma}_n^{-1}\|_{\mathsf{F}}}\right\},
\end{align}
with
$\beta_{2,2} = \sum_{i=1}^n \mathbb{E}[\|\bm{x}_i\|_2^3 + \|\bm{V}_i^\frac{1}{2}\bm{z}_i\|_2^3].$
Here, $\bm{z}_1,\bm{z}_2,\ldots,\bm{z}_n$ are i.i.d. standard Gaussian random variables independent of $\mathscr{F}_n$ and $n\bm{\Sigma}_n$ an arbitrary deterministic positive semi-definite matrix. The rate of convergence is no faster than $O(n^{-\frac{1}{12}})$ and thus slower than the ones we obtain in  Theorem \ref{thm:Berry--Esseen-mtg}; see the discussion paragraph after Appendix \ref{app:proof-cor-Wu} for details.

For the case of a martingale sequence induced by a Markov chain, in Theorem \ref{thm:Berry--Esseen-mtg} we are able to relax the deterministic quadratic variation condition in \eqref{eq:as-P1} via a novel high probability bound on the spectral norm of the difference between $\bm{P}_1$ and $n\bm{\Sigma}_n$. This is the content of our the next result. 
}


\medskip

\begin{theorem}[Matrix Hoeffding's inequality for Markov Chains]\label{thm:matrix-hoeffding}
Consider a Markov chain $\{s_i\}_{i \in [n]}$  satisfying Assumption \ref{as:nu} with parameters $(p,q)$. Let $\{\bm{F}_i\}_{i \in [n]}$ be a sequence of matrix-valued functions from the state space $\mathcal{S}$ into $\mathbb{S}^{d \times d}$ such that $\mathbb{E}_{s \sim \mu}[\bm{F}_i(s)] = \bm{0}$ and $\sup_{s \in \mathcal{S}} \|\bm{F}_i(s)\| \leq M_i$ for all $i \in [n]$ and positive constants $M_1,\ldots,M_n$.
Then for any $\varepsilon > 0$, it holds that 
\begin{align}\label{eq:markov-matrix-hoeffding}
\mathbb{P}\left(\left\|\frac{1}{n}\sum_{i=1}^n \bm{F}_i(s_i)\right\| \geq \varepsilon \right) &\leq 2d^{2-\frac{\pi}{4}} \left\|\frac{\mathrm{d}\nu}{\mathrm{d}\mu}\right\|_{\mu,p}  \exp\left\{-\frac{1-\lambda}{20q}\left(\frac{\pi}{4}\right)^2\frac{n^2\varepsilon^2 }{\sum_{k=1}^n M_k^2} \right\}.
\end{align}
\end{theorem}
\medskip
Soon after the initial posting of this paper, we became aware of a recent paper by \citet{neeman2024concentration} who in their Theorem 2.5 already obtained the same bound as in Theorem~\ref{thm:matrix-hoeffding}. Our proof techniques share some similarities with theirs, as we both make use of the multi-matrix Golden-Thompson inequality of \cite{garg2018matrixexpanderchernoff}. Although we no longer claim novelty for Theorem \ref{thm:matrix-hoeffding}, we include our proof in Appendix \ref{app:proof-matrix-hoeffding} for completeness.

The bound of Theorem~\ref{thm:matrix-hoeffding} can be more conveniently re-stated as follows: for every $\delta \in (0,1)$, it can be guaranteed with probability $1-\delta$ that
\begin{align}
\label{eq:thm-matrix-hoeffding-whp}
\left\|\frac{1}{n}\sum_{i=1}^n \bm{F}_i(s_i)\right\|   \lesssim \sqrt{\frac{q}{1-\lambda} \frac{\sum_{j=1}^n M_j^2}{n} \log \Big(\frac{d}{\delta}\left\|\frac{\mathrm{d}\nu}{\mathrm{d}\mu}\right\|_{\mu,p}\Big)} \cdot \frac{1}{\sqrt{n}}.
\end{align}

\section{Application to TD learning with linear function approximation}\label{sec:TD}

In the second part of the paper, we apply our results to study the properties of the TD learning algorithm with linear function approximation under Markovian samples in the context of reinforcement learning (RL). This important subject has been extensively explored in the recent theoretical literature. 
These results can be broadly classified into two categories: (i) non-asymptotic bounds on the discrepancy between the algorithm’s output and the target quantity, and (ii) non-asymptotic distributional guarantees, such as Berry--Esseen bounds, which measure how fast the sequence of estimators converges to the limiting distribution.    However, when it comes to providing tight, non-asymptotic characterizations of general functions of Markov chains — an essential task in analyzing and calibrating machine learning procedures -- existing theoretical tools have limited scope and applicability, at least compared to the tools available for handling independent data.
The bounds on the performance and Gaussian approximations of the output of the TD learning methodology are, to the best of our knowledge, the sharpest to date. 

We begin by providing some background on TD learning and refer the reader to \cite{sutton2018reinforcement} for an accessible yet comprehensive introduction to this topic.



\subsection{TD learning with linear function approximation}

Consider a Markov Reward Process (MRP), where a Markov chain  on a state space $\mathcal{S}$ is associated with a \emph{reward function} $r: \mathcal{S} \to [0,1]$ that maps each state to a reward. 
We observe a sequence of state-reward pairs, 
\begin{align*}
	(s_0,r_0),\ldots (s_t, r_t),\ldots
\end{align*}
where, for each $t$, $r_t = r(s_t)$ and $s_{t+1} \sim P(\cdot \mid s_t)$. 
The task of \emph{value function} evaluation concerns with estimating a function $V:\mathcal{S} \to \mathbb{R}$, where  for every $s \in \mathcal{S}$, 
\begin{align*}
	V(s):=\mathbb{E} \Bigg[\sum_{t=0}^\infty \gamma^t r(s_t)\Bigg].
\end{align*}
Here, $\gamma \in [0,1)$ denotes a discounted factor of future rewards towards the current state. 
In practice, the state space $\mathcal{S}$ used to describe the environment's configuration is often prohibitively large, making it necessary to consider various approximations of $V$. The most simple and tractable form of approximation is the linear function approximation, which has received considerable attention in literature; see, e.g., \cite{bhandari2018finite,patil2023finite,dalal2018finite,li2023sharp,samsonov2023finitesample,samsonov2024gaussian,wu2024statistical}.  
Specifically, the value function $V$ of a policy is approximated by the linear function
\begin{align}\label{eq:linear-approximation}
	 V_{\bm\theta}(s) := {\bm\phi}(s)^\top {\bm\theta}, \qquad \quad  s \in \mathcal{S},
\end{align}
for a set of feature maps $\bm{\phi}: \mathcal{S} \to \mathbb{R}^d$ and a linear coefficient vector  $\bm{\theta} \in \mathbb{R}^d$. Throughout the paper, we assume that $\|\bm{\phi}(s)\|_2 \leq 1$ for all $s \in \mathcal{S}$. 
Recall that we use $\mu$ to denote the stationary distribution of the Markov chain. 
The optimal coefficient vector, denoted as $\bm{\theta}^\star$, is defined by the projected Bellman equation \citep{tsitsiklis1997analysis}, which admits the fixed point equation 
\begin{align}
\label{eq:defn-theta-star}
\bm{A} \bm{\theta}^{\star}=\bm{b}, 
\end{align}
where 
\begin{align}
	\bm{A} :=\mathop{\mathbb{E}}\limits _{s\sim\mu,s'\sim P(\cdot\mid s)}\Big[\bm{\phi}(s)\left(\bm{\phi}(s)-\gamma\bm{\phi}(s')\right)^{\top}\Big]\in\mathbb{R}^{d\times d},\quad
	\text{and}\quad \bm{b} :=\mathop{\mathbb{E}}\limits _{s\sim\mu}\Big[\bm{\phi}(s)r(s)\Big]\in\mathbb{R}^{d}.
	\label{eq:defn-At-mean}
\end{align}
 Below we indicate $\bm{\Sigma}$ as the feature Gram matrix, i.e.
\begin{align}\label{eq:defn-Sigma}
	\bm{\Sigma} := \mathbb{E}_{s \sim \mu } \Big[\bm{\phi}(s)\bm{\phi}^\top(s)\Big],
\end{align}
and use $\lambda_{\Sigma}$ and $\lambda_0$ to denote its largest and smallest eigenvalues respectively. We assume $\lambda_0 > 0$.

\paragraph{The TD learning algorithm} 
Stochastic approximation (SA) is a standard tool to solve the fixed point equation problems of the form \eqref{eq:defn-theta-star} given a sequence of random observations $\{(\bm{A}_t, \bm{b}_t)\}_{t\geq 1}.$
When specialized to the above setting, it is referred to as the Temporal Difference (TD) learning algorithm \citep{sutton1988learning}. 
Specifically, given a sequence of pre-selected step size $\{\eta_t\}_{t\geq 1}$ and an initial estimator $\bm{\theta}_0 = \bm{0}$, TD proceeds using the updating rule 
	\label{eq:TD-update-all} 
	\begin{align}
	\bm{\theta}_{t} & =\bm{\theta}_{t-1}-\eta_{t}(\bm{A}_{t}\bm{\theta}_{t-1}-\bm{b}_{t}),
	\label{eq:TD-update-rule}
\end{align}
where 
\begin{align}
\bm{A}_{t}  :=\bm{\phi}(s_{t-1})\left(\bm{\phi}(s_{t-1})-\gamma\bm{\phi}(s_{t})\right)^{\top} \quad
\text{and }\quad \bm{b}_{t} :=\bm{\phi}(s_{t-1})r_{t-1}.\label{eq:defn-At}
\end{align}
After $T$ iterations, we deploy Polyak-Ruppert averaging \citep{polyak1992acceleration,ruppert1988efficient} and compute 
\begin{align}
	\label{eq:TD-averaging}
	\bar{\bm{\theta}}_T & =\frac{1}{T}\sum_{t=1}^{T}\bm{\theta}_{t}, 
\end{align}
as the estimator for $\bm{\theta}^\star$. In this work, we follow the precedent of \cite{polyak1992acceleration,ruppert1988efficient} and choose \emph{polynomial-decaying} stepsizes $\eta_t = \eta_0 t^{-\alpha}$ with $\alpha \in (\frac{1}{2},1)$. 

\subsection{Convergence and Berry-Essen bounds for TD estimator}

It is known from the literature on stochastic approximations that, in fixed dimensions,  the TD estimator with polynomial-decaying stepsizes and Polyak-Ruppert averaging $\bar{\bm{\theta}}_T$ satisfies the central limit theorem (CLT) \citep[see, e.g.,][]{fort2015central,mou2020linear,li2023online}
\begin{align*}
    \sqrt{T}(\bar{\bm{\theta}}_T - \bm{\theta}^\star) \xrightarrow{d} \mathcal{N}(\bm{0},\tilde{\bm{\Lambda}}^\star),
\end{align*}
with asymptotic covariance matrix $\tilde{\bm{\Lambda}}^\star$ 
\begin{align}
    \label{eq:defn-tilde-Lambdastar}
    \tilde{\bm{\Lambda}}^\star = \bm{A}^{-1} \tilde{\bm{\Gamma}}\bm{A}^{-\top},
\end{align}
where $\tilde{\bm{\Gamma}}$ is the time-averaging covariance matrix
\begin{align}
\label{eq:defn-tilde-Gamma}
\tilde{\bm{\Gamma}}&=\lim_{T \to \infty} \mathbf{Var}_{s_0 \sim \mu,s_{t+1} \sim P(\cdot \mid s_t)} \left[\frac{1}{T}\sum_{t=1}^T (\bm{A}_t\bm{\theta}^\star - \bm{b}_t)\right]\nonumber \\ 
&= \mathbb{E}[(\bm{A}_1 \bm{\theta}^\star - \bm{b}_1)(\bm{A}_1 \bm{\theta}^\star - \bm{b}_1)^\top] + \sum_{t=2}^{\infty} \mathbb{E}[(\bm{A}_1 \bm{\theta}^\star - \bm{b}_1)(\bm{A}_{t} \bm{\theta}^\star - \bm{b}_{t})^\top + (\bm{A}_{t} \bm{\theta}^\star - \bm{b}_{t})(\bm{A}_1 \bm{\theta}^\star - \bm{b}_1)^\top].
\end{align}
However, the above results are asymptotic in nature and do not explicitly reveal the dependence on the dimension and other problem-related quantities. A line of recent work has made headway toward this goal, attaining non-asymptotic distributional characterization of the TD procedure. 
Nonetheless, most of the literature has focused on the independent setting where each $(\bm{A}_t,\bm{b}_t)$ pair is an independent and identically distributed random variable \citep[see, e.g.,][]{mou2020linear,wu2024statistical, samsonov2024gaussian}. 
To describe our results, throughout this section, we make the additional assumption that the Markov chain mixes exponentially fast, a standard condition in finite sample analyses of Markovian data. 
\medskip
\begin{assumption}\label{as:mixing}
There exists constants $m>0,\rho\in (0,1)$, such that for every positive integer $t$,
\begin{align*}
\sup_{s \in \mathcal{S}} d_{\mathsf{TV}}(P^t(\cdot \mid s), \mu) \leq m \rho^t.
\end{align*}
\end{assumption}
\medskip
We remark that, under this assumption, for any $\varepsilon \in (0,1)$, the corresponding \emph{mixing time} 
\begin{align}\label{eq:defn-tmix}
t_{\mix}(\varepsilon):= \min\{t:\sup_{s \in \mathcal{S}}d_{\mathsf{TV}}(P^t(\cdot \mid s),\mu) \leq \varepsilon\},
\end{align}
satisfies the bound
\begin{align}\label{eq:tmix-bound}
\tmix(\varepsilon) \leq \frac{ \log (m/\varepsilon)}{\log (1/\rho)}.
\end{align}

\medskip
\begin{theorem}[High-probability convergence of TD estimator]\label{thm:TD-whp}
Consider TD with Polyak-Ruppert averaging~\eqref{eq:TD-update-all} with Markov samples and decaying stepsizes $\eta_t = \eta_0 t^{-\alpha}$ for $\alpha \in (\frac{1}{2},1)$. 
Suppose that the Markov transition kernel has a unique stationary distribution, has a positive spectral gap, mixes exponentially as indicated by Assumption \ref{as:mixing}, and starts from a distribution satisfying Assumption \ref{as:nu}. Then for every tolerance level $\delta \in (0,1)$, there exists $\eta_0 = \eta_0(\delta)$ such that
\begin{align*}
\|\bar{\bm{\theta}}_T-\bm{\theta}^\star\|_2 \lesssim \sqrt{\frac{\mathsf{Tr}(\tilde{\bm{\Lambda}}^\star)}{T}\log \frac{1}{\delta}} + o\left(\sqrt{\frac{1}{T}}\log^{\frac{3}{2}} \frac{d}{\delta}\right),
\end{align*}
with probability at least $1-\delta$.
\end{theorem}
\medskip

Theorem~\ref{thm:TD-whp}, proved in Appendix \ref{app:proof-markov-deltat-convergence}, establishes the first high-probability convergence guarantee for the TD estimation error with Markov samples that matches the asymptotic variance $\tilde{\bm{\Lambda}}^\star/T$ up to a log factor.  It may be regarded as the Markovian counterpart of Theorem 3.1 in \cite{wu2024statistical}, which focuses on the TD estimation error with \emph{independent} samples under the same settings. 


Our analysis of the TD estimation error borrows tools and ideas from several previous contributions, starting from the seminal work of \cite{polyak1992acceleration} and including some of the most recent results; see, e.g, \cite{li2023online,samsonov2024gaussian,srikant2024rates}. In particular, we apply the induction technique of \cite{srikant2019finite}, \cite{li2023sharp} and \cite{wu2024statistical}, developed for TD learning with \emph{independent} samples. The generalization from independent to Markov samples is highly nontrivial, and the novel theoretical results obtained in the first part of the paper, especially the newly-established matrix Bernstein inequality for Markovian martingales (Corollary \ref{thm:matrix-bernstein-mtg}), play a critical role in our analysis.

\paragraph{Dependence on problem-related quantities}
We point out that the initial stepsize $\eta_0$ depends on the probability tolerance level $\delta$ as well as other problem-related quantities, as specified in Eq.~\eqref{eq:deltat-condition-markov} in Appendix \ref{app:proof-TD-original}. Though not ideal, this dependence  
also appears in the independent setting in \cite[Theorem 3.1]{wu2024statistical}, who further argues that it is unavoidable. It is also notable that the choice of $\eta_0$ does \emph{not} depend on the sample size $T$, as is the case in \cite{samsonov2023finitesample}, who also considers TD with Markov samples, though with a \emph{time-invariant} stepsize choice. 
The upper bound in Theorem \ref{thm:TD-whp} depends on various problem-related quantities, including the mixing speed of the Markov chain (specifically, the mixing factor $\rho$ and the spectral gap $1-\lambda$), the discount factor $\gamma$, the initial stepsize $\eta_0$, the feature covariance matrix $\bm{\Sigma}$ and the time-averaging variance matrix of the TD error, $\tilde{\bm{\Gamma}}$. These intricate dependencies impact the higher-order (in $T$)  reminder terms, which is not shown in our bound but can be tracked in our proofs. 

In our next and final result, we give a novel high-dimensional Berry--Esseen bound for the TD estimator assuming Markovian data. 

\medskip
\begin{theorem}[Berry--Esseen bound for TD estimator]\label{thm:TD-Berry--Esseen}
Consider TD with Polyak-Ruppert averaging~\eqref{eq:TD-update-all} with Markov samples and decaying stepsizes $\eta_t = \eta_0 t^{-\frac{3}{4}}$, where $\eta_0 < \frac{1}{2}$. Suppose that the Markov transition kernel has a unique stationary distribution, has a positive spectral gap, mixes exponentially as indicated by Assumption \ref{as:mixing}, and is initiated from a distribution $\nu$ satisfying Assumption \ref{as:nu}. Further assume that $\lambda_{\min}(\tilde{\bm{\Gamma}}) > 0$. Then when $T$ is sufficiently large\footnote{The exact constraint on $T$ is indicated in \eqref{eq:Lambda-T-condition} in the Appendix.}, 
\begin{align*}
d_{\mathsf{C}}(\sqrt{T}(\bar{\bm{\theta}}_T-\bm{\theta}^\star),\mathcal{N}(\bm{0},\tilde{\bm{\Lambda}}^\star)) \leq \tilde{C}T^{-\frac{1}{4}}\log T + o(T^{-\frac{1}{4}}),
\end{align*}
where $\widetilde{C}$ is a problem-related quantity independent of $T$, with exact form shown in Appendix \ref{app:proof-TD-Berry--Esseen}.
\end{theorem}
\medskip


In a nutshell, Theorem~\ref{thm:TD-Berry--Esseen}, whose proof is given in Appendix \ref{app:proof-TD-Berry--Esseen}, ensures that the rescaled TD estimator with Polyak-Ruppert averaging converges to its Gaussian limit at a rate of $T^{-1/4}\log T$ (holding all other parameters fixed) with respect to the convex distance. 
The dependence on the feature dimension $d$, as well as other problem-related parameters, is unwieldy and thus not given explicitly in the statement of the theorem. However, in principle it can be tracked through the various steps of the proof.

The closest contribution to ours, and indeed the impetus for some of our work,  is the recent excellent paper by \cite{srikant2024rates}, which claims the bound 
\begin{align*}
    d_{\mathsf{W}}(\sqrt{T}(\bar{\bm{\theta}}_T-\bm{\theta}^\star),\mathcal{N}(\bm{0},\tilde{\bm{\Lambda}}^\star)) &\leq \tilde{C}T^{-\frac{1}{6}}\log T.
\end{align*}
where $\tilde{C}$ is a problem-related quantity independent of $T$. In our analysis, we have filled in the gaps in some of their arguments concerning vector-valued martingales and their application to TD learning, while also strengthening the bound from Wasserstein to convex distance.
Other recent and directly relevant contributions by \cite{samsonov2024gaussian} and \cite{wu2024statistical} have also produced high-dimensional Berry--Esseen bound for TD learning with \emph{independent samples},  with \cite{wu2024statistical} claiming the the state-of-the-art rate (in $T$) of $O(T^{-\frac{1}{3}})$. Though it remains to be seen whether the rate of convergence of order  $O(T^{-\frac{1}{4}}\log T)$ for Markovian data that we obtain in Theorem \ref{thm:TD-Berry--Esseen} is sharp, the generalization from independent to Markov samples is nonetheless highly nontrivial and, we believe,  significant. In any case, we are able to show in Appendix \ref{app:proof-Berry--Esseen-tight} that, for any $\alpha > \frac{3}{4}$ and when $T$ is sufficiently large,
\begin{align}\label{eq:TD-Berry--Esseen-tight}
d_{\mathsf{C}}(\sqrt{T}(\bar{\bm{\theta}}_T-\bm{\theta}^\star),\mathcal{N}(\bm{0},\tilde{\bm{\Lambda}}^\star)) &\geq \widetilde{C}T^{-\frac{1}{4}}\log T,\quad\text{for all } \alpha \in \left(\frac{3}{4},1\right).
\end{align}
Here $\tilde{C}$ is a problem-related quantity independent of $T$. This lower bound provides a partial clue that our rate might in fact be optimal.

\section{Discussion and future directions}

In this paper, we derive novel, high-dimensional uncertainty quantification results for Markov chain induced martingales, including a Bernstein-style concentration inequality and a Berry--Esseen bound. Through the process, we construct improved Berry--Esseen bounds for martingales with fixed terminal quadratic variance, which may be of independent interest.

We apply these results to obtain state-of-the-art finite sample convergence guarantees and Gaussian approximations for TD learning with Markovian observation. 
While this work addresses a wide range of theoretical problems, it also points to several directions for future research. First, our Berry--Esseen bound for the convex distance in  Theorem \ref{thm:Berry--Esseen-mtg} exhibits a polynomial dependence on the dimension $d$. This is a feature of virtually all such bounds from the literature, most of which are concerned with independent sums; see, e.g., \cite{Bentkus2003OnTD, Bentkus2005ALB, schulte2019multivariate,10.1214/23-AAP2014, Rai2019, nourdin2021multivariate, cattaneo2024yurinskiiscouplingmartingales}.
A poly-logarithmic dependence in $d$ holds for the sub-class of hyper-rectangles; see \cite{KOJEVNIKOV2022109448}. In regards to other applications, beyond TD learning, Theorems \ref{thm:matrix-bernstein-mtg} and \ref{thm:Berry--Esseen-mtg} provide novel tools for analyzing the statistical properties of other Markov chain-based ML algorithms like MCMC, general stochastic approximation and Q-learning. {\color{black} However,  we remark that, in order to carry out statistical inference, it is necessary to have an (ideally efficient) estimator of the variance matrix $\tilde{\bm{\Lambda}}^\star$. The construction and convergence rate of any such estimator remain important open problems in the field.} Finally, whether it is possible to close the gap between the $O(T^{-\frac{1}{4}}\log T)$ rate in the Berry--Esseen bound for TD with Markov data in Theorem \ref{thm:TD-Berry--Esseen} and the $O(T^{-\frac{1}{3}})$ rate for its independent-sample counterpart \cite{wu2024statistical} is also an interesting problem.

\section*{Ackowledgement}
W. Wu and A.Rinaldo were partially supported in part by NIH under Grant R01 NS121913 and NSF award DMS-2113611.
Y. Wei is supported in part by the NSF grants CCF-2106778, CCF-2418156 and CAREER award DMS-2143215. We thank Bobby Shi for making us aware of Theorem 2.5 of \citet{neeman2024concentration}, giving the same inequality as in Theorem~\ref{thm:matrix-hoeffding}.


\bibliography{refs.bib,bibfileRL,bibfileRL-2}
\bibliographystyle{plainnat}

\appendix 
\section{Preliminary facts}

\subsection{Markov chain basics}\label{app:MC-basics}

\subsubsection{Stationary distribution and corresponding norm}
Let $\mathcal{S}$ denote the state space of the homogeneous, discrete time Markov chain, which we assumed throughout to be a Borel space endowed with it Borel $\sigma$-algebra $\mathscr{F}$. For every $x \in \mathcal{S}$, the \emph{transition kernel} $P(x,\cdot)$ is a probability measure on $\mathcal{S}$ defined as
\begin{align*}
P(x,B) = \mathbb{P}(s_2 \in B \mid s_1 = x), \quad B \in \mathscr{F}
\end{align*}
which is guaranteed to be a regular conditional distribution.  
For any integer $k > 0$, the $k$th power of the kernel $P$ is the regular conditional probability distribution defined iteratively as  
\[
P^k(x,B) = \int_{\mathcal{S}} P(x,dy)P^{k-1}(y,B) = \mathbb{P}(s_k \in B | s_0 = x), \quad x \in \mathcal{S}, B \in \mathscr{F},
\]
where we define $P^0(x,B) = \mathsf{1}_{\{ x \in B\}}$.
We will assume throughout that the Markov chain admits a (unique!) \emph{stationary distribution} $\mu$,  a probability distribution on $(\mathcal{S},\mathscr{F})$ satisfying
\begin{align*}
\mu(B) = \int_{ \mathcal{S}}P(y,B) \mu(\mathrm{d}y), \quad \forall B \in \mathscr{F}.
\end{align*}
 For a vector-valued function $\bm{g}: \mathcal{S} \to \mathbb{R}^{d}$, we use $\mu(\bm{g})$ to denote $\mathbb{E}_{x \sim \mu} [\bm{g}(x)]$, if well-defined. Furthermore, for simplicity, we use $\bm{g}^{\parallel}$ and $\bm{g}^{\perp}$ to denote the functions
\begin{align*}
&\bm{g}^{\parallel}(x) = \mu(\bm{g}), \quad \text{and} \\ 
&\bm{g}^{\perp}(x) = \bm{g}(x) - \bm{g}^{\parallel}(x), \quad \forall x \in \mathcal{S},
\end{align*}
provided that $\mu(\bm{g})$ exists.
We denote with $L_2(\mu)$ the space of Borel functions $\bm{f}\colon \mathcal{S} \to \mathbb{R}^{d}$ such that $\int_{\mathcal{S}} \| \bm{f}(x) \|^2 \mu(dx) < \infty $. Then, $L_2(\mu)$ is a Hilbert space, 
with \emph{inner product} given by
\begin{align*}
\langle \bm{g}_1, \bm{g}_2 \rangle_{\mu} 
= \int_{\mathcal{S}} \bm{g}_1^\top(x)\bm{g}_2(x)\mu(\mathrm{d}x), \quad \bm{g}_1,\bm{g}_2 \in L_2(\mu),
\end{align*}
which, in turn, induces the \emph{$\mu$-norm} 
\begin{align*}
\|\bm{g}\|_{\mu} 
= \left(\int\|\bm{g}(x)\|_2^2 \mu(\mathrm{d}x)\right)^{\frac{1}{2}}, \quad \bm{g} \in L_2(\mu).
\end{align*}
Note that, for a fixed vector $\bm{v}$, $\|\bm{v}\|_{\mu}$ reduces to the Euclidean norm $\|\bm{v}\|_{2}$. 
We shall interchange these two ways of writing for ease of exposition.


To the transition kernel $P$ admitting a stationary distribution $\mu$ there corresponds a bounded linear operator, denoted as $\mathcal{P}$, which maps a vector-valued function $\bm{g} \in L_2(\mu)$ 
to another vector-valued function in $L_2(\mu)$, denoted as $\mathcal{P}\bm{g}$, given by
\begin{align*}
x \in \mathcal{S} \mapsto \mathcal{P}\bm{g}(x) = \int_{y \in \mathcal{S}} \bm{g}(y)P(x,\mathrm{d}y) = \mathbb{E}_{s_2 \sim P(\cdot \mid x)}[\bm{g}(s_2)\mid s_1 = x],
\end{align*}
where, with a slight abuse of notation, we write $P(\cdot \mid s_1)$ for $P(s_1,\cdot)$. Accordingly, for any positive integer $k$, we define the operator $\mathcal{P}^k$ over $L_2(\mu)$ by 
\begin{align*}
x \in \mathcal{S} \mapsto \mathcal{P}^k\bm{g}(x) = \int_{y \in \mathcal{S}} \bm{g}(y)P^k(x,\mathrm{d}y) = \mathbb{E}_{s_k \sim P^k(\cdot \mid x)}[\bm{g}(s_k)\mid s_0 = x].
\end{align*}
The \emph{adjoint operator} of $\mathcal{P}$, denoted as $\mathcal{P}^*$, is a bounded operator mapping a vector-valued function $\bm{g} \in L_2(\mu)$ to another vector-valued function in $L_2(\mu)$, denoted as $\mathcal{P}^*\bm{g}$ such that 
\begin{align}\label{eq:adjoint}
\langle\mathcal{P}\bm{g}_1, \bm{g}_2\rangle_{\mu} = \langle\bm{g}_1, \mathcal{P}^*\bm{g}_2\rangle_{\mu}.
\end{align}

The \emph{adjoint transition kernel} $P^*$ is the transition kernel (regular conditional probability) satisfying 
\begin{align*}
\int_{A }\mu(\mathrm{d}x) P^*(x,B) = \int_{B } P(y,A) \mu(\mathrm{d}y), \quad \forall A, B \in \mathscr{F}.
\end{align*}

Accordingly, $\mathcal{P}^*$ maps any $\bm{g} \in L_2(\mu)$ to a new function $\mathcal{P}^* \bm{g} \in L_2(\mu)$ given by
\begin{align*}
x \in \mathcal{S} \mapsto \mathcal{P}^*\bm{g}(x) = \int_{y \in \mathcal{S}} \bm{g}(y)P^*(x,\mathrm{d}y) = \mathbb{E}[\bm{g}(s_1)\mid s_2 = x],
\end{align*}
and condition \eqref{eq:adjoint} can be equivalently expressed as
\[
\int_{\mathcal{S}} \mu(dx)P^*(x,dy)  \bm{f}^\top(x)  \bm{g}(y) = \int_{\mathcal{S}} \mu(dy) P(y,dx) \bm{f}^\top(x)  \bm{g}(y), \quad \forall \bm{f}, \bm{g} \in L_2(\mu).
\]

A Markov chain is called \emph{self-adjoint}, or \emph{reversible}, if $P = P^*$ or, equivalently, if $\mathcal{P} = \mathcal{P}^*$.

Finally, with a minor abuse of notation, for every measure $\nu$ on $(\mathcal{S},\mathscr{F})$, we use $\mathcal{P}\nu$ to denote the measure defined by
\begin{align*}
\mathcal{P}\nu(B) = \int_{\mathcal{S}} P(y,B)\nu(\mathrm{d}y), \quad \forall B \in \mathscr{F}.
\end{align*}
By definition, the stationary distribution $\mu$ satisfies $\mathcal{P}\mu = \mu$.

\subsubsection{Spectral gap}
For a Markov chain with transition kernel $P$ and unique stationary distribution $\mu$, let $\mathcal{L}_2$ denote the Hilbert space
\begin{align*}
\mathcal{L}_2:= \{g:\mathcal{S} \to \mathbb{R} \mid \|g\|_{\mu} < \infty\}. 
\end{align*}

The \emph{spectral gap} of the Markov chain is defined as $1-\lambda$, where
\begin{align*}
\lambda = \lambda(P):= \sup_{g \in \mathcal{L}_2,\mu(g) = 0} \frac{\|\mathcal{P} g\|_{\mu}}{\|g\|_{\mu}}.
\end{align*}
We now present some properties of the spectral gap.

\begin{proposition}\label{prop:adjoint-spectral}
For any Markov chain with transition kernel $P$, as long as $\lambda(P)$ and the adjoint transition kernel $P^*$ are both well-defined, it can be guaranteed that
\begin{align*}
\lambda(P) = \lambda(P^*).
\end{align*}
\end{proposition}
\begin{proof}
Define $\Pi$ as the projection of any function onto the subspace of constant functions:
\begin{align*}
\Pi g(x):= \mu(g), \quad \forall x\in \mathcal{S}.
\end{align*}
It is easy to verify that $\lambda(P)$ is equal to the operator norm of $\mathcal{P}-\Pi$, and that $\lambda(P^*)$ is equal to the operator norm of $\mathcal{P}^* - \Pi$. Furthermore, the operators $\mathcal{P}-\Pi$ and $\mathcal{P}^* - \Pi$ are adjoint. Therefore, the proposition follows from the facts that adjoint operators have the same spectrum in Hilbert spaces. 
\end{proof}

\begin{proposition}\label{prop:spectral-expansion}
For any dimension $d$, and $\bm{g}:\mathcal{S} \to \mathbb{R}^d$, it can be guaranteed that as long as $\mu(\bm{g}) = \bm{0}$ and $\|\bm{g}\|_{\mu} < \infty$,
\begin{align*}
\|\mathcal{P}^*\bm{g}\|_{\mu} \leq \lambda \|\bm{g}\|_{\mu}
\end{align*}
\end{proposition}
\begin{proof}
This property is easy to verify by representing $\bm{g} = \{g_{i}\}_{1 \leq i\leq d}$. 
\end{proof}

\begin{proposition}\label{prop:P-top}
For any $\bm{g}: \mathcal{S} \to \mathbb{R}^{d}$, if $\mu(\bm{g}) = \bm{0}$, then $\mu(\mathcal{P}^* \bm{g}) = 0$ and therefore $\|\mathcal{P}^{*} \bm{g}\|_{\mu} \leq \lambda \|\bm{g}\|_{\mu}$.
\end{proposition}
\begin{proof} 
By definition, $\mu(\mathcal{P}^* \bm{g})$ is featured by
\begin{align*}
\mu(\mathcal{P}^* \bm{g}) &= \int_{x \in \mathcal{S}} \mathcal{P}^* \bm{g}(x) \mu(\mathrm{d}x) \\ 
&= \int_{x \in \mathcal{S}} \left[\int_{y \in \mathcal{S}} \bm{g}(y)\frac{P(y,\mathrm{d}x)}{\mu(\mathrm{d}x)}\mu(\mathrm{d}y)\right] \mu(\mathrm{d}x)\\ 
&= \int_{y \in \mathcal{S}} \bm{g}(y) \left[\int_{x \in \mathcal{S}} P(y,\mathrm{d}x) \right] \mu(\mathrm{d}y)\\ 
&= \int_{y \in \mathcal{S}} \bm{g}(y) \mu(\mathrm{d}y) = \mu(\bm{g}) = \bm{0}.
\end{align*}
Consequently, $\|\mathcal{P}^* \bm{g}\|_{\mu} \leq \lambda \|\bm{g}\|_{\mu}$ follows from Proposition \ref{prop:spectral-expansion}. 
\end{proof}

As a direct corollary of Proposition \ref{prop:P-top}, it is easy to verify that $(\mathcal{P}^* \bm{g})^{\parallel} = \bm{g}^{\parallel}$ and $(\mathcal{P}^* \bm{g})^{\perp} = \mathcal{P}^{*}(\bm{g}^{\perp})$.


\subsection{Algebraic facts}

\paragraph{Supplementary notation} For any complex number $z \in \mathbb{C}$, we use $\mathsf{Re}(z)$ to denote its real part. If a matrix $\bm{M}$ is positive (semi-)definite, i.e. $\lambda_{\min}(\bm{M}) > (\geq) 0$, we write $\bm{M} \succ (\succeq) \bm{0}$; when two matrices $\bm{X}$ and $\bm{Y}$ satisfy $\bm{X} - \bm{Y} \succ (\succeq) \bm{0}$, we write $\bm{X} \succ(\succeq) \bm{Y}$ or equivalently $\bm{Y} \prec (\preceq) \bm{X}$. 
For a vector $\bm{x} \in \mathbb{R}^d$, we use $\{x_i\}_{1 \leq i \leq d}$ to denote its entries; for a matrix $\bm{X} \in \mathbb{R}^{m \times n}$, we use $\{X_{ij}\}_{1 \leq i \leq m,1\leq j \leq n}$ to denote its entries. The \emph{vectorization} of matrix $\bm{X} \in \mathbb{R}^{m \times n}$, denoted as $\textbf{vec}(\bm{X})$, is defined as a vector in $\mathbb{R}^{mn}$ with entries $[\textbf{vec}(\bm{X})]_{(i-1)m+j} = \bm{X}_{ij}$ for every $1 \leq i \leq m$ and $1 \leq j \leq n$. 
For any matrices $\bm{X} \in \mathbb{R}^{m \times n}$, and $\bm{Y} \in \mathbb{R}^{p \times q}$, their \emph{Kronecker product}, denoted as $\bm{X} \otimes \bm{Y}$, is a matrix of the size of $mp \times nq$, and can be written in blocked form as
\begin{align*}
\bm{X} \otimes \bm{Y} = \begin{bmatrix}
X_{11}\bm{Y} & X_{12}\bm{Y} & \ldots & X _{1n}\bm{Y} \\ 
X_{21}\bm{Y} & X_{22}\bm{Y} & \ldots & X_{2n} \bm{Y} \\ 
\ldots & \ldots & \ldots & \ldots \\ 
X_{m1}\bm{Y} & X_{m2}\bm{Y} &\ldots &X_{mn} \bm{Y}
\end{bmatrix}.
\end{align*}
With subscripts, it can be represented as $(\bm{X} \otimes \bm{Y})_{(i-1)m + j, (k-1)p + \ell} = X_{ik}Y_{j\ell}$.

\subsubsection{General algebraic theorems}

\begin{theorem}[Properties of the Kronecker product]\label{thm:Kronecker}
The following properties hold:
\begin{enumerate}[label={(\arabic*)}]
\item Let $\bm{X},\bm{Y},\bm{Z},\bm{W}$ be four matrices such that the matrix products $\bm{XZ}$ and $\bm{YW}$ are well-defined. Then $(\bm{X} \otimes \bm{Y})(\bm{Z}\otimes \bm{W}) = (\bm{XZ}) \otimes \bm{YW}$;
\item As a direct consequence of (1), let $\bm{X} \in \mathbb{R}^{m \times n}$ and $\bm{Y} \in \mathbb{R}^{p \times q}$, then $\bm{X} \otimes \bm{Y} = (\bm{X} \otimes \bm{I}_p)(\bm{I}_n \otimes \bm{Y})$; 
\item For any matrices $\bm{X},\bm{Y}$, $(\bm{X} \otimes \bm{Y})^\top = \bm{X}^\top \otimes \bm{Y}^\top$;
\item For any matrices $\bm{X} \in \mathbb{R}^{m \times n}$, $\bm{Y} \in \mathbb{R}^{p \times q}$ and $\bm{Z} \in \mathbb{R}^{q \times n}$, it can be guaranteed that $(\bm{A} \otimes \bm{B})\textbf{vec}(\bm{Z}) = \textbf{vec}(\bm{Y}\bm{Z}\bm{X}^\top)$;
\item As a direct consequence of (4), let $\bm{X} \in \mathbb{R}^{m \times n}$ and $\bm{y} \in \mathbb{R}^n$, then $\bm{Xy} = (\bm{I}_m \otimes \bm{y})\textbf{vec}(\bm{X})$;
\item Also as a direct consequence of (5), let $\bm{X},\bm{Y} \in \mathbb{R}^{d \times d}$, then $\mathsf{Tr}(\bm{XY}^\top) = [\mathbf{vec}(\bm{I}_d)]^\top (\bm{X} \otimes \bm{Y}) \mathbf{vec}(\bm{I}_d)$;
\item For any $\bm{X},\bm{Y} \in \mathbb{R}^{d \times d}$, it can be guaranteed that $\exp(\bm{X}) \otimes \exp(\bm{Y}) = \exp(\bm{X} \otimes \bm{I}_{d^2} + \bm{I}_{d^2} \otimes \bm{Y})$.
\end{enumerate}
\end{theorem}
\begin{theorem}\label{thm:Klein}
For any positive definite functions $\bm{A},\bm{B} \in \mathbb{S}^{d \times d}$, it can be guaranteed that
\begin{align*}
\mathsf{Tr}(\bm{A}^{-1}(\bm{A}-\bm{B})) \leq \mathsf{Tr}(\log(\bm{A}) -\log(\bm{B})).
\end{align*}
\end{theorem}

\begin{proof}
This conclusion is a direct application of the Klein's inequality (see, for example, Theorem 2.11 of \cite{Carlen2009TRACEIA}), which states that 
\begin{align}\label{eq:Klein-inequality}
\mathsf{Tr}(f(\bm{A})-f(\bm{B})-(\bm{A} - \bm{B})^\top f'(\bm{A})) \geq 0
\end{align}
for any concave function $f: (0,\infty) \to \mathbb{R}$, on the concave function $f(x) = \log x$. For the specific definition of $f(\bm{X})$ where $\bm{X}$ is a positive semidefinite matrix, and the proof of the inequality \eqref{eq:Klein-inequality}, we refer the readers to \cite{Carlen2009TRACEIA}.
\end{proof}


\subsubsection{Properties of the matrices involved in the paper regarding TD learning}
In this section, we present several useful algebraic results, most of which were established by \cite{wu2024statistical}.
The following lemma reveals several critical algebraic features of the matrix $\bm{A}$.
\begin{lemma}[Lemma A.4 of \cite{wu2024statistical}]\label{lemma:A}
Let $\bm{A}_t,\bm{A}$ be defined as in \eqref{eq:defn-At} and \eqref{eq:defn-At-mean} respectively, and $\bm{\Sigma}$ be defined as in \eqref{eq:defn-Sigma} with $\lambda_0$ and $\lambda_{\Sigma}$ being its smallest and largest eigenvalues. Then the following features hold:
\begin{align}
&2(1-\gamma) \bm{\Sigma} \preceq \bm{A} + \bm{A}^\top \preceq 2(1+\gamma) \bm{\Sigma}; \label{eq:lemma-A-1} \\ 
&\min_{1 \leq i \leq d} \mathsf{Re}(\lambda_i(\bm{A})) \geq (1-\gamma) \lambda_0; \label{eq:lemma-A-2}\\ 
&\mathbb{E}[\bm{A}_t^\top \bm{A}_t] \preceq \bm{A} + \bm{A}^\top; \label{eq:lemma-A-3}\\ 
&\bm{A}^\top \bm{A} \preceq \lambda_{\Sigma}(\bm{A} + \bm{A}^\top); \label{eq:lemma-A-4}\\
&\left\|\bm{I}-\eta \bm{A}\right\| \leq 1-\frac{1-\gamma}{2}\lambda_0 \eta, \quad  \forall \eta \in \left(0,\frac{1}{2\lambda_{\Sigma}}\right); \label{eq:lemma-A-5}\\
&\|\bm{A}^{-1}\| \leq \frac{1}{\lambda_0(1-\gamma)}.\label{eq:lemma-A-6}
\end{align}
\end{lemma}
Notice that the bound \eqref{eq:lemma-A-6} directly implies that
\begin{align}\label{eq:theta-norm-bound}
\|\bm{\theta}^\star\|_2 \leq \|\bm{A}^{-1}\|\|\bm{b}\|_2 \leq \frac{1}{\lambda_0(1-\gamma)}.
\end{align}

\begin{lemma}[Lemma A.5 of \cite{wu2024statistical}]\label{lemma:Lyapunov}
For any matrix $\bm{E} \succeq \bm{0}$, there exists a unique positive definite matrix $\bm{X} \succeq \bm{0}$, such that
\begin{align*}
\bm{AX} + \bm{XA}^\top  = \bm{E}.
\end{align*}
Furthermore, it can be guaranteed that
\begin{align*}
\|\bm{X}\| \leq \frac{1}{2(1-\gamma)\lambda_0} \|\bm{E}\|, \quad \text{and} \quad \mathsf{Tr}(\bm{X}) \leq \frac{1}{2(1-\gamma)\lambda_0} \mathsf{Tr}(\bm{E}).
\end{align*}
\end{lemma}

The following lemma further features the compressive power of the matrix $\prod_{k=1}^t (\bm{I} - \eta_k \bm{A})$.
\begin{lemma}[Lemma A.6 of \cite{wu2024statistical}]\label{lemma:R}
Assume that $\beta \in (0,1)$ and $\alpha \in (\frac{1}{2},1)$. Let $\eta_t = \eta_0 t^{-\alpha}$ for all positive integer $t$, then the following inequalities hold:
\begin{align}
&t^{\alpha} \prod_{k=1}^t \left(1-\beta k^{-\alpha} \right) \leq e \left(\frac{\alpha}{\beta}\right)^{\frac{\alpha}{1-\alpha}} ;\label{eq:lemma-R-1} \\ 
&\sum_{i=1}^t i^{-\nu} \prod_{k=i+1}^t \left(1-\beta k^{-\alpha} \right) \leq \frac{1}{\nu-1} \left(\frac{2(\nu-\alpha)}{\beta}\right)^{\frac{\nu-\alpha}{1-\alpha}} t^{\alpha-\nu}, \quad \forall \nu \in (1,\alpha+1]; \label{eq:lemma-R-2} \\ 
&\sum_{i=1}^t i^{-2\alpha} \prod_{k=i+1}^t \left(1-\beta k^{-\alpha} \right)  = \frac{1}{\beta}t^{-\alpha} + o\left(t^{-\alpha}\right); \label{eq:lemma-R-3} \\ 
&\max_{1 \leq i \leq t} i^{-\alpha} \prod_{k=i+1}^t \left(1-\beta k^{-\alpha} \right) \leq e \left(\frac{\alpha}{\beta}\right)^{\frac{\alpha}{1-\alpha}} t^{-\alpha}. \label{eq:lemma-R-4}
\end{align} 
\end{lemma}

\begin{lemma}[Lemma A.7 of \cite{wu2024statistical}]\label{lemma:Q-uni}
For every $\beta \in (0,1)$, $\alpha \in (\frac{1}{2},1)$ and sufficiently large $T$, it can be guaranteed that
\begin{align}
&t^{-\alpha} \sum_{j=t}^T \prod_{k=t+1}^j (1-\beta k^{-\alpha}) < 3\left(\frac{2}{\beta}\right)^{\frac{1}{1-\alpha}}, \quad \text{and} \label{eq:Q-uni-1}\\ 
&t^{-\alpha} \sum_{j=t}^\infty \prod_{k=t+1}^j (1-\beta k^{-\alpha}) - \frac{1}{\beta} \lesssim  \left(\frac{1}{\beta}\right)^{\frac{1}{1-\alpha}}\Gamma\left(\frac{1}{1-\alpha}\right) t^{\alpha-1}. \label{eq:Q-uni-2}
\end{align}
\end{lemma}

\begin{lemma}[Lemma A.8 of \cite{wu2024statistical}]\label{lemma:Q}
Let $\eta_t = \eta_0 t^{-\alpha}$ with $\alpha \in (\frac{1}{2},1)$ and $0 \leq \eta_0 \leq \frac{1}{2\lambda_{\Sigma}}$. For every $t \in [T]$, define $\bm{Q}_t$ as
\begin{align}
\bm{Q}_t = \eta_t \sum_{j=t}^{T}\prod_{k=t+1}^{j} (\bm{I}-\eta_k \bm{A}) \label{eq:defn-Qt} 
\end{align}

Then for any positive integer $\ell < T$, the following relation hold:
\begin{subequations}
\begin{align}
\sum_{t=1}^T (\bm{Q}_t - \bm{A}^{-1}) = -\bm{A}^{-1} \sum_{j=1}^T \prod_{k=1}^j (\bm{I}-\eta_k \bm{A}),\label{eq:Qt-Ainv-bound-1}
\end{align}
Furthermore, there exist a sequence of matrix $\{\bm{S}_t\}_{1 \leq t \leq T}$, such that the difference between $\bm{Q}_t$ and $\bm{A}^{-1}$ can be represented as
\begin{align}\label{eq:Qt-Ainv-decompose}
\bm{Q}_t - \bm{A}^{-1}  = \bm{S}_t- \bm{A}^{-1} \prod_{k=t}^T (\bm{I}-\eta_k \bm{A}),
\end{align}
in which $\bm{S}_t$ is independent of $T$, and its norm is bounded by
\begin{align}\label{eq:St1-bound}
\|\bm{S}_t\| \lesssim \eta_0 \Gamma\left(\frac{1}{1-\alpha}\right)\left(\frac{2}{(1-\gamma)\lambda_0 \eta_0}\right)^{\frac{1}{1-\alpha}}t^{\alpha-1}.
\end{align}
\end{subequations}
\end{lemma}


\begin{lemma}[Lemma A.10 of \cite{wu2024statistical}]\label{lemma:Q-bound}
Let $\eta_t = \eta_0 t^{-\alpha}$ with $\alpha \in (\frac{1}{2},1)$ and $0 \leq \eta \leq \frac{1}{2\lambda_{\Sigma}}$. For $\bm{Q}_t$ defined as in \eqref{eq:defn-Qt}, it can be guaranteed that
\begin{align}
&\|\bm{Q}_t\| \leq 3\eta_0^{-\frac{\alpha}{1-\alpha}} \left(\frac{4\alpha}{\lambda_0(1-\gamma)}\right)^{\frac{1}{1-\alpha}}, \quad \forall t \in [T];  \label{eq:Q-bound} \\ 
&\|\bm{A} \bm{Q}_t\| \leq 2 + \eta_0 \left(\frac{1}{\beta}\right)^{\frac{1}{1-\alpha}}\Gamma \left(\frac{1}{1-\alpha}\right)t^{\alpha-1}; \label{eq:AQ-bound} \quad \text{and}\\ 
&\frac{1}{T}\sum_{t=1}^T \|\bm{Q}_t - \bm{A}^{-1}\|^2 \leq \frac{1}{\lambda_0(1-\gamma)}T^{\alpha-1} + \widetilde{C}T^{2\alpha-2}\label{eq:AQ-I-bound}
\end{align}
where $\beta = \frac{1-\gamma}{2}\lambda_0 \eta_0$  and $\widetilde{C}$ depends on $\alpha,\eta_0,\lambda_0,\gamma$.
\end{lemma}

The following lemma bounds the norms and trace of the time-averaging variance matrix $\tilde{\bm{\Gamma}}$.
\begin{lemma}\label{lemma:Gamma}
Under Assumption \ref{as:mixing}, when $\tilde{\bm{\Gamma}}$ is defined as in \eqref{eq:defn-tilde-Gamma}, it can be guaranteed that
\begin{align*}
\|\tilde{\bm{\Gamma}}\| \leq \|\tilde{\bm{\Gamma}}\|_{\mathsf{F}} \leq \mathsf{Tr}(\tilde{\bm{\Gamma}}) \leq \left(1+\frac{2m}{1-\rho}\right) (2\|\bm{\theta}^\star\|_2+1)^2 \lesssim \frac{m}{1-\rho}\left(\frac{1}{\lambda_0(1-\gamma)}\right)^2.
\end{align*}
\end{lemma}
\begin{proof} 
The first two inequalities follow directly from the fact that $\tilde{\bm{\Gamma}}$, a variance-covariance matrix, is positive semi-definite; the last inequality follows directly from \eqref{eq:theta-norm-bound}. It now boils down to proving that
\begin{align*}
\mathsf{Tr}(\tilde{\bm{\Gamma}}) \leq \left(1+\frac{2m}{1-\rho}\right) (2\|\bm{\theta}^\star\|_2+1).
\end{align*}
For every integer $\ell>1$, we firstly use the iteration of expectations to derive
\begin{align*}
&\mathsf{Tr}(\mathbb{E}[(\bm{A}_1 \bm{\theta}^\star - \bm{b}_1)(\bm{A}_{\ell+1} \bm{\theta}^\star - \bm{b}_{\ell+1})^\top]) \\ 
&= \mathbb{E}_{s_0 \sim \mu,s_1 \sim P(\cdot \mid s_0), s_{\ell} \sim P^{\ell-1}(\cdot \mid s_1),s_{\ell+1} \sim P(\cdot \mid s_{\ell})}[(\bm{A}_1 \bm{\theta}^\star - \bm{b}_1)^\top (\bm{A}_{\ell+1} \bm{\theta}^\star - \bm{b}_{\ell+1})] \\ 
&= \mathbb{E}_{s_0 \sim \mu,s_1 \sim P(\cdot \mid s_0)}[\mathbb{E}_{s_\ell \sim P^{\ell-1}(\cdot \mid s_1)}[\mathbb{E}_{s_{\ell+1} \sim P(\cdot \mid s_{\ell})} (\bm{A}_1 \bm{\theta}^\star - \bm{b}_1)^\top (\bm{A}_{\ell+1} \bm{\theta}^\star - \bm{b}_{\ell+1}) \mid\mathscr{F}_1]];
\end{align*}
meanwhile, by definition,
\begin{align*}
&\mathbb{E}_{s_\ell \sim \mu}[\mathbb{E}_{s_{\ell+1} \sim P(\cdot \mid s_{\ell})} (\bm{A}_1 \bm{\theta}^\star - \bm{b}_1)^\top (\bm{A}_{\ell+1} \bm{\theta}^\star - \bm{b}_{\ell+1}) \mid\mathscr{F}_1] \\ 
&=(\bm{A}_1 \bm{\theta}^\star - \bm{b}_1)^\top \mathbb{E}_{s_\ell \sim \mu,s_{\ell+1} \sim P(\cdot \mid s_{\ell})}[\bm{A}_{\ell+1} \bm{\theta}^\star - \bm{b}_{\ell+1}] = 0.
\end{align*}
Therefore, the basic property of TV distance and Assumption \ref{as:mixing} guarantees
\begin{align*}
&\mathbb{E}_{s_\ell \sim P^{\ell-1}(\cdot \mid s_1)}[\mathbb{E}_{s_{\ell+1} \sim P(\cdot \mid s_{\ell})} (\bm{A}_1 \bm{\theta}^\star - \bm{b}_1)^\top (\bm{A}_{\ell+1} \bm{\theta}^\star - \bm{b}_{\ell+1}) \mid\mathscr{F}_1] \\ 
&= \left(\mathbb{E}_{s_\ell \sim P^{\ell-1}(\cdot \mid s_1)} - \mathbb{E}_{s_{\ell} \sim \mu}\right)[\mathbb{E}_{s_{\ell+1} \sim P(\cdot \mid s_{\ell})} [(\bm{A}_1 \bm{\theta}^\star - \bm{b}_1)^\top (\bm{A}_{\ell+1} \bm{\theta}^\star - \bm{b}_{\ell+1}) \mid\mathscr{F}_1]] \\ 
&\leq d_{\mathsf{TV}}(P^{\ell-1}(\cdot \mid s_1),\mu) \cdot \sup_{s_{\ell} \in \mathcal{S}} \mathbb{E}_{s_{\ell+1} \sim P(\cdot \mid s_{\ell})} [(\bm{A}_1 \bm{\theta}^\star - \bm{b}_1)^\top (\bm{A}_{\ell+1} \bm{\theta}^\star - \bm{b}_{\ell+1}) \mid\mathscr{F}_1]\\ 
&\leq m\rho^{\ell-1} \cdot (2\|\bm{\theta}^\star\|_2+1)^2.
\end{align*}
By taking expectations with regard to $s_0$ and $s_1$, we obtain
\begin{align*}
&\mathsf{Tr}(\mathbb{E}[(\bm{A}_1 \bm{\theta}^\star - \bm{b}_1)(\bm{A}_{\ell+1} \bm{\theta}^\star - \bm{b}_{\ell+1})^\top]) \\ 
&= \mathbb{E}_{s_0 \sim \mu,s_1 \sim P(\cdot \mid s_0)}[\mathbb{E}_{s_\ell \sim P^{\ell-1}(\cdot \mid s_1)}[\mathbb{E}_{s_{\ell+1} \sim P(\cdot \mid s_{\ell})} (\bm{A}_1 \bm{\theta}^\star - \bm{b}_1)^\top (\bm{A}_{\ell+1} \bm{\theta}^\star - \bm{b}_{\ell+1}) \mid\mathscr{F}_1]] \\ 
&\leq m\rho^{\ell-1} \cdot (2\|\bm{\theta}^\star\|_2+1)^2.
\end{align*}
Hence by definition, the trace of $\tilde{\bm{\Gamma}}$ is bounded by
\begin{align*}
\mathsf{Tr}(\tilde{\bm{\Gamma}}) &= \mathbb{E}\|\bm{A}_1 \bm{\theta}^\star - \bm{b}_1\|_2^2 + 2\sum_{\ell=1}^{\infty} \mathsf{Tr}(\mathbb{E}[(\bm{A}_1 \bm{\theta}^\star - \bm{b}_1)(\bm{A}_{\ell+1} \bm{\theta}^\star - \bm{b}_{\ell+1})^\top]) \\ 
&\leq (1 + 2\sum_{\ell=1}^{\infty} m\rho^{\ell-1}) (2\|\bm{\theta}^\star\|_2+1)^2 = \left(1+\frac{2m}{1-\rho}\right) (2\|\bm{\theta}^\star\|_2+1)^2.
\end{align*}
This completes the proof of the lemma.
\end{proof}

\begin{theorem}[Analogy to Theorem A.11 of \cite{wu2024statistical}]\label{thm:Lambda}
Define $\tilde{\bm{\Lambda}}_T$ as 
\begin{align}\label{eq:defn-tilde-LambdaT}
\tilde{\bm{\Lambda}}_T = \frac{1}{T}\sum_{t=1}^T \bm{Q}_t \tilde{\bm{\Gamma}}\bm{Q}_t^\top,
\end{align}
and let $\tilde{\bm{\Lambda}}^\star$ be defined as in \eqref{eq:defn-tilde-Lambdastar}. Define $\bm{X}(\tilde{\bm{\Lambda}}^\star)$ as the unique solution to the Lyapunov equation
\begin{align}\label{eq:defn-X-Lambda}
\eta_0(\bm{AX+XA}^\top) = \tilde{\bm{\Lambda}}^\star;
\end{align}
then as $T \to \infty$, the difference between $\tilde{\bm{\Lambda}}_T$ and $\tilde{\bm{\Lambda}}^\star$ is featured by
\begin{align*}
&\|\tilde{\bm{\Lambda}}_T - \tilde{\bm{\Lambda}}^\star-T^{\alpha-1}\bm{X}(\tilde{\bm{\Lambda}}^\star)\| \leq \widetilde{C}T^{2\alpha-2}\|\tilde{\bm{\Gamma}}\|, \quad \text{and} \\ 
&\mathsf{Tr}(\tilde{\bm{\Lambda}}_T - \tilde{\bm{\Lambda}}^\star-T^{\alpha-1}\bm{X}(\tilde{\bm{\Lambda}}^\star)) \leq \widetilde{C}T^{2\alpha-2} \mathsf{Tr}(\tilde{\bm{\Gamma}}).
\end{align*}
Here, $\widetilde{C}$ is a problem-related quantity that can be represented by $\eta_0,\alpha$ and $\gamma$.
\end{theorem}
\begin{proof} 
This theorem can be proved by replacing $\bar{\bm{\Lambda}}_T$ by $\tilde{\bm{\Lambda}}_T$, $\bm{\Lambda}^\star$ by $\tilde{\bm{\Lambda}}^\star$ and $\bm{\Gamma}$ by $\tilde{\bm{\Gamma}}$ in the proof of Theorem A.11 in \cite{wu2024statistical}. Details are omitted.
\end{proof}

\begin{lemma}[Analogy to Lemma A.12 of \cite{wu2024statistical}]\label{thm:A-Lambda}
Let $\tilde{\bm{\Lambda}}_T$ and $\tilde{\bm{\Lambda}}^\star$ be defined as in \eqref{eq:defn-tilde-LambdaT} and \eqref{eq:defn-tilde-Lambdastar} respectively. When 
\begin{align}\label{eq:Lambda-T-condition}
T \geq 4 \left(\frac{2}{(1-\gamma)\lambda_0\eta_0}\right)^{\frac{1}{1-\alpha}}(1-\alpha)^{\frac{\alpha}{1-\alpha}}\Gamma(\frac{1}{1-\alpha})\mathsf{cond}(\Gamma),
\end{align}
it can be guaranteed that
\begin{align*}
\lambda_{\min}(\bm{A}\tilde{\bm{\Lambda}}_T \bm{A}^\top) \geq \frac{1}{2}\lambda_{\min}(\tilde{\bm{\Gamma}}) = \frac{1}{2\|\tilde{\bm{\Gamma}}^{-1}\|}.
\end{align*}
\end{lemma}
\begin{proof} 
This lemma can be proved by replacing $\bar{\bm{\Lambda}}_T$ by $\tilde{\bm{\Lambda}}_T$, $\bm{\Lambda}^\star$ by $\tilde{\bm{\Lambda}}^\star$ and $\bm{\Gamma}$ by $\tilde{\bm{\Gamma}}$ in the proof of Lemma A.12 in \cite{wu2024statistical}. Details are omitted.
\end{proof}

\begin{lemma}\label{lemma:delta-Q}
Let $\bm{Q}_t$ be defined as in \eqref{eq:defn-Qt}, then it can be guaranteed that
\begin{align*}
\frac{1}{T}\sum_{t=1}^T \|\bm{Q}_t - \bm{Q}_{t-1}\| \lesssim \eta_0 \left[\eta_0 \Gamma\left(\frac{1}{1-\alpha}\right)+\alpha\right]\left(\frac{2}{(1-\gamma)\lambda_0 \eta_0}\right)^{\frac{1}{1-\alpha}} \frac{\log T}{T}.
\end{align*}
\end{lemma}
\begin{proof} By definition, the matrices $\bm{Q}_{t-1}$ and $\bm{Q}_t$ can be related by
\begin{align*}
\bm{Q}_{t-1}& = \eta_{t-1} \bm{I} + \eta_{t-1} \sum_{j=t}^T \prod_{k=t}^j (\bm{I} - \eta_k \bm{A})\\ 
&= \eta_{t-1} \bm{I} + \frac{\eta_{t-1}}{\eta_t} (\bm{I}-\eta_t \bm{A})\bm{Q}_t;
\end{align*}
Therefore, the difference between $\bm{Q}_t$ and $\bm{Q}_{t-1}$ is featured by
\begin{align*}
\bm{Q}_t - \bm{Q}_{t-1} &= -\eta_{t-1}(\bm{I} - \bm{AQ}_{t}) - \left(\frac{\eta_{t-1}}{\eta_t} - 1\right) \bm{Q}_{t} \\ 
&= \eta_{t-1}\bm{AS}_t - \eta_{t-1}\prod_{k=t}^T (\bm{I}-\eta_k \bm{A}) - \left(\frac{\eta_{t-1}}{\eta_t} - 1\right)\bm{Q}_{t}
\end{align*}
where we invoked \eqref{eq:Qt-Ainv-decompose} in the last equation. Hence by triangle inequality,
\begin{align*}
&\frac{1}{T}\sum_{t=1}^T \|\bm{Q}_t - \bm{Q}_{t-1}\|\\ 
&\leq \frac{1}{T}\|\bm{A}\|\sum_{t=1}^T  \eta_{t-1}\|\bm{S}_t\| + \frac{1}{T}\sum_{t=1}^T \eta_{t-1}\left\|\prod_{k=t}^T (\bm{I}-\eta_k \bm{A}) \right\| + \frac{1}{T}\sum_{t=1}^T\left(\frac{\eta_{t-1}}{\eta_t} - 1\right)\|\bm{Q}_{t}\|.
\end{align*}
By \eqref{eq:St1-bound}, \eqref{eq:Q-bound} and the telescoping method, it is easy to verify that the three terms on the right-hand-side can be bounded respectively by
\begin{align*}
&\frac{1}{T}\|\bm{A}\|\sum_{t=1}^T  \eta_{t-1}\|\bm{S}_t\| \lesssim \eta_0^2 \Gamma\left(\frac{1}{1-\alpha}\right)\left(\frac{2}{(1-\gamma)\lambda_0 \eta_0}\right)^{\frac{1}{1-\alpha}} \frac{\log T}{T}, \\ 
&\frac{1}{T}\sum_{t=1}^T \eta_{t-1}\left\|\prod_{k=t}^T (\bm{I}-\eta_k \bm{A}) \right\| \leq \frac{2}{(1-\gamma)\lambda_0 } \frac{1}{T}, \quad \text{and} \\ 
&\frac{1}{T}\sum_{t=1}^T\left(\frac{\eta_{t-1}}{\eta_t} - 1\right)\|\bm{Q}_{t}\| \lesssim \alpha \left(\frac{2}{(1-\gamma)\lambda_0 \eta_0}\right)^{\frac{1}{1-\alpha}} \frac{\log T}{T}.
\end{align*}
The lemma follows immediately.
\end{proof}

\subsection{Other basic facts}
The following theorem proposed by \cite{devroye2018total} gives an upper bound for the total variation (TV) distance between two Gaussian random variables with same means and different covariance matrices.
\begin{theorem}[\cite{devroye2018total}]\label{thm:DMR}
 Let $\bm{\Lambda}_1, \bm{\Lambda}_2 \in \mathbb{S}^{d \times d}$ be two positive definite matrices, and $\bm{\mu}$ be any vector in $\mathbb{R}^d$. Then the TV distance between Gaussian distributions $\mathcal{N}(\bm{\mu},\bm{\Lambda}_1)$ and $\mathcal{N}(\bm{\mu},\bm{\Lambda}_2)$ is bounded by
\begin{align*}
\min\left\{1,\frac{1}{100}\left\|\bm{\Lambda}_1^{-1/2}\bm{\Lambda}_2 \bm{\Lambda}_1^{-1/2}-\bm{I}_d\right\|_{\mathsf{F}} \right\} 
 \leq d_{\mathsf{TV}}(\mathcal{N}(\bm{\mu},\bm{\Lambda}_1),\mathcal{N}(\bm{\mu},\bm{\Lambda}_2)) 
\leq \frac{3}{2}\left\|\bm{\Lambda}_1^{-1/2}\bm{\Lambda}_2 \bm{\Lambda}_1^{-1/2}-\bm{I}_d\right\|_{\mathsf{F}}. 
\end{align*}
\end{theorem}

\begin{theorem}[\cite{nazarov2003maximal}]\label{thm:Gaussian-reminder}
Let $\bm{z} \sim \mathcal{N}(\bm{0},\bm{\Lambda})$ be a $d$-dimensional Gaussian random variable, and $\mathcal{A}$ be any non-convex subset of $\mathbb{R}^d$. For any $\varepsilon \geq 0$, we define
\begin{align}\label{eq:defn-A-eps}
&\mathcal{A}^{\varepsilon} := \{\bm{x} \in \mathbb{R}^d: \inf_{y \in \mathcal{A}} \|\bm{x} - \bm{y}\|_2 \leq \varepsilon\}, \quad \text{and} \nonumber \\ 
&\mathcal{A}^{-\varepsilon}:=\{\bm{x} \in \mathbb{R}^d: B(\bm{x},\varepsilon) \subset \mathcal{A}\},
\end{align}
where $B(\bm{x},\varepsilon)$ represents the $d$-dimensional ball centered at $\bm{x}$ with radius $\varepsilon$. Then it can be guaranteed that
\begin{align*}
&\mathbb{P}(\bm{z} \in \mathcal{A}^{\varepsilon} - \mathcal{A}) \lesssim \|\bm{\Lambda}^{-1}\|_{\mathsf{F}}^{\frac{1}{2}}\varepsilon, \quad \text{and} \\ 
&\mathbb{P}(\bm{z} \in \mathcal{A} - \mathcal{A}^{-\varepsilon})\lesssim \|\bm{\Lambda}^{-1}\|_{\mathsf{F}}^{\frac{1}{2}}\varepsilon.
\end{align*}
\end{theorem}

The following theorem from \cite{nourdin2021multivariate} relates the distance on convex sets with the Wasserstein distance when one of the distributions being compared corresponds to a Gaussian random variable.
\begin{theorem}\label{thm:Gaussian-convex-Wass}
Let $\bm{x}$ be a random vector in $\mathbb{R}^d$, and $\bm{y} \sim \mathcal{N}(\bm{0},\bm{\Lambda})$ where $\bm{\Lambda} \in \mathbb{S}^{d \times d}$ is positive-definite. Then it can be guaranteed that
\begin{align*}
d_{\mathsf{C}}(\bm{x},\bm{y}) \lesssim \|\bm{\Lambda}^{-1}\|_{\mathsf{F}}^{\frac{1}{4}}\sqrt{d_{\mathsf{W}}(\bm{x},\bm{y})}.
\end{align*}
\end{theorem}

We also recall the following concentration inequalities. 

\begin{theorem}[Corollary A.16 of \cite{wu2024statistical}]\label{thm:vector-Azuma}
Let $\{\bm{x}_i\}_{i \geq 1}$ be a martingale in $\mathbb{R}^d$, and let $W_{\max}$ be a positive constant that satisfies
\begin{align*}
W_{\max} \geq \sum_{i=1}^t \|\bm{x}_i - \bm{x}_{i-1}\|_2^2 \quad \text{almost surely.}
\end{align*}
Then it can be guaranteed with probability at least $1-\delta$ that
\begin{align*}
\|\bm{x}_t\|_2 \leq 2\sqrt{2W_{\max}\log \frac{3}{\delta}}.
\end{align*}
\end{theorem}

\begin{theorem}[Theorem 2.1 of \cite{peng2024advances}]\label{thm:Hilbert-Freedman}
Let $\mathcal{X}$ be a Hilbert space, $\{\bm{X}_i\}_{i=1}^n$ be an $\mathcal{X}$-valued martingale difference sequence adapted to the filtration $\{\mathscr{F}_i\}_{i=1}^n$, $\bm{Y}_i := \sum_{j=1}^i \bm{X}_j$ be the corresponding martingale, and $W_i = \sum_{j=1}^i \mathbb{E}\|\bm{X}_j\|^2\mid\mathscr{F}_{j-1}$. Suppose that $\max_{i \in [n]}\|\bm{X}_i\| \leq b$ for some constant $b > 0$. Then for any $\varepsilon$ and $\sigma>0$, it can be guaranteed that
\begin{align*}
\mathbb{P}\left(\exists k \in [n], \text{ s.t. } \|\bm{Y}_k\|\geq \varepsilon \text{ and } W_k \leq \sigma^2 \right) \leq 2\exp\left(-\frac{\varepsilon^2/2}{\sigma^2 + b\varepsilon/3}\right).
\end{align*}
\end{theorem}
As a direct consequence, under the conditions of Theorem \ref{thm:Hilbert-Freedman}, for every $\delta \in (0,1)$ it can be guaranteed with probability at least $1-\delta$ that
\begin{align*}
W_n > \sigma^2 \quad \text{or} \quad \|\bm{Y}_n\| \lesssim \sqrt{\frac{\sigma^2}{n}\log \frac{1}{\delta}} + \frac{b}{n} \log \frac{1}{\delta}.
\end{align*}

\begin{theorem}[Theorem 1 of \cite{Fan2021Hoeffding}]\label{thm:markov-hoeffding-1d}
Let $\{s_i\}_{i=1}^n$ be a Markov chain on state space $\mathcal{S}$ with a unique stationary distribution $\mu$ and a positive spectral gap $1-\lambda>0$. Let $\{g_i\}_{i=1}^n$ be a sequence of functions mapping from $\mathcal{S}$ to $\mathbb{R}$, bounded by $f_i \in [a_i,b_i]$ respectively. Then when $s_1 \sim \mu$, for every $\varepsilon>0$, it can be guaranteed that
\begin{align*}
\mathbb{P}\left(\sum_{i=1}^n g_i(s_i) - \sum_{i=1}^n \mu(g_i) > \varepsilon \right) \leq \exp\left(-\frac{1-\lambda}{1+\lambda} \frac{\varepsilon^2}{\sum_{i=1}^n (b_i-a_i)^2/4}\right).
\end{align*}
\end{theorem}
The following corollary follows rather directly by combining Theorem \ref{thm:markov-hoeffding-1d} with Holder's inequality.

\begin{corollary}\label{cor:markov-hoeffding-1d}
Under the conditions of Theorem \ref{thm:markov-hoeffding-1d}, let $\nu$ be a distribution on $\mathcal{S}$ satisfying Assumption \ref{as:nu} with parameters $(p,q)$, then when $s_1 \sim \nu$, for every $\varepsilon>0$, it can be guaranteed that
\begin{align*}
\mathbb{P}\left(\sum_{i=1}^n g_i(s_i) - \sum_{i=1}^n \mu(g_i) > \varepsilon \right) \leq \left\|\frac{d\mathrm{\nu}}{\mathrm{d}\mu}\right\|_{\mu,p} \exp\left(-\frac{1-\lambda}{(1+\lambda)q} \frac{\varepsilon^2}{\sum_{i=1}^n (b_i-a_i)^2/4}\right).
\end{align*}
\end{corollary}
Equivalently, for any $\delta \in (0,1)$, it can be guaranteed with probability at least $1-\delta$ that
\begin{align*}
\sum_{i=1}^n g_i(s_i) - \sum_{i=1}^n \mu(g_i) \lesssim \log^{\frac{1}{2}} \left(\left\|\frac{d\mathrm{\nu}}{\mathrm{d}\mu}\right\|_{\mu,p} \frac{1}{\delta}\right) \cdot \sqrt{\frac{q}{1-\lambda} \sum_{i=1}^n (b_i-a_i)^2}.
\end{align*}

\section{Proof of theoretical results regarding general Markov chains}

\subsection{Proof of Theorem \ref{thm:matrix-hoeffding}}\label{app:proof-matrix-hoeffding}
For simplicity, we firstly consider the case that $\nu = \mu$, and then generalize our result to any distribution $\nu$ that satisfies Assumption \ref{as:nu}. 
A classic Chernoff argument indicates
\begin{align}\label{eq:matrix-chernoff}
\mathbb{P}_{s_1 \sim \mu}\left(\left\|\frac{1}{n}\sum_{i=1}^n \bm{F}_i(s_i)\right\| \geq \varepsilon \right) &\leq 2\inf_{t \geq 0} \exp(-nt\varepsilon) \mathbb{E}\left[\mathsf{Tr}\left(\exp\left(t\sum_{i=1}^n \bm{F}_i(s_i)\right)\right)\right].
\end{align}
In order to bound the right-hand-side, \cite{garg2018matrixexpanderchernoff} illustrated in their Equation (11), through an application of the multi-matrix Golden-Thompson inequality, that there exists a probability distribution $\phi$ on the interval $[-\frac{\pi}{2},\frac{\pi}{2}]$, such that
\begin{align}\label{eq:multi-matrix-golden-thompson}
\mathsf{Tr}\left(\exp\left(t\sum_{i=1}^n \bm{F}_i(s_i)\right)\right) \leq d^{1-\frac{\pi}{4}}\int_{-\frac{\pi}{2}}^{\frac{\pi}{2}}\mathsf{Tr}\left[\prod_{i=1}^n \exp\left(\frac{2}{\pi}e^{\mathbf{i}\theta}t\bm{F}_i(s_i)\right)\prod_{i=n}^1 \exp\left(\frac{2}{\pi}e^{-\mathbf{i}\theta}t\bm{F}_i(s_i)\right)\right]\mathrm{d}\phi(\theta).
\end{align}
Furthermore, by repeatedly applying the basic properties of Kronecker products, the trace on the right-hand-side can be computed as 
\begin{align}\label{eq:matrix-hoeffding-kronecker}
&\mathsf{Tr}\left[\prod_{i=1}^n \exp\left(\frac{2}{\pi}e^{\mathbf{i}\theta}t\bm{F}_i(s_i)\right)\prod_{i=n}^1 \exp\left(\frac{2}{\pi}e^{-\mathbf{i}\theta}t\bm{F}_i(s_i)\right)\right] \nonumber \\ 
&= [\mathbf{vec}(\bm{I}_d)]^\top \left\{\prod_{i=n}^1 \exp\left(\frac{2}{\pi}e^{-\mathbf{i}\theta}t\bm{F}_i(s_i)\right) \otimes  \prod_{i=n}^1\exp\left(\frac{2}{\pi}e^{\mathbf{i}\theta}t\bm{F}_i(s_i)\right)\right\}\mathbf{vec}(\bm{I}_d) \nonumber \\ 
&= [\mathbf{vec}(\bm{I}_d)]^\top \prod_{i=n}^1\left\{\exp\left(\frac{2}{\pi}e^{-\mathbf{i}\theta}t\bm{F}_i(s_i)\right) \otimes \exp\left(\frac{2}{\pi}e^{\mathbf{i}\theta}t\bm{F}_i(s_i)\right)\right\}\mathbf{vec}(\bm{I}_d) \nonumber \\ 
&= [\mathbf{vec}(\bm{I}_d)]^\top \prod_{i=n}^1 \exp(t\bm{H}_i(s_i,\theta))\mathbf{vec}(\bm{I}_d),
\end{align}
where we define, for simplicity,
\begin{align}\label{eq:defn-Hi}
\bm{H}_i(s_i,\theta):= \frac{2}{\pi}e^{-\mathbf{i}\theta}\bm{F}_i(s_i) \otimes \bm{I}_{d^2} + \bm{I}_{d^2} \otimes \frac{2}{\pi}e^{\mathbf{i}\theta}\bm{F}_i(s_i).
\end{align}
Through the deduction of \eqref{eq:matrix-hoeffding-kronecker}, we applied properties (6), (1) and (7) in Theorem \ref{thm:Kronecker} in the second, third and fourth line respectively. 
It is easy to verify that the matrix function $\bm{H}_i$ defined as in \eqref{eq:defn-Hi} has the following properties:
\begin{align}\label{eq:Hi-properties}
&\mathbb{E}_{s_i \sim \mu} [\bm{H}_i(s_i,\theta)] = \bm{0}, \quad \forall i \in [n], \theta \in \left[-\frac{\pi}{2},\frac{\pi}{2}\right];\\ 
&\|\bm{H}_i(s_i,\theta)\| \leq \frac{4}{\pi} M_i, \quad \text{a.s.}, \quad \forall i \in [n],  \theta \in \left[-\frac{\pi}{2},\frac{\pi}{2}\right].
\end{align}

As a combination of \eqref{eq:matrix-chernoff}, \eqref{eq:multi-matrix-golden-thompson} and \eqref{eq:matrix-hoeffding-kronecker}, our goal is to bound
\begin{align*}
&\mathbb{P}\left(\left\|\frac{1}{n}\sum_{i=1}^n \bm{F}_i(s_i)\right\| \geq \varepsilon \right)  \\ 
&\leq 2d^{1-\frac{\pi}{4}} \inf_{t \geq 0} \exp(-nt\varepsilon) [\mathbf{vec}(\bm{I}_d)]^\top \mathbb{E}_{\theta \sim \phi} \mathbb{E}_s\left[\prod_{i=n}^1 \exp(t\bm{H}_i(s_i,\theta))\mathbf{vec}(\bm{I}_d)\right] \\ 
&\leq 2d^{1-\frac{\pi}{4}} \inf_{t \geq 0}  \sup_{\theta \in [-\frac{\pi}{2},\frac{\pi}{2}]} \exp(-nt\varepsilon) [\mathbf{vec}(\bm{I}_d)]^\top \mathbb{E}_s\left[\prod_{i=n}^1 \exp(t\bm{H}_i(s_i,\theta))\mathbf{vec}(\bm{I}_d)\right]
\end{align*}
Here, we use $\mathbb{E}_s$ to denote the expectation taken over all the samples in the Markov chain $s_1,s_2,\ldots,s_n$, where $s_1 \sim \mu$ and $s_{i+1} \sim P(\cdot \mid s_i)$ for all $1 \leq i < n$. In order to bound this expectation, we define, for any $t>0$ and $\theta  \in [-\frac{\pi}{2},\frac{\pi}{2}]$, a sequence of vector-valued functions $\{\bm{g}_k\}_{k \in [n]}$ taking values in $\mathbb{R}^{d^2}$, where 
\begin{align*}
&\bm{g}_0 (x) = \textbf{vec}(\bm{I}_d), \quad \text{and} \\ 
&\bm{g}_k (x) = \mathbb{E}\left[\prod_{i=k}^1 \exp(t\bm{H}_i(s_i,\theta))\textbf{vec}(\bm{I}_d)\bigg| s_{k+1} = x\right], \quad k \geq 1.
\end{align*}
In this way, the left-hand-side of \eqref{eq:markov-matrix-hoeffding} is bounded by
\begin{align*}
\mathbb{P}_{s_1 \sim \mu}\left(\left\|\frac{1}{n}\sum_{i=1}^n \bm{F}_i(s_i)\right\| \geq \varepsilon \right) &\leq   2d^{1-\frac{\pi}{4}} \inf_{t \geq 0}  \sup_{\theta \in [-\frac{\pi}{2},\frac{\pi}{2}]} \exp(-nt\varepsilon) [\mathbf{vec}(\bm{I}_d)]^\top \mu(\bm{g}_n) \\ 
&\leq 2d^{1-\frac{\pi}{4}} \inf_{t \geq 0}  \sup_{\theta \in [-\frac{\pi}{2},\frac{\pi}{2}]} \exp(-nt\varepsilon) \left\|\mathbf{vec}(\bm{I}_d)\right\|_2 \cdot \|\mu(\bm{g}_n)\|_2 \\ 
&= 2d^{1-\frac{\pi}{4}} \inf_{t \geq 0}  \sup_{\theta \in [-\frac{\pi}{2},\frac{\pi}{2}]} \exp(-nt\varepsilon) \sqrt{d} \cdot \|\bm{g}_n^{\parallel}\|_{\mu}
\end{align*}
where we applied the fact that $s_{n+1} \sim \mu$ when $s_1 \sim \mu$. Recall that by definition, $\mu(\bm{g}_n) = \mathbb{E}_{x \sim \mu}[\bm{g}_n(x)]$ and $\bm{g}_n^{\parallel}$ represents the constant function taking this value.
Next, we construct a recursive relation among the sequence $\{\bm{g}_k\}_{k \in [n]}$ for the purpose of bounding the norm of $\bm{g}_n$. Notice that, on one hand,
\begin{align*}
\exp\left(t\bm{H}_k(x,\theta)\right)\bm{g}_{k-1}(x) = \mathbb{E}\left[\prod_{i=k}^1 \exp(t\bm{H}_i(s_i))\bigg| s_{k} = x\right];
\end{align*}
on the other hand, 
\begin{align*}
\bm{g}_k(x) &= \mathbb{E}\left[\prod_{i=k}^1 \exp(t\bm{H}_i(s_i,\theta))\textbf{vec}(\bm{I}_d)\bigg| s_{k+1} = x\right] \\ 
&=  \int_{y \in \mathcal{S}} \mathbb{E}\left[\prod_{i=k}^1 \exp(t\bm{H}_i(s_i,\theta))\textbf{vec}(\bm{I}_d)\bigg| s_{k} = y\right] \mathrm{d}\mathbb{P}(s_k = y\mid s_{k+1} = x) \\ 
&= \int_{y \in \mathcal{S}} \left\{\mathbb{E}\left[\prod_{i=k}^1 \exp(t\bm{H}_i(s_i,\theta))\textbf{vec}(\bm{I}_d)\bigg| s_{k} = y\right] \right\}P^*(x,\mathrm{d}y);
\end{align*}
here the contents in the curly bracket is equal to $\exp\left(t\bm{H}_k(y,\theta)\right)\bm{g}_{k-1}(u)$. Therefore, the sequence $\{\bm{g}_k\}$ is featured by the recursive relation
\begin{align}\label{eq:markov-matrix-hoeffding-G-recursive}
\bm{g}_{k}(x) &= \int_{y \in \mathcal{S}}\exp\left(t\bm{H}_k(y,\theta)\right)\bm{g}_{k-1}(y) P^*(x,\mathrm{d}y) \nonumber \\ 
&= (\mathcal{P}^* (\exp(t\bm{H}_k)\bm{g}_{k-1}))(x).
\end{align}
Based on this recursive relationship, the following proposition comes in handy for bounding the norm of $\bm{g}_k$ recursively.

\begin{proposition}\label{prop:PE}
For any matrix function $\bm{g}: \mathcal{S} \to \mathbb{R}^{m}$ and any bounded symmetric matrix function $\bm{H}: \mathcal{S} \to \mathbb{S}^{m \times m}$ with  $\mu(\bm{H}) = \bm{0}$ and $\|\bm{F}(x)\| \leq M$ almost surely. Then the following hold for any $t > 0$:
\begin{align}
&\left\|\left(\mathcal{P}^* \exp(t\bm{H})\bm{g}^{\parallel}\right)^{\parallel}\right\|_{\mu} \leq \alpha_1 \|\bm{g}^{\parallel}\|_{\mu},\quad \text{where} \quad \alpha_1 = \exp(tM)-tM;\label{eq:PE1}\\
&\left\|\left(\mathcal{P}^* \exp(t\bm{H})\bm{g}^{\parallel}\right)^{\perp}\right\|_{\mu} \leq \alpha_2 \|\bm{g}^{\parallel}\|_{\mu},\quad \text{where} \quad \alpha_2 = \lambda(\exp(tM)-1);\label{eq:PE2}\\ 
&\left\|\left(\mathcal{P}^* \exp(t\bm{H})\bm{g}^{\perp}\right)^{\parallel}\right\|_{\mu} \leq \alpha_3 \|\bm{g}^{\perp}\|_{\mu},\quad \text{where} \quad \alpha_3 = \exp(tM)-1; \label{eq:PE3}\\ 
&\left\|\left(\mathcal{P}^* \exp(t\bm{H})\bm{g}^{\perp}\right)^{\perp}\right\|_{\mu} \leq \alpha_4 \|\bm{g}^{\parallel}\|_{\mu},\quad \text{where} \quad \alpha_4 = \lambda \exp(tM).\label{eq:PE4}
\end{align}
\end{proposition}
\begin{proof} 
See Appendix \ref{proof-prop-PE}.
\end{proof}

A recursive relation can be constructed to bound the norm of $\bm{g}_n^{\parallel}$. Specifically, for every $k \in [n]$, \eqref{eq:markov-matrix-hoeffding-G-recursive},the triangle inequality and \eqref{eq:PE1}, \eqref{eq:PE2} guarantee that
\begin{align*}
\|\bm{g}_k^{\parallel}\|_{\mu} &= \left\|(\mathcal{P}^* \exp(t\bm{H}_k)\bm{g}_{k-1})^{\parallel}\right\|_{\mu} \\ 
&\leq \left\|(\mathcal{P}^* \exp(t\bm{H}_k)\bm{g}_{k-1}^{\parallel})^{\parallel}\right\|_{\mu} + \left\|(\mathcal{P}^* \exp(t\bm{H}_k))\bm{g}_{k-1}^{\perp})^{\parallel}\right\|_{\mu} \\ 
&\leq \alpha_{1k} \|\bm{g}_{k-1}^{\parallel}\|_{\mu} + \alpha_{2k} \|\bm{g}_{k-1}^{\perp}\|_{\mu};
\end{align*}
here, we define $\alpha_{1k} = \exp(\frac{4}{\pi}tM_k) - \frac{4}{\pi}tM_k$ and $\alpha_{2k} = \lambda(\exp(\frac{4}{\pi}tM_k)-1)$. In order to iteratively bound the norm of $\bm{g}_k^{\parallel}$, we also need to bound the norm of $\bm{g}_{k-1}^{\perp}$. Towards this end, we observe, due to \eqref{eq:markov-matrix-hoeffding-G-recursive}, the triangle inequality and \eqref{eq:PE3}, \eqref{eq:PE4} that for every $k \in [n]$,
\begin{align*}
\|\bm{g}_k^{\perp}\|_{\mu} &= \left\|(\mathcal{P}^* \exp(t\bm{H}_k))\bm{g}_{k-1})^{\perp}\right\|_{\mu} \\ 
&\leq \left\|(\mathcal{P}^* \exp(t\bm{H}_k))\bm{g}_{k-1}^{\parallel})^{\perp}\right\|_{\mu} + \left\|(\mathcal{P}^* \exp(t\bm{H}_k))\bm{g}_{k-1}^{\perp})^{\perp}\right\|_{\mu} \\ 
&\leq \alpha_{3k} \|\bm{g}_{k-1}^{\parallel}\|_{\mu} + \alpha_{4k} \|\bm{g}_{k-1}^{\perp}\|_{\mu}.
\end{align*}
Here again, we define $\alpha_{3k} = \exp(\frac{4}{\pi}tM_k) - 1$ and $\alpha_{4k} = \lambda \exp(\frac{4}{\pi}tM_k)$. For simplicity, we denote
\begin{align*}
&\bm{x}_0 = \begin{pmatrix}
\|\bm{g}_0^{\parallel}\|_{\mu} \\ 
\|\bm{g}_0^{\perp}\|_{\mu} 
\end{pmatrix}=\begin{pmatrix}
\sqrt{d} \\ 
0
\end{pmatrix}, \quad \text{and} \\ 
&\bm{x}_k = \underset{\bm{U}_j}{\underbrace{\begin{pmatrix}
\alpha_{1k} & \alpha_{3k} \\ 
\alpha_{3k} & \alpha_{4k}
\end{pmatrix}}} \bm{x}_{k-1}, \quad \forall k \in [n].
\end{align*}
It can then be easily verified through an induction argument that 
\begin{align*}
\|\bm{g}_n^{\parallel}\|_{\mu} \leq x_{n1}.
\end{align*}
Notice here that we applied the fact $\alpha_{2k}< \alpha_{3k}$, since $\lambda < 1$. Consequently, 
\begin{align*}
\|\bm{g}_n^{\parallel}\|_{\mu} &\leq x_{n1} 
\leq \|\bm{x}_n\|_2 
= \left\|\prod_{k=1}^n \bm{U}_k \bm{x}_0\right\|_2 
 \leq \prod_{k=1}^n \|\bm{U}_k\| \|\bm{x}_0\|_2 =  \sqrt{d} \cdot \prod_{k=1}^n \|\bm{U}_k\|.
\end{align*}
Since $\bm{U}_k$ is a symmetric $2 \times 2$ matrix, its operator norm is featured by
\begin{align*}
\|\bm{U}_k\| = \max\{|\sigma_{k1}|,|\sigma_{k2}|\}
\end{align*}
where $\sigma_{k1}$ and $\sigma_{k2}$ are the two eigenvalues of $\bm{U}_k$. Recall from elementary linear algebra that $\sigma_{j1}$ and $\sigma_{j2}$ are solutions to the equation
\begin{align*}
(\alpha_{1k} - x)(\alpha_{4k}-x) -\alpha_{3k}^2 = 0.
\end{align*}
Since $\alpha_{1k} > 0$ and $\alpha_{4k} > 0$, it can be easily verified that
\begin{align*}
\|\bm{U}_k\| = \frac{\alpha_{1k} + \alpha_{4k}}{2} + \frac{\sqrt{(\alpha_{1k} - \alpha_{4k})^2 + 4\alpha_{3k}^2}}{2}.
\end{align*}
The following lemma comes in handy for bounding the norm of $\bm{U}_k$.
\begin{lemma}\label{lemma:Uk}
For any $x>0$ and $\lambda \in (0,1)$, denote $\alpha_1 = e^x - x$, $\alpha_3 = e^x - 1$ and $\alpha_4 = \lambda e^x$. It can then be guaranteed that
\begin{align*}
\frac{\alpha_1 + \alpha_4}{2} + \frac{\sqrt{(\alpha_1 - \alpha_4)^2 + 4\alpha_3^2}}{2} \leq \exp\left(\frac{5}{1-\lambda}x^2\right).
\end{align*}
\end{lemma}  
\begin{proof} 
See Appendix \ref{app:proof-lemma-Uk}.
\end{proof}

As a direct implication of lemma \ref{lemma:Uk}, the norm of $\bm{U}_k$ is bounded by
\begin{align*}
\|\bm{U}_k\| \leq \exp\left(\frac{5}{1-\lambda} (\frac{4}{\pi}t)^2 M_k^2\right);
\end{align*}
consequently, the norm of $\bm{g}_n^{\parallel}$ is bounded by
\begin{align*}
\|\bm{g}_n^{\parallel}\|_{\mu} &\leq \sqrt{d}\exp\left(\left(\frac{4}{\pi}\right)^2\frac{5t^2}{1-\lambda} \sum_{k=1}^n M_k^2 \right) 
\end{align*}
for any $t \geq 0$ and $\theta \in \left[-\frac{\pi}{2},\frac{\pi}{2}\right]$. Therefore, the left-hand-side of \eqref{eq:markov-matrix-hoeffding} is bounded by
\begin{align*}
\mathbb{P}_{s_1 \sim \mu}\left(\left\|\frac{1}{n}\sum_{i=1}^n \bm{F}_i(s_i)\right\| \geq \varepsilon \right) &\leq   2d^{2-\frac{\pi}{4}} \inf_{t \geq 0}  \exp(-nt\varepsilon)\exp\left(\left(\frac{4}{\pi}\right)^2\frac{5t^2}{1-\lambda} \sum_{k=1}^n M_k^2 \right).
\end{align*}

By letting
\begin{align*}
t = \frac{1-\lambda}{10} \left(\frac{\pi}{4}\right)^2 \frac{n}{\sum_{k=1}^n M_k^2} \cdot \varepsilon,
\end{align*}
we obtain
\begin{align*}
\mathbb{P}_{s_1 \sim \mu}\left(\left\|\frac{1}{n}\sum_{i=1}^n \bm{F}_i(s_i)\right\| \geq \varepsilon \right) &\leq 2d^{2-\frac{\pi}{4}} \exp\left\{-\frac{1-\lambda}{20}\left(\frac{\pi}{4}\right)^2\frac{n^2\varepsilon^2 }{\sum_{k=1}^n M_k^2} \right\}
\end{align*}
which completes the analysis for the case that $\nu = \mu$. In order to extend this result to general $\nu$, we define functions $f(x)$ and $g(x)$ as 
\begin{align*}
f(x):= \mathbb{P}\left(\left\|\frac{1}{n}\sum_{i=1}^n \bm{F}_i(s_i)\right\| \geq \varepsilon \Bigg| s_1 = x\right), \quad \text{and} \quad g(x) := \frac{\nu(\mathrm{d}x)}{\mu(\mathrm{d}x)}
\end{align*}
respectively. Then by definition, $g(x) \geq 0$ and $f(x) \in [0,1]$ for all $x \in \mathcal{S}$. Therefore, the Holder's inequality guarantees
\begin{align}\label{eq:matrix-hoeffding-holder}
\mathbb{P}_{s_1 \sim \nu}\left(\left\|\frac{1}{n}\sum_{i=1}^n \bm{F}_i(s_i)\right\| \geq \varepsilon \right)&=\int_{\mathcal{S}}f(x) \nu(\mathrm{d}x)\nonumber \\
&= \int_{\mathcal{S}}f(x) g(x) \mu(\mathrm{d}x) \nonumber \\ 
&=\|fg\|_{\mu,1} \leq \|f\|_{\mu,q} \|g\|_{\mu,p} = \|f\|_{\mu,q}\left\|\frac{\mathrm{d}\nu}{\mathrm{d}\mu}\right\|_{\mu,p},
\end{align}
where, since $q > 1$ and $f(x) \in [0,1]$ almost surely, $\|f\|_{\mu,q}$ is bounded by
\begin{align}\label{eq:matrix-hoeffding-qbound}
\|f\|_{\mu,q} &= \left(\int_{\mathcal{S}}f^q(x)\mathrm{d}\mu \right)^{\frac{1}{q}}  
\leq \left(\int_{\mathcal{S}}f(x)\mathrm{d}\mu \right)^{\frac{1}{q}}  
\leq \left[\mathbb{P}_{s_1 \sim \mu}\left(\left\|\frac{1}{n}\sum_{i=1}^n \bm{F}_i(s_i)\right\| \geq \varepsilon \right)\right]^{\frac{1}{q}}.
\end{align}
Theorem \ref{thm:matrix-hoeffding} follows immediately.

\subsection{Proof of Theorem \ref{thm:Srikant-generalize}}\label{app:Srikant-generalize}
This proof is a correction and improvement of Theorem 1 in \cite{srikant2024rates}, which applies Stein's method and Lindeberg's decomposition.
Using notation from \cite{JMLR2019CLT}, we denote, for every $k \in [n]$,
\begin{align*}
\bm{S}_k = \sum_{j=1}^k \bm{x}_k,  \quad \bm{T}_k = \sum_{j=k}^n \bm{V}_j^{\frac{1}{2}}\bm{z}_j, \quad \text{and} \quad \bm{T}_k' = \bm{T}_k + \bm{\Sigma}^{\frac{1}{2}}\bm{z}'.
\end{align*}
Here, $\{\bm{z}_k\}_{k=1}^n$ and $\bm{z}'$ are $i.i.d.$ standard Gaussian random variables in $\mathbb{R}^d$ and independent of the filtration $\{\mathscr{F}_k\}_{k=0}^n$. Note that since we assume $\bm{P}_1 = n\bm{\Sigma}_n$ almost surely, and that $\{\bm{V}_j\}_{1 \leq j \leq k}$ are all measurable with respect to $\mathscr{F}_{k-1}$, the matrices 
\begin{align*}
&\bm{P}_k = \bm{P}_1 - \sum_{j=1}^{k-1} \bm{V}_j = n\bm{\Sigma}_n - \sum_{j=1}^{k-1} \bm{V}_j, \quad \text{and} \\ 
&\bm{P}_{k+1} = \bm{P}_1 - \sum_{j=1}^{k} \bm{V}_j = n\bm{\Sigma}_n - \sum_{j=1}^{k} \bm{V}_j
\end{align*}
are also measurable with respect to $\mathscr{F}_{k-1}$. 
Next, since $\frac{1}{\sqrt{n}}\bm{T}_1 \sim \mathcal{N}(\bm{0},\bm{\Sigma}_n)$, we have that 
\begin{align}\label{eq:Srikant-Wasserstein}
d_{\mathsf{W}}\left(\frac{1}{\sqrt{n}}\sum_{k=1}^n \bm{x}_k,\mathcal{N}(\bm{0},\bm{\Sigma}_n)\right) &= \sup_{h \in \mathsf{Lip}_1} \left|\mathbb{E}\left[h\left(\frac{1}{\sqrt{n}}\bm{S}_n\right)\right] - \mathbb{E}\left[h\left(\frac{1}{\sqrt{n}}\bm{T}_1\right)\right]\right| \nonumber \\ 
&= \frac{1}{\sqrt{n}} \sup_{h \in \mathsf{Lip}_1}|\mathbb{E}[h(\bm{S}_n)]-\mathbb{E}[h(\bm{T}_1)]|\nonumber \\ 
&\leq \frac{1}{\sqrt{n}} \sup_{h \in \mathsf{Lip}_1}|\mathbb{E}[h(\bm{S}_n)]-\mathbb{E}[h(\bm{T}_1')]| + \frac{1}{\sqrt{n}} \sup_{h \in \mathsf{Lip}_1}|\mathbb{E}[h(\bm{T}_1)]-\mathbb{E}[h(\bm{T}_1')]|\nonumber \\ 
&\overset{(i)}\leq \frac{1}{\sqrt{n}} \sup_{h \in \mathsf{Lip}_1}|\mathbb{E}[h(\bm{S}_n)]-\mathbb{E}[h(\bm{T}_1')]| + \frac{1}{\sqrt{n}} \sup_{h \in \mathsf{Lip}_1} \mathbb{E}|h(\bm{T}_1) - h(\bm{T}_1')| \nonumber \\ 
&\overset{(ii)}\leq \frac{1}{\sqrt{n}} \sup_{h \in \mathsf{Lip}_1}|\mathbb{E}[h(\bm{S}_n)]-\mathbb{E}[h(\bm{T}_1')]| + \frac{1}{\sqrt{n}} \mathbb{E}\|\bm{\Sigma}^{\frac{1}{2}}\bm{z}\|_2 \nonumber \\ 
&\overset{(iii)}\leq \frac{1}{\sqrt{n}} \sup_{h \in \mathsf{Lip}_1}|\mathbb{E}[h(\bm{S}_n)]-\mathbb{E}[h(\bm{T}_1')]| + \sqrt{\frac{\mathsf{Tr}(\bm{\Sigma})}{n}},
\end{align}
where we have used Jensen's inequality in steps (i) and (iii) and the Lipchitz property in (ii). 
Using Lindeberg's swapping and the triangle inequality, we obtain that 
\begin{align}\label{eq:Srikant-Lindeberg}
|\mathbb{E}[h(\bm{S}_n)]-\mathbb{E}[h(\bm{T}_1')]| &\leq \sum_{k=1}^n |\mathbb{E}[h(\bm{S}_k+\bm{T}_{k+1}')]-\mathbb{E}[h(\bm{S}_{k-1}+\bm{T}_k')]|,
\end{align}
where we define $\bm{S}_0=\bm{T}_{n+1}=0$ for consistency. For each $k \in [n]$, define
\begin{align*}
\tilde{h}_k(\bm{x}):= h(\bm{P}_k'^{\frac{1}{2}}\bm{x}+\bm{S}_{k-1}),
\end{align*}
where we denote, for simplicity, $$\bm{P}_k' = \bm{P}_k + \bm{\Sigma}.$$ 
Since $h \in \mathsf{Lip}_1$, $\tilde{h}_k$ has Lipchitz coefficient $L_k = \|\bm{P}_k'^{\frac{1}{2}}\|$. 
Let  $f_k: \mathbb{R}^d \to \mathbb{R}$ be such that
\begin{align*}
\tilde{h}_k(\bm{x})-\mathbb{E}[\tilde{h}_k(\bm{z})] = \Delta f_k(\bm{x}) - \bm{x}^\top \nabla f_k(\bm{x})
\end{align*}
holds for any $\bm{x} \in \mathbb{R}^d$, where $\bm{z}$ is the $d$-dimensional standard Gaussian random variable; $f_k$ is the solution to the multivariate Stein equation, and therefore well defined and measurable with respect to $\mathcal{F}_{k-1}$ \citep[see, e.g.,][Section 4.3]{Nourdin_Peccati_2012}. 
Therefore, 
\begin{align*}
&h(\bm{S}_k + \bm{T}_{k+1}') -\mathbb{E}[h(\bm{S}_{k-1} + \bm{T}_k')\mid \mathscr{F}_{k-1}]\\ 
&= h(\bm{S}_k + \bm{T}_{k+1}') -\mathbb{E}[h(\bm{S}_{k-1} + \bm{P}_k'^{\frac{1}{2}}\bm{z})\mid \mathscr{F}_{k-1}] \\ 
&= \tilde{h}_k(\bm{P}_k'^{-\frac{1}{2}}(\bm{x}_k + \bm{T}_{k+1}'))-\mathbb{E}[\tilde{h}_k(\bm{z})] \\ 
&= \Delta f_k(\bm{P}_k'^{-\frac{1}{2}}(\bm{x}_k + \bm{T}_{k+1})) - \left[\bm{P}_k'^{-\frac{1}{2}}(\bm{x}_k + \bm{T}_{k+1}')\right]^\top \nabla f_k(\bm{P}_k'^{-\frac{1}{2}}(\bm{x}_k + \bm{T}_{k+1}')).
\end{align*}
Notice that by definition, $\bm{P}_k'^{-\frac{1}{2}}\bm{T}_{k+1} \mid \mathscr{F}_{k-1} \sim \mathcal{N}(\bm{0},\bm{P}_k'^{-\frac{1}{2}}\bm{P}_{k+1}'\bm{P}_k'^{-\frac{1}{2}})$. Below we let $\tilde{\bm{z}}_k$ be a random vector such that $\tilde{\bm{z}}_k \mid \mathscr{F}_{k-1} \sim \mathcal{N}(\bm{0},\bm{P}_k'^{-\frac{1}{2}}\bm{P}_{k+1}'\bm{P}_k'^{-\frac{1}{2}})$ and independent of $\bm{x}_k$  conditionally on $\mathscr{F}_{k-1}$. By taking conditional expectations, we obtain
\begin{align*}
&\mathbb{E}\left[h(\bm{S}_k + \bm{T}_{k+1}')\mid \mathscr{F}_{k-1}\right] - \mathbb{E}[h(\bm{S}_{k-1} + \bm{T}_k')\mid \mathscr{F}_{k-1}]\\ 
&=\mathbb{E}[\Delta f_k(\bm{P}_k'^{-\frac{1}{2}}\bm{x}_k + \tilde{\bm{z}}_k)\mid \mathscr{F}_{k-1}]- \mathbb{E}\left[(\bm{P}_k'^{-\frac{1}{2}}\bm{x}_k + \tilde{\bm{z}}_k)^\top \nabla f_k(\bm{P}_k'^{-\frac{1}{2}}\bm{x}_k + \tilde{\bm{z}}_k)\mid \mathscr{F}_{k-1}\right].
\end{align*}
Since, by definition, $\Delta f_k = \mathsf{Tr}(\nabla^2 f_k)$, this can be further decomposed 
\begin{align}\label{eq:srikant-decompose}
&\mathbb{E}\left[h(\bm{S}_k + \bm{T}_{k+1}')\mid \mathscr{F}_{k-1}\right] - \mathbb{E}[h(\bm{S}_{k-1} + \bm{T}_k')\mid \mathscr{F}_{k-1}]\nonumber \\ 
&= \underset{I_1}{\underbrace{\mathbb{E}\left[\mathsf{Tr}\left(\bm{P}_k'^{-\frac{1}{2}}\bm{P}_{k+1}'\bm{P}_k'^{-\frac{1}{2}}\nabla^2 f_k(\bm{P}_k'^{-\frac{1}{2}}\bm{x}_k + \tilde{\bm{z}}_k)\right)-\tilde{\bm{z}}_k^\top \nabla f_k(\bm{P}_k'^{-\frac{1}{2}}\bm{x}_k + \tilde{\bm{z}}_k) \bigg| \mathscr{F}_{k-1}\right]}} \nonumber \\ 
&+ \underset{I_2}{\underbrace{\mathbb{E}\left[\mathsf{Tr}\left(\left(\bm{I}-\bm{P}_k'^{-\frac{1}{2}}\bm{P}_{k+1}'\bm{P}_k'^{-\frac{1}{2}}\right)\nabla^2 f_k(\bm{P}_k'^{-\frac{1}{2}}\bm{x}_k + \tilde{\bm{z}}_k)\right)\bigg|\mathscr{F}_{k-1}\right]}} \nonumber \\ 
&- \underset{I_3}{\underbrace{\mathbb{E}\left[(\bm{P}_k'^{-\frac{1}{2}}\bm{x}_k)^\top \nabla f_k(\tilde{\bm{z}}_k) \bigg|\mathscr{F}_{k-1}\right]}} - \underset{I_4}{\underbrace{\mathbb{E}\left[(\bm{P}_k'^{-\frac{1}{2}}\bm{x}_k)^\top \left(\nabla f_k(\bm{P}_k'^{-\frac{1}{2}}\bm{x}_k + \tilde{\bm{z}}_k)- \nabla f_k(\tilde{\bm{z}}_k) \right)\bigg|\mathscr{F}_{k-1}\right]}}.
\end{align}
Of the four terms on the right-hand-side, $I_1$ is equal to $\bm{0}$ according to Lemma 2 in \cite{srikant2024rates}; $I_3$ is also $\bm{0}$ according to martingale property. Below, we will handle $I_2$ and $I_4$.

\paragraph{Controlling $I_2$} We first make the observation that
\begin{align*}
\bm{I}-\bm{P}_k'^{-\frac{1}{2}}\bm{P}_{k+1}'\bm{P}_k'^{-\frac{1}{2}}=\bm{P}_k'^{-\frac{1}{2}}(\bm{P}_k'-\bm{P}_{k+1}')\bm{P}_k'^{-\frac{1}{2}}=\bm{P}_k'^{-\frac{1}{2}}\bm{V}_k\bm{P}_k'^{-\frac{1}{2}};
\end{align*}
Therefore, the term $I_2$ can be represented as
\begin{align}\label{eq:Srikant-I2}
&\mathbb{E}\left[\mathsf{Tr}\left(\left(\bm{I}-\bm{P}_k'^{-\frac{1}{2}}\bm{P}_{k+1}'\bm{P}_k'^{-\frac{1}{2}}\right)\nabla^2 f_k(\bm{P}_k'^{-\frac{1}{2}}\bm{x}_k + \tilde{\bm{z}}_k)\right)\bigg|\mathscr{F}_{k-1}\right] \nonumber\\ 
&= \mathbb{E}\left[\mathsf{Tr}\left(\bm{P}_k'^{-\frac{1}{2}}\bm{V}_k\bm{P}_k'^{-\frac{1}{2}} \nabla^2 f_k(\bm{P}_k'^{-\frac{1}{2}}\bm{x}_k + \tilde{\bm{z}}_k)\right)\bigg|\mathscr{F}_{k-1}\right]\nonumber\\ 
&\overset{(i)}{=}\mathsf{Tr}\left\{\bm{P}_k'^{-\frac{1}{2}}\bm{V}_k\bm{P}_k'^{-\frac{1}{2}}\mathbb{E}\left[\nabla^2 f_k(\bm{P}_k'^{-\frac{1}{2}}\bm{x}_k + \tilde{\bm{z}}_k)\bigg|\mathscr{F}_{k-1}\right]\right\} \nonumber\\ 
&= \mathsf{Tr}\left\{\bm{P}_k'^{-\frac{1}{2}}\bm{V}_k\bm{P}_k'^{-\frac{1}{2}}\mathbb{E}\left[\nabla^2 f_k(\tilde{\bm{z}}_k)\bigg|\mathscr{F}_{k-1}\right]\right\}\nonumber \\ 
&+ \mathsf{Tr}\left\{\bm{P}_k'^{-\frac{1}{2}}\bm{V}_k\bm{P}_k'^{-\frac{1}{2}}\mathbb{E}\left[\nabla^2 f_k(\bm{P}_k'^{-\frac{1}{2}}\bm{x}_k + \tilde{\bm{z}}_k)-\nabla^2 f_k(\tilde{\bm{z}}_k)\bigg|\mathscr{F}_{k-1}\right]\right\},
\end{align}
where we invoked the linearity of trace and expectation in (i). 


\paragraph{Controlling $I_4$} For clarity, we define a uni-dimensional function
\begin{align*}
F(t) := (\bm{P}_k'^{-\frac{1}{2}}\bm{x}_k)^\top \nabla f_k(t\bm{P}_k'^{-\frac{1}{2}}\bm{x}_k + \tilde{\bm{z}}_k), \quad \forall t \in [0,1].
\end{align*}
Since $f_k$ is twice-differentiable, the derivative of $F(t)$ is given by
\begin{align*}
F'(t) = (\bm{P}_k'^{-\frac{1}{2}}\bm{x}_k)^\top \nabla^2 f_k(t\bm{P}_k'^{-\frac{1}{2}}\bm{x}_k + \tilde{\bm{z}}_k) (\bm{P}_k'^{-\frac{1}{2}}\bm{x}_k).
\end{align*}
Consequently, the Lagrange's mean-value theorem guarantees the existence of $\Theta  \in [0,1]$, such that
{\color{black}
\begin{align*}
&(\bm{P}_k'^{-\frac{1}{2}}\bm{x}_k)^\top \left(\nabla f_k(\bm{P}_k'^{-\frac{1}{2}}\bm{x}_k + \tilde{\bm{z}}_k)- \nabla f_k(\tilde{\bm{z}}_k) \right) \\
&= F(1) - F(0) \\ 
&= \int_0^1 F'(t)\mathrm{d}t \\ 
&= \int_0^1 (\bm{P}_k'^{-\frac{1}{2}}\bm{x}_k)^\top \nabla^2 f_k(t\bm{P}_k'^{-\frac{1}{2}}\bm{x}_k + \tilde{\bm{z}}_k) (\bm{P}_k'^{-\frac{1}{2}}\bm{x}_k)\mathrm{d}t.
\end{align*}
}
Consequently, the term $I_4$ is characterized by
{\color{black}
\begin{align}\label{eq:Srikant-I4}
&\mathbb{E}\left[(\bm{P}_k'^{-\frac{1}{2}}\bm{x}_k)^\top \left(\nabla f_k(\bm{P}_k'^{-\frac{1}{2}}\bm{x}_k + \tilde{\bm{z}}_k)- \nabla f_k(\tilde{\bm{z}}_k) \right)\bigg|\mathscr{F}_{k-1}\right]\nonumber \\ 
&=\mathbb{E}\left[\int_0^1 (\bm{P}_k'^{-\frac{1}{2}}\bm{x}_k)^\top \nabla^2f_k(t \bm{P}_k'^{-\frac{1}{2}}\bm{x}_k + \tilde{\bm{z}}_k)(\bm{P}_k'^{-\frac{1}{2}}\bm{x}_k) \mathrm{d}t\bigg|\mathscr{F}_{k-1}\right]\nonumber \\
&= \mathbb{E}\left[(\bm{P}_k'^{-\frac{1}{2}}\bm{x}_k)^\top \nabla^2f_k(\tilde{\bm{z}}_k)(\bm{P}_k'^{-\frac{1}{2}}\bm{x}_k) \bigg|\mathscr{F}_{k-1}\right] \nonumber \\ 
&+ \mathbb{E}\left[\int_0^1(\bm{P}_k'^{-\frac{1}{2}}\bm{x}_k)^\top (\nabla^2f_k(t \bm{P}_k'^{-\frac{1}{2}}\bm{x}_k + \tilde{\bm{z}}_k) - \nabla^2 f_k(\tilde{\bm{z}}_k))(\bm{P}_k'^{-\frac{1}{2}}\bm{x}_k)\mathrm{d}t \bigg|\mathscr{F}_{k-1}\right],
\end{align}
}
where $\Theta \in (0,1)$. Here, the first term on the right-most part of the equation can be further computed by
\begin{align}\label{eq:Srikant-I41}
&\mathbb{E}\left[(\bm{P}_k'^{-\frac{1}{2}}\bm{x}_k)^\top \nabla^2f_k(\tilde{\bm{z}}_k)(\bm{P}_k'^{-\frac{1}{2}}\bm{x}_k) \bigg|\mathscr{F}_{k-1}\right] \nonumber \\
&= \mathbb{E}\left[\mathsf{Tr}\left((\bm{P}_k'^{-\frac{1}{2}}\bm{x}_k)^\top \nabla^2f_k(\tilde{\bm{z}}_k)(\bm{P}_k'^{-\frac{1}{2}}\bm{x}_k)\right)\bigg|\mathscr{F}_{k-1}\right]\nonumber \\ 
&\overset{(i)}{=} \mathbb{E}\left[\mathsf{Tr}\left((\bm{P}_k'^{-\frac{1}{2}}\bm{x}_k)(\bm{P}_k'^{-\frac{1}{2}}\bm{x}_k)^\top \nabla^2f_k(\tilde{\bm{z}}_k)\right)\bigg|\mathscr{F}_{k-1}\right]\nonumber\\
&\overset{(ii)}{=} \mathsf{Tr}\left\{\mathbb{E}\left[\bm{P}_k'^{-\frac{1}{2}}\bm{x}_k\bm{x}_k^\top \bm{P}_k'^{-\frac{1}{2}}\nabla^2f_k(\tilde{\bm{z}}_k)\bigg|\mathscr{F}_{k-1}\right]\right\}\nonumber \\ 
&\overset{(iii)}{=}\mathsf{Tr}\left\{\mathbb{E}[\bm{P}_k'^{-\frac{1}{2}}\bm{x}_k\bm{x}_k^\top \bm{P}_k'^{-\frac{1}{2}} \mid \mathscr{F}_{k-1}] \mathbb{E}[\nabla^2f_k(\tilde{\bm{z}}_k)\bigg|\mathscr{F}_{k-1}]\right\}\nonumber \\
&\overset{(iv)}{=}\mathsf{Tr}\left\{\bm{P}_k'^{-\frac{1}{2}}\bm{V}_k \bm{P}_k'^{-\frac{1}{2}}\mathbb{E}[\nabla^2f_k(\tilde{\bm{z}}_k)\bigg|\mathscr{F}_{k-1}]\right\};
\end{align}
notice that we applied the basic property of matrix trace in (i), the linearity of expectation in (ii), the conditional independence between $\bm{x}_k$ and $\tilde{\bm{z}}_k$ in (iii), and the definition of $\bm{V}_k$ in (iv). The right-most part of this equation is exactly the same as the first term on the right-most part of \eqref{eq:Srikant-I2}.

\medskip
Consequently, putting relations \eqref{eq:srikant-decompose}, \eqref{eq:Srikant-I2}, \eqref{eq:Srikant-I4} and \eqref{eq:Srikant-I41}, we obtain
{\color{black}
\begin{align}\label{eq:Srikant-reorganize}
&\mathbb{E}\left[h(\bm{S}_k + \bm{T}_{k+1}')\mid \mathscr{F}_{k-1}\right] - \mathbb{E}[h(\bm{S}_{k-1} + \bm{T}_k')\mid \mathscr{F}_{k-1}]\nonumber \\ 
&=\mathsf{Tr}\left\{\bm{P}_k'^{-\frac{1}{2}}\bm{V}_k\bm{P}_k'^{-\frac{1}{2}}\mathbb{E}\left[\nabla^2 f_k(\bm{P}_k^{-\frac{1}{2}}\bm{x}_k + \tilde{\bm{z}}_k)-\nabla^2 f_k(\tilde{\bm{z}}_k)\bigg|\mathscr{F}_{k-1}\right]\right\} \nonumber \\ 
&- \mathbb{E}\left[\int_0^1 (\bm{P}_k'^{-\frac{1}{2}}\bm{x}_k)^\top (\nabla^2f_k(t \bm{P}_k'^{-\frac{1}{2}}\bm{x}_k + \tilde{\bm{z}}_k) - \nabla^2 f_k(\tilde{\bm{z}}_k))(\bm{P}_k'^{-\frac{1}{2}}\bm{x}_k)\mathrm{d}t \bigg|\mathscr{F}_{k-1}\right]
\end{align}
}
Both terms on the right-hand-side can be bounded by the Holder's property of $\nabla^2 f_k$. Specifically, we have the following proposition:
\begin{proposition}\label{prop:Stein-smooth}
Let $h: \mathbb{R}^d \to \mathbb{R} \in \mathsf{Lip}_1$, $\bm{\mu} \in \mathbb{R}^d$ and $\bm{\Sigma} \succ \bm{0} \in \mathbb{S}^{d \times d}$. Further define $g: \mathbb{R}^d \to \mathbb{R}$ as
\begin{align*}
g(\bm{x})= h(\bm{\Sigma}^{\frac{1}{2}}\bm{x} +\bm{\mu}), \quad \forall x \in \mathbb{R}^d,
\end{align*}
and use $f_g$ to denote the solution to Stein's equation
\begin{align*}
\Delta f(\bm{x}) - \bm{x}^\top \nabla f(\bm{x}) = g(\bm{x}) - \mathbb{E}[g(\bm{z})]
\end{align*}
where $\bm{z}$ is the $d$-dimensional standard Gaussian distribution. It can then be guaranteed that
\begin{align}
\left\|\nabla^2 f_g(\bm{x}) - \nabla^2 f_g(\bm{y})\right\| \lesssim (2+\log (d\|\bm{\Sigma}\|))^+ \cdot \|\bm{\Sigma}^{\frac{1}{2}}(\bm{x} - \bm{y})\|_2 + e^{-1}.
\end{align}
\end{proposition}
\begin{proof} 
See Appendix \ref{app:proof-Stein-smooth}. 
\end{proof}

Proposition \ref{prop:Stein-smooth} guarantees for every $t \in [0,1]$ that
\begin{align}
\notag \left\|\nabla^2f_k(t \bm{P}_k'^{-\frac{1}{2}}\bm{x}_k + \tilde{\bm{z}}_k) - \nabla^2 f_k(\tilde{\bm{z}}_k)\right\| &\leq (2+\log(d\|\bm{P}_k'\|))^+ \|\bm{x}_k\|_2 + e^{-1}\nonumber \\ 
&\leq (2+\log(d\|(n\bm{\Sigma}_n + \bm{\Sigma})\|))^+ \|\bm{x}_k\|_2 + e^{-1}
\end{align}

Consequently, the first term on the right hand side of \eqref{eq:Srikant-reorganize} can be bounded by
\begin{align*}
&\left|\mathsf{Tr}\left\{\bm{P}_k'^{-\frac{1}{2}}\bm{V}_k\bm{P}_k'^{-\frac{1}{2}}\mathbb{E}\left[\nabla^2 f_k(\bm{P}_k^{-\frac{1}{2}}\bm{x}_k + \tilde{\bm{z}}_k)-\nabla^2 f_k(\tilde{\bm{z}}_k)\bigg|\mathscr{F}_{k-1}\right]\right\}\right| \\ 
&\leq (2+\log(d\|(n\bm{\Sigma}_n + \bm{\Sigma})\|))^+\mathbb{E}\left[\|\bm{P}_k'^{-\frac{1}{2}}\bm{x}_k\|_2^2 \bigg|\mathscr{F}_{k-1}\right] \mathbb{E}[\|\bm{x}_k\|_2 |\mathscr{F}_{k-1}] + e^{-1} \mathbb{E}\left[\|\bm{P}_k'^{-\frac{1}{2}}\bm{x}_k\|_2^2 \bigg|\mathscr{F}_{k-1}\right] \\ 
\end{align*}
Here, we further notice that since $\bm{P}_k'$ is measurable with respect to $\mathscr{F}_{k-1}$, and that $\|\bm{P}_k'^{-\frac{1}{2}}\bm{x}_k\|_2^2$ and $\|\bm{x}_k\|_2$ are positively correlated \footnote{We refer readers to Appendix \ref{app:proof-correlated-norms} for the details of this claim.}, 
\begin{align}\label{eq:correlated-norms}
\mathbb{E}\left[\|\bm{P}_k'^{-\frac{1}{2}}\bm{x}_k\|_2^2 \bigg|\mathscr{F}_{k-1}\right] \mathbb{E}[\|\bm{x}_k\|_2 |\mathscr{F}_{k-1}] \leq \mathbb{E}\left[\|\bm{P}_k'^{-\frac{1}{2}}\bm{x}_k\|_2^2 \|\bm{x}_k\|_2 \bigg|\mathscr{F}_{k-1}\right].
\end{align}
Therefore, the upper bound can be further simplified by
\begin{align*}
&\left|\mathsf{Tr}\left\{\bm{P}_k'^{-\frac{1}{2}}\bm{V}_k\bm{P}_k'^{-\frac{1}{2}}\mathbb{E}\left[\nabla^2 f_k(\bm{P}_k^{-\frac{1}{2}}\bm{x}_k + \tilde{\bm{z}}_k)-\nabla^2 f_k(\tilde{\bm{z}}_k)\bigg|\mathscr{F}_{k-1}\right]\right\}\right| \\ 
&\leq (2+\log(d\|(n\bm{\Sigma}_n + \bm{\Sigma})\|))^+ \mathbb{E}\left[\|\bm{P}_k'^{-\frac{1}{2}}\bm{x}_k\|_2^2 \|\bm{x}_k\|_2 \bigg|\mathscr{F}_{k-1}\right]+ \mathbb{E}\left[\|\bm{P}_k'^{-\frac{1}{2}}\bm{x}_k\|_2^2 \bigg|\mathscr{F}_{k-1}\right].
\end{align*}
Meanwhile, the second term can also be bounded by
{\color{black}
\begin{align*}
&\left|\mathbb{E}\left[\int_0^1(\bm{P}_k^{-\frac{1}{2}}\bm{x}_k)^\top (\nabla^2f_k(t \bm{P}_k^{-\frac{1}{2}}\bm{x}_k + \tilde{\bm{z}}_k) - \nabla^2 f_k(\tilde{\bm{z}}_k))(\bm{P}_k^{-\frac{1}{2}}\bm{x}_k)\mathrm{d}t \bigg|\mathscr{F}_{k-1}\right]\right| \\ 
&\leq (2+\log(d\|(n\bm{\Sigma}_n + \bm{\Sigma})\|))^+ \mathbb{E}\left[\|\bm{P}_k'^{-\frac{1}{2}}\bm{x}_k\|_2^2 \|\bm{x}_k\|_2 \bigg|\mathscr{F}_{k-1}\right]+ \mathbb{E}\left[\|\bm{P}_k'^{-\frac{1}{2}}\bm{x}_k\|_2^2 \bigg|\mathscr{F}_{k-1}\right].
\end{align*}
}
Since these two bounds are equivalent, in combination, the triangle inequality guarantees
\begin{align*}
&\Big|\mathbb{E}\left[h(\bm{S}_k + \bm{T}_{k+1})\mid \mathscr{F}_{k-1}\right] - \mathbb{E}[h(\bm{S}_{k-1} + \bm{T}_k)\mid \mathscr{F}_{k-1}]\Big|\\ 
& \lesssim (2+\log(d\|(n\bm{\Sigma}_n + \bm{\Sigma})\|))^+ \mathbb{E}\left[\|\bm{P}_k'^{-\frac{1}{2}}\bm{x}_k\|_2^2 \|\bm{x}_k\|_2 \bigg|\mathscr{F}_{k-1}\right]+ \mathbb{E}\left[\|\bm{P}_k'^{-\frac{1}{2}}\bm{x}_k\|_2^2 \bigg|\mathscr{F}_{k-1}\right].
\end{align*}
Plugging into \eqref{eq:Srikant-Lindeberg}, we obtain
\begin{align}\label{eq:Srikant-Lindeberg-2}
|\mathbb{E}[h(\bm{S}_n)]-\mathbb{E}[h(\bm{T}_1')]| &\leq \sum_{k=1}^n |\mathbb{E}[h(\bm{S}_k+\bm{T}_{k+1}')]-\mathbb{E}[h(\bm{S}_{k-1}+\bm{T}_k')]| \nonumber \\ 
&\lesssim (2+\log(d\|(n\bm{\Sigma}_n + \bm{\Sigma})\|))^+ \sum_{k=1}^n \mathbb{E}\left[\mathbb{E}\left[\|(\bm{P}_k + \bm{\Sigma})^{-\frac{1}{2}}\bm{x}_k\|_2^2 \|\bm{x}_k\|_2 \bigg|\mathscr{F}_{k-1}\right]\right] \nonumber \\ 
&+ \sum_{k=1}^n \mathbb{E}\left[\mathbb{E}\left[\|(\bm{P}_k + \bm{\Sigma})^{-\frac{1}{2}}\bm{x}_k\|_2^2 \bigg|\mathscr{F}_{k-1}\right]\right].
\end{align}
We now aim to proof the last summation on the right-hand-side of \eqref{eq:Srikant-Lindeberg-2}. The law of total expectation directly implies, for every $k \in [n]$, that
\begin{align*}
\mathbb{E}\left[\mathbb{E}\left[\|(\bm{P}_k + \bm{\Sigma})^{-\frac{1}{2}}\bm{x}_k\|_2^2 \bigg|\mathscr{F}_{k-1}\right]\right] = \mathbb{E}\left[\mathsf{Tr}(\bm{P}_k'^{-1}\bm{V}_k)\mid \mathscr{F}_0\right].
\end{align*}
To further feature the summation, we invoke a telescoping technique. By taking $\bm{A} = \bm{P}_k'$ and $\bm{B} = \bm{P}_{k+1}'$ in Theorem \ref{thm:Klein}, the summand can be bounded by
\begin{align*}
\mathsf{Tr}(\bm{P}_k'^{-1}\bm{V}_k)&= \mathsf{Tr}(\bm{P}_k'^{-1}(\bm{P}_k' - \bm{P}_{k+1}')) \leq \mathsf{Tr}(\log (\bm{P}_k')) - \mathsf{Tr}(\log(\bm{P}_{k+1}')).
\end{align*}
Summing from $k=1$ through $k=n$, we obtain
\begin{align}\label{eq:Srikant-telescope}
\sum_{k=1}^n \mathbb{E}\left[\mathbb{E}\left[\|(\bm{P}_k + \bm{\Sigma})^{-\frac{1}{2}}\bm{x}_k\|_2^2 \bigg|\mathscr{F}_{k-1}\right]\right] 
&=\sum_{k=1}^n \mathbb{E}\left[\mathsf{Tr}(\bm{P}_k'^{-1}\bm{V}_k)\mid\mathscr{F}_0\right] \nonumber \\
&\leq \sum_{k=1}^n \mathbb{E}\left[\mathsf{Tr}(\log (\bm{P}_k')) - \mathsf{Tr}(\log(\bm{P}_{k+1}')) \big|\mathscr{F}_0\right] \nonumber \\ 
&= \mathbb{E}[\mathsf{Tr}(\log \bm{P}_1')]-\mathbb{E}[\mathsf{Tr}(\log \bm{P}_n')]\nonumber \\
&= \mathsf{Tr}(\log(n\bm{\Sigma}_n+\bm{\Sigma})) - \log(\bm{\Sigma})).
\end{align}
Our target result follows by combining \eqref{eq:Srikant-Wasserstein}, \eqref{eq:Srikant-Lindeberg-2} and \eqref{eq:Srikant-telescope}.

\subsection{Proof of Corollary \ref{cor:Wu}}\label{app:proof-cor-Wu}
In order to simplify the upper bound in Theorem \ref{thm:Srikant-generalize}, we firstly take $\bm{\Sigma} = \bm{\Sigma}_n$, yielding
\begin{align*}
d_{\mathsf{W}}\left(\frac{1}{\sqrt{n}}\sum_{k=1}^n \bm{x}_k,\mathcal{N}(\bm{0},\bm{\Sigma}_n)\right)  
& \lesssim  \frac{(2+\log(dn\|\bm{\Sigma}_n\|))^+}{\sqrt{n}} \sum_{k=1}^n \mathbb{E}\left[\mathbb{E}\left[\|(\bm{P}_k + \bm{\Sigma}_n)^{-\frac{1}{2}}\bm{x}_k\|_2^2 \|\bm{x}_k\|_2 \bigg|\mathscr{F}_{k-1}\right]\right] \\ 
&+ \frac{1}{\sqrt{n}}\left[\mathsf{Tr}(\log((n+1)\bm{\Sigma}_n)) - \log(\bm{\Sigma_n}))\right]+ \sqrt{\frac{\mathsf{Tr}(\bm{\Sigma_n})}{n}}.
\end{align*}
For the first term on the second line, we observe
\begin{align*}
\mathsf{Tr}(\log((n+1)\bm{\Sigma}_n)) - \log(\bm{\Sigma_n}))&= \sum_{i=1}^d \log (\lambda_i((n+1)\bm{\Sigma}_n)) - \sum_{i=1}^d \log (\lambda_i(\bm{\Sigma}_n))\\ 
&= \sum_{i=1}^n (\log (n+1) + \log \lambda_i(\bm{\Sigma}_n)) - \log \lambda_i(\bm{\Sigma}_n) = d \log(n+1).
\end{align*}
Meanwhile, the condition \eqref{eq:3rd-momentum-condition} can be invoked on the first term to obtain
\begin{align*}
\sum_{k=1}^n \mathbb{E}\left[\mathbb{E}\left[\|(\bm{P}_k + \bm{\Sigma}_n)^{-\frac{1}{2}}\bm{x}_k\|_2^2 \|\bm{x}_k\|_2 \bigg|\mathscr{F}_{k-1}\right]\right] &\leq M \cdot \sum_{k=1}^n\mathbb{E}\left[\mathbb{E}\left[\|(\bm{P}_k + \bm{\Sigma}_n)^{-\frac{1}{2}}\bm{x}_k\|_2^2 \bigg|\mathscr{F}_{k-1}\right]\right] \\ 
&= M \cdot \left[\mathsf{Tr}(\log((n+1)\bm{\Sigma}_n)) - \log(\bm{\Sigma_n}))\right]\\ 
& = Md\log(n+1),
\end{align*}
where the third line follows from \eqref{eq:Srikant-telescope}. In combination, the Wasserstein distance can be bounded by
\begin{align*}
d_{\mathsf{W}}\left(\frac{1}{\sqrt{n}}\sum_{k=1}^n \bm{x}_k,\mathcal{N}(\bm{0},\bm{\Sigma}_n)\right) &\lesssim \frac{M(2+\log(dn\|\bm{\Sigma}_n\|))^+ + 1}{\sqrt{n}}\cdot d\log n+ \sqrt{\frac{\mathsf{Tr}(\bm{\Sigma_n})}{n}}.
\end{align*}

\paragraph{Comparison with the corresponding uni-dimensional Corollary 2.3 in \cite{rollin2018}} When $d=1$, the condition \eqref{eq:3rd-momentum-condition} reduces to
\begin{align*}
\mathbb{E}|x_k|^3\mid \mathcal{F}_{k-1} \leq M \mathbb{E}[x_k^2]\mid \mathcal{F}_{k-1}, \quad \forall k \in [n].
\end{align*}
Notice that this is a weaker assumption than the one outlined in Equation (2.13) of \cite{rollin2018}, which further assumes that the conditional third momentum of $x_k$ is uniformly bounded. In our proof of Corollary \ref{cor:Wu}, we invoke a telescoping method that not only simplifies the proof but also yields a tighter upper bound.

\paragraph{Comments on Corollary 3 in \cite{JMLR2019CLT}} In a similar attempt to generalize this uni-dimensional corollary from \cite{rollin2018} to multi-dimensional settings, Corollary 3 of \cite{JMLR2019CLT} made the assumption that 
\begin{align}\label{eq:JMLR-Cor-assumption}
\mathbb{E}[\|\bm{x}_k\|_2^3 \mid \mathcal{F}_{k-1}] \leq \beta \vee \delta \mathsf{Tr}(\bm{V}_k), \quad \text{a.s.}
\end{align}
and applied the technique used by \cite{rollin2018}, which involved defining a sequence of stopping times $\{\tau_k\}_{k \in [n]}$. However, this generalization is non-trivial since the positive semi-definite order of symmetric matrices is \emph{incomplete}; in other words, when $\bm{A},\bm{B} \in \mathbb{S}^{d \times d}$, $\bm{A} \npreceq \bm{B}$ does not imply $\bm{A} \succ \bm{B}$. Consequently, the derivation from Equation (A.36) to (A.37) in the proof of Corollary 3 of \cite{JMLR2019CLT} is invalid: apparently, the authors' reasoning was, under their notation:
\begin{align*}
\mathsf{Tr}(\bar{\bm{V}}_{\tau_k} - \bar{\bm{V}}_{\tau_{k-1}})&= \mathsf{Tr}(\bar{\bm{V}}_{\tau_k} - \bar{\bm{V}}_{\tau_{k-1} + 1}) + \mathsf{Tr}(\bm{V}_k) \\ 
&\leq\mathsf{Tr}\left( \frac{k}{n}\bm{\Sigma}_n - \frac{k-1}{n}\bm{\Sigma}_n\right) + \beta^{\frac{2}{3}} = \frac{1}{n}\mathsf{Tr}(\bm{\Sigma}_n)+ \beta^{\frac{2}{3}}
\end{align*}
where the inequality on the second line follows from $\bm{V}_{\tau_k} \preceq \frac{k}{n}\bm{\Sigma}_n$, and that $\bm{V}_{\tau_{k-1}+1} \npreceq \frac{k-1}{n}\bm{\Sigma}_n$. However, since the positive semi-definite order is incomplete, the fact that $\bm{V}_{\tau_{k-1}+1} \npreceq \frac{k-1}{n}\bm{\Sigma}_n$ is not equivalent to $\bm{V}_{\tau_{k-1}+1} \succ \frac{k-1}{n}\bm{\Sigma}_n$; neither is there any guarantee that $\mathsf{Tr}(\bm{V}_{\tau_{k-1}+1})$ is greater than $\frac{k-1}{n}\mathsf{Tr}(\bm{\Sigma}_n)$.

In our proof of Corollary \ref{cor:Wu}, we solve this problem by invoking a different assumption \eqref{eq:3rd-momentum-condition} than \eqref{eq:JMLR-Cor-assumption}, and the telescoping method. 

{\color{black}\paragraph{Comments on \cite{cattaneo2024yurinskiiscouplingmartingales}}  Without assuming $\bm{P}_1 =n\bm{\Sigma}_n$ almost surely, \cite{cattaneo2024yurinskiiscouplingmartingales} established the following high-dimensional Berry--Esseen bound in convex distance:
\begin{align}\label{eq:Cattaneo}
d_{\mathsf{c}}\left(\frac{1}{\sqrt{n}}\sum_{i=1}^n \bm{x}_i,\mathcal{N}(\bm{0},\bm{\Sigma}_n)\right) &\lesssim \inf_{\eta > 0} \left\{ \frac{(\beta_{2,2}d)^{\frac{1}{3}}}{\sqrt{n}\eta} + \left(\frac{\mathbb{E}\|\bm{P}_1 - n\bm{\Sigma}_n\|\cdot d}{n\eta^2}\right)^{\frac{1}{3}} + \eta \sqrt{\|\bm{\Sigma}_n^{-1}\|_{\mathsf{F}}}\right\},
\end{align}
where
\begin{align*}
\beta_{2,2} = \sum_{i=1}^n \mathbb{E}[\|\bm{x}_i\|_2^3 + \|\bm{V}_i^\frac{1}{2}\bm{z}_i\|_2^3].
\end{align*}
Here $\bm{z}_1,\bm{z}_2,\ldots,\bm{z}_n$ are i.i.d. standard Gaussian random variables independent of $\mathscr{F}_n$. 
 The second term on the right-hand-side of \eqref{eq:Cattaneo} captures the discrepancy between $\bm{P}_1$ and $n\bm{\Sigma}_n$, while 
the last term arises from the Gaussian anti-concentration inequality over the set of convex sets (see, e.g., Theorem \ref{thm:Gaussian-reminder} in the supplementary material, which we also apply in our proof of Theorem \ref{thm:Berry--Esseen-mtg}.) However, $\beta_{2,2}$ necessarily grows at least linearly in $n$. Indeed,
\begin{align*}
\sum_{i=1}^n \mathbb{E}\|\bm{V}_i^\frac{1}{2}\bm{z}_i\|_2^3 &\geq \sum_{i=1}^n \left(\mathbb{E}\|\bm{V}_i^\frac{1}{2}\bm{z}_i\|_2^2\right)^{\frac{3}{2}} \\ 
&=\sum_{i=1}^n \left(\mathbb{E}[\mathsf{Tr}(\bm{V}_i)]\right)^{\frac{3}{2}} \\ 
&\geq n^{-1/2} \cdot \left(\sum_{i=1}^n \mathbb{E}[\mathsf{Tr}(\bm{V}_i)]\right)^{\frac{3}{2}} = n [\mathsf{Tr}(\bm{\Sigma}_n)]^{\frac{3}{2}}.
\end{align*}
Consequently, the bound given by \eqref{eq:Cattaneo} cannot decay faster than
\begin{align*}
&\inf_{\eta > 0} \left\{ \frac{(\beta_{2,2}d)^{\frac{1}{3}}}{\sqrt{n}\eta} + \left(\frac{\mathbb{E}\|\bm{P}_1 - n\bm{\Sigma}_n\|\cdot d}{n\eta^2}\right)^{\frac{1}{3}} + \eta \sqrt{\|\bm{\Sigma}_n^{-1}\|_{\mathsf{F}}}\right\} \\ 
&\geq \inf_{\eta > 0} \left\{n^{-\frac{1}{6}}d^{\frac{1}{3}}[\mathsf{Tr}(\bm{\Sigma}_n)]^{\frac{1}{2}}\eta^{-1} + \eta \sqrt{\|\bm{\Sigma}_n^{-1}\|_{\mathsf{F}}}\right\} \\
&\geq d^{\frac{1}{6}}[\mathsf{Tr}(\bm{\Sigma}_n)]^{\frac{1}{4}}\|\bm{\Sigma}_n^{-1}\|_{\mathsf{F}}^{\frac{1}{4}}n^{-\frac{1}{12}},
\end{align*}
where the last line invokes the AM-GM inequality. This yields a convergence rate of at best $O(n^{-1/12})$, which is slower than the $O(n^{-1/4}\log n)$ rate obtained in Theorem \ref{thm:Berry--Esseen-mtg}. The improvement in our result comes from exploiting the Markov chain structure to derive a sharper bound on the discrepancy between $\bm{P}_1$ and $n\bm{\Sigma}_n$, together with a Stein's equation argument that tightens the contribution of the first term in \eqref{eq:Cattaneo}.
}

\subsection{Proof of Corollary \ref{thm:matrix-bernstein-mtg}}\label{app:proof-matrix-berstein-mtg}
\label{app:proof-matrix-bernstein-mtg}


Since $\mathbb{E}_{s' \sim P(\cdot \mid s)}[\bm{f}_i(s,s')] = \bm{0}$ holds for all $s \in \mathcal{S}$ and $i \in [n]$, it can be guaranteed that
\begin{align*}
\left\{\frac{1}{n}\bm{f}_i(s_{i-1},s_i)\right\}_{i=1}^n
\end{align*}
is a martingale difference sequence adapted to the filtration $\{\mathcal{F}_i\}_{i=1}^n$; therefore, by defining
\begin{align*}
&\bm{X}_i := \frac{1}{n}\bm{f}_i(s_{i-1},s_i), \\ 
&\bm{Y}_i := \sum_{j=1}^i \bm{X}_j, \quad \text{and} \\ 
&W_i:= \sum_{j=1}^i \mathbb{E}\|\bm{X}_j\|_2^2 \mid \mathscr{F}_{j-1},
\end{align*}
we can apply Theorem \ref{thm:Hilbert-Freedman} to obtain that for every $\sigma^2$ and $\delta \in (0,1)$, it can be guaranteed with probability at least $1-\frac{\delta}{2}$ that
\begin{align}\label{eq:matrix-freedman}
\left\|\frac{1}{n}\sum_{i=1}^n \bm{f}_i(s_{i-1},s_i)\right\|_2 \lesssim \sqrt{\frac{\sigma^2}{n}\log \frac{1}{\delta}} + \frac{F}{n} \log \frac{1}{\delta}, \quad \text{or} \quad W_n \geq \sigma^2.
\end{align}
In what follows, we aim to bound the $W_n$ by controlling its different from $\mathsf{Tr}(\bm{\Sigma}_n)$. Specifically, we can define
\begin{align*}
\bm{g}_{i-1}(s) = \frac{1}{n^2} \mathbb{E}_{s' \sim P(\cdot \mid s)} \|\bm{f}_i(s,s')\|_2^2 
\end{align*}
and apply Corollary \ref{cor:markov-hoeffding-1d} with $a_i = 0$, $b_i = F^2/n^2$ to obtain
\begin{align*}
W_n - \mathsf{Tr}(\bm{\Sigma}_n) &= \sum_{i=1}^n [g_{i}(s_{i-1})-\mu(g_i)]\\ 
&\lesssim \sqrt{\frac{q}{1-\lambda}} \frac{F}{\sqrt{n}}\log^{\frac{1}{2}}\left(\frac{1}{\delta}\left\|\frac{\mathrm{d}\nu}{\mathrm{d}\mu}\right\|_{\mu,p}\right)
\end{align*}
with probability at least $1-\frac{\delta}{2}$. Hence, the triangle inequality directly yields
\begin{align}\label{eq:matrix-bernstein-quadratic}
W_n \leq \mathsf{Tr}(\bm{\Sigma}_n) +  \sqrt{\frac{q}{1-\lambda}} \frac{F}{\sqrt{n}}\log^{\frac{1}{2}}\left(\frac{1}{\delta}\left\|\frac{\mathrm{d}\nu}{\mathrm{d}\mu}\right\|_{\mu,p}\right)
\end{align}
The theorem follows by combining \eqref{eq:matrix-freedman}, \eqref{eq:matrix-bernstein-quadratic} using a union bound argument and taking
\begin{align*}
\sigma^2 = \mathsf{Tr}(\bm{\Sigma}_n) +  \sqrt{\frac{q}{1-\lambda}} \frac{F}{\sqrt{n}}\log^{\frac{1}{2}}\left(\frac{1}{\delta}\left\|\frac{\mathrm{d}\nu}{\mathrm{d}\mu}\right\|_{\mu,p}\right).
\end{align*}

\subsection{Proof of Corollary \ref{thm:Berry--Esseen-mtg}}\label{app:proof-Berry--Esseen-mtg}

Part of the proof is inspired by the proof of Lemma B.8 in \citet{cattaneo2024yurinskiiscouplingmartingales} and Theorem 2.1 in \cite{
belloni2018highdimensionalcentrallimit}. We firstly notice that the convex distance satisfies
\begin{align*}
d_{\mathsf{C}}\left(\frac{1}{\sqrt{n}}\sum_{i=1}^n \bm{f}_i(s_{i-1}, s_i), \mathcal{N}(\bm{0},\bm{\Sigma}_n)\right) = d_{\mathsf{C}}\left(\frac{1}{\sqrt{n}}\sum_{i=1}^n \bm{\Sigma}_n^{-\frac{1}{2}}\bm{f}_i(s_{i-1}, s_i), \mathcal{N}(\bm{0},\bm{I})\right);
\end{align*}
hence, we can assume without loss of generality that $\bm{\Sigma}_n = \bm{I}$.\footnote{To see this, take $\bm{g}_i = \bm{\Sigma}_n^{-\frac{1}{2}}\bm{f}_i$ and the proof follows for the sequence $\{\bm{g}_i\}_{1 \leq i \leq n}$.}

Borrowing the notation from the proof of Theorem \ref{thm:Srikant-generalize} and Corollary \ref{cor:Wu}, we define, for each $k=1,\ldots,n$,
\begin{align*}
&\bm{S}_k := \sum_{j=1}^k \bm{f}_j(s_{j-1},s_j), \\ 
&\bm{V}_k := \mathbb{E}[\bm{f}_k(s_{k-1},s_k) \bm{f}_k^\top(s_{k-1},s_k)\mid \mathscr{F}_{k-1}]   \\ 
&\bm{T}_k := \sum_{j=k}^n \bm{V}_j^{1/2}\bm{z}_j, \text{ and } \\ 
&\bm{P}_k := \sum_{j=k}^n \bm{V}_j.
\end{align*}
This proof approaches the theorem in four steps:
\begin{enumerate}
\item Find a value $\kappa = \kappa(n) = O(\frac{\log n}{\sqrt{n}})$, such that 
\begin{align*}
\mathbb{P}(\|\bm{P}_1 - n\bm{I}\| \geq n\kappa) \leq n^{-\frac{1}{2}}.
\end{align*}
\item Construct a martingale $\{\tilde{\bm{S}}_{j}\}_{j=1}^N$, whose differentiation satisfies the condition of Theorem \ref{thm:Srikant-generalize}, 
such that
\begin{align*}
\mathbb{E}[\tilde{\bm{S}}_N \tilde{\bm{S}}_N^\top] = n(1+ \kappa) \bm{I}, \quad \text{and} \quad \mathbb{P}\left(\|\bm{S}_n - \tilde{\bm{S}}_{N}\|_2 > \sqrt{2d\kappa n  \log n}\right)\leq  n^{-\frac{1}{2}}.
\end{align*}
\item Apply Theorem \ref{thm:Srikant-generalize} to $\tilde{\bm{S}}_{N}$ to derive a Berry--Esseen bound between the distributions of $\tilde{\bm{S}}_{N}$ and $\mathcal{N}(\bm{0},(1+\kappa) \bm{I}))$;
\item Combine the results above to achieve the desired Berry--Esseen bound.
\end{enumerate}
\paragraph{Step 1: find $\kappa$} Due to the Markovian property, the matrix $\bm{V}_k$ is a function of $s_{k-1}$ for every $k \in [n]$. Define
\begin{align*}
\bar{\bm{V}}_k = \mathbb{E}_{s \sim \mu,s' \sim P(\cdot \mid s)}[\bm{f}_i(s,s')\bm{f}_i^\top (s,s')],
\end{align*}
then it can be guaranteed that
\begin{align*}
\mathbb{E}_{s_{k-1}\sim \mu} [\bm{V}_k] = \bar{\bm{V}}_k, 
\end{align*}
and that
\begin{align*}
\|\bm{V}_k - \bar{\bm{V}}_k\| \leq M_k^2, \quad \text{a.s.}
\end{align*}
hold for every $k \in [n]$.
Consequently, a direct application of Theorem \ref{thm:matrix-hoeffding} yields
\begin{align*}
\|\bm{P}_1 - n\bm{I}\|= n \cdot \left\|\frac{1}{n}\sum_{i=1}^n (\bm{V}_k - \bar{\bm{V}}_k)\right\| &\leq \sqrt{\sum_{i=1}^n M_i^4} \sqrt{\frac{20q}{1-\lambda} \log \left(\frac{2d}{n^{-\frac{1}{2}}}\left\|\frac{\mathrm{d}\nu}{\mathrm{d}\mu}\right\|_{\mu,p}\right)} \\ 
 &\leq \sqrt{\sum_{i=1}^n M_i^4}\sqrt{\frac{40q}{1-\lambda} \log \left(2d n \left\|\frac{\mathrm{d}\nu}{\mathrm{d}\mu}\right\|_{\mu,p}\right)},
\end{align*}
with probability at least $1-n^{-\frac{1}{2}}$. In what follows, we take
\begin{align*}
\kappa &= \frac{1}{\sqrt{n}}\sqrt{\frac{\sum_{i=1}^n M_i^4 }{n} \frac{40q}{1-\lambda} \log \left(2d n \left\|\frac{\mathrm{d}\nu}{\mathrm{d}\mu}\right\|_{\mu,p}\right)} = \frac{\bar{M}^2}{\sqrt{n}} \sqrt{\frac{40q}{1-\lambda} \log \left(2d n \left\|\frac{\mathrm{d}\nu}{\mathrm{d}\mu}\right\|_{\mu,p}\right)}.
\end{align*}
\paragraph{Step 2: Construct $\{\tilde{\bm{S}}_j\}_{j=1}^N$} Define the stopping time
\begin{align*}
\tau := \sup\left\{t \leq n: \sum_{i=1}^t \bm{V}_i \preceq n(1+\kappa)\bm{I}\right\},
\end{align*}
and let
{\color{black}
\begin{align*}
&m:= \left\lceil \frac{d}{M^2} \left\|n (1+\kappa)\bm{I} - \sum_{i=1}^\tau \bm{V}_i\right\| \right\rceil, \quad \text{and} \quad N:= \left\lceil \frac{dn (1+\kappa)}{M^2} \right\rceil + n.
\end{align*}
}
By definition, it can be guaranteed that $n+m \leq N$, and that $N \asymp n$. We now construct a martingale difference process $\{ \tilde{\bm{x}}_i \}_{i=1}^{N}$ in the following way: for $1 \leq i \leq \tau$, let $\tilde{\bm{x}}_i = \bm{f}_i(s_{i-1},s_i)$ and for $\tau < i \leq \tau + m$, let 
{\color{black}
\begin{align*}
\tilde{\bm{x}}_i = \frac{1}{\sqrt{m}} \left(n (1+\kappa)\bm{I} - \sum_{i=1}^\tau \bm{V}_i \right)^{\frac{1}{2}}\bm{\epsilon}_i,
\end{align*}
}
where $\bm{\epsilon_i} = (\epsilon_{i1},\epsilon_{i2},\cdots,\epsilon_{id})$ and $\{\epsilon_{ij}\}_{\tau < i \leq \tau + m, 1 \leq j \leq d}$ are $i.i.d.$ Rademacher random variables independent of the $s_i$'s, i.e.
\begin{align*}
\epsilon_{ij} = \begin{cases}
+1, \quad \text{w.p. }\frac{1}{2}; \\ 
-1, \quad \text{w.p. }\frac{1}{2}.
\end{cases}
\end{align*}
In particular, it holds that, for any $ \tau < i \leq \tau + m$,  
\begin{align*}
\mathbb{E}[\tilde{\bm{x}}_i] = \bm{0}, \quad \mathbb{E}[\tilde{\bm{x}}_i \tilde{\bm{x}}_i^\top\mid \mathscr{F}_{i-1}] = \frac{1}{m} \left(n (1+\kappa)\bm{I} - \sum_{i=1}^\tau \bm{V}_i \right),
\end{align*}
and 
{\color{black}
\begin{align*}
\|\tilde{\bm{x}}_i\|_2^2 \leq \frac{d}{m} \left\|n (1+\kappa)\bm{I} - \sum_{i=1}^\tau \bm{V}_i\right\| \leq M^2
\end{align*}
}
almost surely. Finally, if $\tau + m < i \leq  N$, we simply set $\tilde{\bm{x}}_i = \bm{0}$. 

The martingale $\{\tilde{\bm{S}}_j\}_{j=1}^N$ is naturally constructed by
\begin{align*}
\tilde{\bm{S}}_j = \sum_{i=1}^j \tilde{\bm{x}}_i.
\end{align*}
In this step, we explore the difference between $\bm{S}_n$ and $\tilde{\bm{S}}_N$. Specifically, observe that
\begin{align*}
&\mathbb{P}\left(\|\bm{S}_n - \tilde{\bm{S}}_{N}\|_2 > \sqrt{2d\kappa n\log n}\right) \\ 
&= \mathbb{P}(\|\bm{S}_n - \tilde{\bm{S}}_{N}\|_2 > \sqrt{2d\kappa n\log n}, \|\bm{P}_1 - n\bm{I}\| \leq \kappa n) \\ 
&+ \mathbb{P}(\|\bm{S}_n - \tilde{\bm{S}}_{N}\|_2 > \sqrt{2d\kappa n\log n}, \|\bm{P}_1 - n\bm{I}\| > \kappa n)  \\ 
&\leq \mathbb{P}(\|\bm{S}_n - \tilde{\bm{S}}_{N}\|_2 > \sqrt{2d\kappa n\log n}, \|\bm{P}_1 - n\bm{I}\| \leq \kappa n) + \mathbb{P}(\|\bm{P}_1 - n\bm{I}\| > \kappa n) \\ 
&\leq \mathbb{P}(\|\bm{S}_n - \tilde{\bm{S}}_{N}\|_2 > \sqrt{2d\kappa n\log n}, \|\bm{P}_1 - n\bm{I}\| \leq \kappa n ) + n^{-\frac{1}{2}}.
\end{align*}
To bound the first term on the left hand side of the last inequality, notice that, when $\|\bm{P}_1 - n\bm{I}\| \leq \kappa n$,
\begin{align*}
\bm{P}_1 = \sum_{i=1}^n \bm{V}_i \preceq n(1+\kappa) \bm{I}.
\end{align*}
Thus, on the same event, $\tau = n$ and for every $j \in [d]$, 
\begin{align*}
\lambda_j \leq \|n(1+ \kappa) \bm{I} - \bm{P}_1\| \leq 2\kappa n.
\end{align*}

Consequently, 
{\color{black}
\begin{align*}
\|\bm{S}_n - \tilde{\bm{S}}_{N}\|_2^2 &= \left\|\sum_{i=n+1}^{n+m} \tilde{\bm{x}}_i \right\|_2^2 \\
&= \frac{1}{m}\left\|(n(1+ \kappa) \bm{I} - \bm{P}_1)^{\frac{1}{2}}\sum_{i=n+1}^{n+m} \bm{\epsilon}_i\right\|_2^2 \\ 
&\leq \frac{\|n(1+ \kappa) \bm{I} - \bm{P}_1\|}{m} \left\|\sum_{i=n+1}^{n+m} \bm{\epsilon}_i\right\|_2^2 \\ 
&\leq \frac{2\kappa n}{m}  \sum_{j=1}^d \left(\sum_{i=n+1}^{n+m}\epsilon_{ij}\right)^2.
\end{align*}
}
Since $\{\epsilon_{ij}\}_{n+1 \leq i \leq n+m, 1 \leq j \leq d}$ are $i.i.d.$ Rademacher random variables, the Hoeffding's inequality guarantees that
\begin{align*}
\left|\sum_{i=n+1}^{n+m}\epsilon_{ij}\right| \lesssim \sqrt{m \log n}
\end{align*}
with probability at least $1-2n^{-\frac{1}{2}}$. As a direct consequence, the difference between $\bm{S}_n$ and $\tilde{\bm{S}}_N$ can be bounded by
\begin{align*}
\|\bm{S}_n - \tilde{\bm{S}}_{N}\|_2^2 \lesssim 2 d \kappa n\log n
\end{align*}
with probability at least $1-2n^{-\frac{1}{2}}$. In combination, we obtain
\begin{align*}
\mathbb{P}(\|\bm{S}_n - \tilde{\bm{S}}_N\|_2 \gtrsim \sqrt{2d\kappa n\log n}) \leq 3n^{-\frac{1}{2}}.
\end{align*}

\paragraph{Step 3: Berry--Esseen bound on $\tilde{\bm{S}}_N$} It can be easily verified that the sequence $\{ \tilde{\bm{x}}_i\}_{i=1}^N$ is a martingale difference such that
\begin{align*}
&\sum_{i=1}^N \mathbb{E}[\tilde{\bm{x}}_i \tilde{\bm{x}}_i^\top \mid \mathscr{F}_{i-1}] = n(1+ \kappa) \bm{I}, \quad \text{and} \quad \|\bm{\tilde{x}}_i\|_2 \leq M \quad \text{a.s., } \forall i \in [N].
\end{align*}
Hence, Corollary \ref{cor:Wu} can be applied on $\tilde{\bm{S}}_N$ to obtain 
\begin{align*}
&d_{\mathsf{W}}\left(\frac{\tilde{\bm{S}}_N}{\sqrt{n}}, \mathcal{N}(\bm{0},(1+ \kappa) \bm{I})\right) \lesssim Md \log d\cdot \frac{\log^2 n}{\sqrt{n}}.
\end{align*} 
\paragraph{Step 4: Completing the proof} By the triangle inequality,
\begin{align}\label{eq:Wu-initial-decompose}
d_{\mathsf{C}}\left(\frac{\bm{S}_n}{\sqrt{n}},\mathcal{N}(\bm{0},\bm{I})\right) \leq d_{\mathsf{C}}\left(\frac{\bm{S}_n}{\sqrt{n}},\mathcal{N}(\bm{0},(1+\kappa) \bm{I})\right) + d_{\mathsf{C}}\left(\mathcal{N}(\bm{0},\bm{I}),\mathcal{N}(\bm{0},(1 + \kappa) \bm{I})\right)
\end{align}
where the second term on the right-hand-side can be bounded using a direct application of Theorem \ref{thm:DMR} by
\begin{align}\label{eq:Wu-Gaussian-comparison}
d_{\mathsf{C}}\left(\mathcal{N}(\bm{0},\bm{I}),\mathcal{N}(\bm{0},(1 + \kappa) \bm{I})\right) \lesssim \kappa \|\bm{I}\|_{\mathsf{F}} \leq \kappa \sqrt{d}  = O\left( \bar{M}^2\sqrt{\frac{d}{n}} \log n \right).
\end{align}
For the first term, consider any convex set $\mathcal{A} \subset \mathbb{R}^d$, the triangle inequality guarantees, for every $x > 0$, that
\begin{align}\label{eq:Wu-decompose-upper}
\mathbb{P}\left(\frac{\bm{S}_n}{\sqrt{n}} \in \mathcal{A}\right)&= \mathbb{P}\left(\frac{\bm{S}_n}{\sqrt{n}} \in \mathcal{A}, \left\|\frac{\bm{S}_n - \tilde{\bm{S}}_N}{\sqrt{n}}\right\|_2 \leq x\right) + \mathbb{P}\left(\frac{\bm{S}_n}{\sqrt{n}} \in \mathcal{A}, \left\|\frac{\bm{S}_n - \tilde{\bm{S}}_N}{\sqrt{n}}\right\|_2 > x\right) \nonumber \\ 
&\leq \mathbb{P}\left(\frac{\tilde{\bm{S}}_N}{\sqrt{n}} \in \mathcal{A}^x \right) + \mathbb{P}\left(\left\|\frac{\bm{S}_n - \tilde{\bm{S}}_N}{\sqrt{n}}\right\|_2 > x\right)\nonumber \\
&\leq \mathbb{P}\left(\frac{\tilde{\bm{T}}_N}{\sqrt{n}} \in \mathcal{A}^x \right) + d_{\mathsf{C}}\left(\frac{\tilde{\bm{S}}_N}{\sqrt{n}},\frac{\tilde{\bm{T}}_N}{\sqrt{n}}\right)+ \mathbb{P}\left(\left\|\frac{\bm{S}_n - \tilde{\bm{S}}_N}{\sqrt{n}}\right\|_2 > x\right) \nonumber \\ 
&\leq \mathbb{P}\left(\frac{\tilde{\bm{T}}_N}{\sqrt{n}} \in \mathcal{A} \right) + (1+\kappa)^{-1/2} d^{\frac{1}{4}} x + d_{\mathsf{C}}\left(\frac{\tilde{\bm{S}}_N}{\sqrt{n}},\frac{\tilde{\bm{T}}_N}{\sqrt{n}}\right)+ \mathbb{P}\left(\left\|\frac{\bm{S}_n - \tilde{\bm{S}}_N}{\sqrt{n}}\right\|_2 > x\right).
\end{align}
Here, $\tilde{\bm{T}}_N \sim \mathcal{N}(\bm{0},n(\bm{\Sigma}_n +\kappa \bm{I}))$ and we applied Theorem \ref{thm:Gaussian-reminder} on the last inequality. On the right-most part of the inequality, the third term can be bounded by the Berry--Esseen bound on $\tilde{\bm{S}}_N$ and Theorem \ref{thm:Gaussian-convex-Wass}:
\begin{align*}
d_{\mathsf{C}}\left(\frac{\tilde{\bm{S}}_N}{\sqrt{n}},\frac{\tilde{\bm{T}}_N}{\sqrt{n}}\right)
&\leq \|(1+\kappa)^{-1} \bm{I}\|_{\mathsf{F}}^{\frac{1}{4}} \sqrt{d_{\mathsf{W}}\left(\frac{\tilde{\bm{S}}_N}{\sqrt{n}},\frac{\tilde{\bm{T}}_N}{\sqrt{n}}\right)} \\ 
&\leq (1+\kappa)^{-\frac{1}{4}}d^{\frac{1}{8}} \cdot \sqrt{M} d^{\frac{1}{2}}\log^{\frac{1}{2}} d n^{-\frac{1}{4}}\log n \\ 
&\leq \sqrt{M}d^{\frac{5}{8}}\log^{\frac{1}{2}} d n^{-\frac{1}{4}}\log n
\end{align*}
Therefore, taking $x = \sqrt{2d\kappa \log n}$ in \eqref{eq:Wu-decompose-upper} yields
\begin{align*}
\mathbb{P}\left(\frac{\bm{S}_n}{\sqrt{n}} \in \mathcal{A}\right)& \leq \mathbb{P}\left(\frac{\tilde{\bm{T}}_N}{\sqrt{n}} \in \mathcal{A} \right) +  d^{\frac{1}{4}} \cdot \sqrt{2d\kappa\log n} \\
&+ \sqrt{M}d^{\frac{5}{8}}\log^{\frac{1}{2}} d n^{-\frac{1}{4}}\log n + 3n^{-\frac{1}{2}}.
\end{align*}
A simple reorganization yields
\begin{align}\label{eq:Wu-upper-bound}
&  \mathbb{P}\left(\frac{\bm{S}_n}{\sqrt{n}} \in \mathcal{A}\right)- \mathbb{P}\left(\frac{\tilde{\bm{T}}_N}{\sqrt{n}} \in \mathcal{A} \right)\nonumber \\ 
&\lesssim \left\{\bar{M}\left(\frac{q}{1-\lambda}\right)^{\frac{1}{4}}d^{\frac{3}{4}}\log^{\frac{1}{4}}\left(d\left\|\frac{\mathrm{d}\nu}{\mathrm{d}\mu}\right\|_{\mu,p}\right)+ \sqrt{M} d^{\frac{5}{8}}\log^{\frac{1}{2}} d\right\} n^{-\frac{1}{4}}\log n.
\end{align}
Meanwhile, a union bound argument and the triangle inequality guarantee
\begin{align*}
\mathbb{P}\left(\frac{\bm{S}_n}{\sqrt{n}} \in \mathcal{A}\right)& \geq \mathbb{P}\left(\frac{\bm{S}_n}{\sqrt{n}} \in \mathcal{A} \cup \left\|\frac{\bm{S}_n - \tilde{\bm{S}}_N}{\sqrt{n}}\right\|_2 > x\right) - \mathbb{P}\left(\left\|\frac{\bm{S}_n - \tilde{\bm{S}}_N}{\sqrt{n}}\right\|_2 > x\right) \\ 
&\geq \mathbb{P}\left(\frac{\tilde{\bm{S}}_N}{\sqrt{n}} \in \mathcal{A}^{-x} \right) - \mathbb{P}\left(\left\|\frac{\bm{S}_n - \tilde{\bm{S}}_N}{\sqrt{n}}\right\|_2 > x\right) \\ 
&\geq \mathbb{P}\left(\frac{\tilde{\bm{T}}_N}{\sqrt{n}} \in \mathcal{A}^{-x} \right) - d_{\mathsf{C}}\left(\frac{\tilde{\bm{S}}_N}{\sqrt{n}},\frac{\tilde{\bm{T}}_N}{\sqrt{n}}\right)- \mathbb{P}\left(\left\|\frac{\bm{S}_n - \tilde{\bm{S}}_N}{\sqrt{n}}\right\|_2 > x\right)  \\ 
&\geq \mathbb{P}\left(\frac{\tilde{\bm{T}}_N}{\sqrt{n}} \in \mathcal{A} \right) - (1+\kappa)^{-\frac{1}{2}} d^{\frac{1}{4}}x - d_{\mathsf{C}}\left(\frac{\tilde{\bm{S}}_N}{\sqrt{n}},\frac{\tilde{\bm{T}}_N}{\sqrt{n}}\right)- \mathbb{P}\left(\left\|\frac{\bm{S}_n - \tilde{\bm{S}}_N}{\sqrt{n}}\right\|_2 > x\right);
\end{align*}
consequently, it can be symmetrically proved that
\begin{align}\label{eq:Wu-lower-bound}
&\mathbb{P}\left(\frac{\tilde{\bm{T}}_N}{\sqrt{n}} \in \mathcal{A} \right)-\mathbb{P}\left(\frac{\bm{S}_n}{\sqrt{n}} \in \mathcal{A}\right) \nonumber  \\ 
&\lesssim \left\{\bar{M}\left(\frac{q}{1-\lambda}\right)^{\frac{1}{4}}d^{\frac{3}{4}}\log^{\frac{1}{4}}\left(d\left\|\frac{\mathrm{d}\nu}{\mathrm{d}\mu}\right\|_{\mu,p}\right)+ \sqrt{M} d^{\frac{5}{8}}\log^{\frac{1}{2}} d\right\} n^{-\frac{1}{4}}\log n.
\end{align}
The theorem follows by combining \eqref{eq:Wu-initial-decompose}, \eqref{eq:Wu-Gaussian-comparison}, \eqref{eq:Wu-upper-bound}, \eqref{eq:Wu-lower-bound} and taking a supremum over $\mathcal{A} \in \mathscr{C}$.

\section{Proof of results regarding TD learning}\label{app:proof-TD}

Throughout this section, we denote
\begin{align}
&\bm{\Delta}_t = \bm{\theta}_t - \bm{\theta}^\star, \quad \forall t \in [T], \quad \text{and}  \\
&\bar{\bm{\Delta}}_T = \bar{\bm{\theta}}_T - \bm{\theta}^\star.
\end{align}

Furthermore, for every $t = 0,1,\ldots,T$, we denote 
\begin{align*}
\mathbb{E}_t[\cdot] := \mathbb{E}[\cdot \mid s_0,s_1,\ldots,s_t].
\end{align*}
Without any subscript, the operator $\mathbb{E}$ represents taking expectation with respect to all the samples starting from $s_0$.

\subsection{$L^2$ convergence of the TD estimation error}
The following theorem captures the asymptotic property of $\mathbb{E}\|\bm{\Delta}_t\|_2^2$ with Markov samples and is useful in our proofs for other results. 
Note that the bound holds, non-asymptotically, for all $t \geq t^\star$ where $t^\star$ is a problem-dependent quantity; we state it as an asymptotic result only for convenience.

\begin{theorem}\label{thm:markov-L2-convergence}
Consider TD with Polyak-Ruppert averaging~\eqref{eq:TD-update-all} with Markov samples and decaying stepsizes $\eta_t = \eta_0 t^{-\alpha}$ for $\alpha \in (\frac{1}{2},1)$. Suppose that the Markov transition kernel has a unique stationary distribution, a strictly positive spectral gap, and mixes exponentially as indicated by Assumption \ref{as:mixing}.  It can then be guaranteed that when $t \to \infty$,
\begin{align*}
\mathbb{E}\big[\|\bm{\Delta}_t\|_2^2\big] \lesssim (2\|\bm{\theta}^\star\|_2+1)^2 \left[\frac{1}{(1-\rho)^2}\frac{\eta_0}{\lambda_0(1-\gamma)}t^{-\alpha} \log^2 t + o\left(t^{-\alpha} \log^2 t\right)\right].
\end{align*}
\end{theorem}

\begin{proof}
We firstly construct an iterative relation along the sequence $\{\mathbb{E}\|\bm{\Delta}_t\|_2^2\}_{t \geq 0}$ in general, and then refine our analysis using a specific choice of stepsizes. The TD iteration rule \eqref{eq:TD-update-all} directly implies that
\begin{align*}
\|\bm{\Delta}_{t}\|_2^2 &= \|\bm{\Delta}_{t-1}\|_2^2 -2\eta_t \bm{\Delta}_{t-1}^\top (\bm{A}_t \bm{\theta}_{t-1} - \bm{b}_t) + \eta_t^2 \|\bm{A}_t \bm{\theta}_{t-1} - \bm{b}_t\|_2^2 \\ 
&\leq \|\bm{\Delta}_{t-1}\|_2^2 -2\eta_t \bm{\Delta}_{t-1}^\top (\bm{A}_t \bm{\theta}^\star + \bm{A}_t \bm{\Delta}_{t-1} - \bm{b}_t) + 2\eta_t^2 (\|\bm{A}_t \bm{\Delta}_{t-1}\|_2^2 + \|\bm{A}_t \bm{\theta}^\star - \bm{b}_t\|_2^2)\\
&= \|\bm{\Delta}_{t-1}\|_2^2 -2\eta_t \bm{\Delta}_{t-1}^\top \bm{A \Delta}_{t-1} - 2\eta_t\bm{\Delta}_{t-1}^\top(\bm{A}_t-\bm{A})\bm{\Delta}_{t-1}\\ 
& - 2\eta_t \bm{\Delta}_{t-1}^\top (\bm{A}_t \bm{\theta}^\star - \bm{b}_t)+ 2\eta_t^2 (\|\bm{A}_t \bm{\Delta}_{t-1}\|_2^2 + \|\bm{A}_t \bm{\theta}^\star - \bm{b}_t\|_2^2).
\end{align*}
Since $\bm{\Delta}_{t-1}^\top \bm{A \Delta}_{t-1} \geq \lambda_0(1-\gamma)\|\bm{\Delta}\|_{t-1}$ due to \eqref{eq:lemma-A-1} and $\|\bm{A}_t\|\leq 1+\gamma$, we can bound $\mathbb{E}\|\bm{\Delta}_{t}\|_2^2$ by 
\begin{align}\label{eq:markov-L2-iter}
\mathbb{E}\|\bm{\Delta}_{t}\|_2^2 &\leq \mathbb{E}\|\bm{\Delta}_{t-1}\|_2^2 -2\lambda_0(1-\gamma) \eta_{t} \mathbb{E}\|\bm{\Delta}_{t-1}\|_2^2 +2\eta_t^2(1+\gamma)^2  \mathbb{E}\|\bm{\Delta}_{t-1}\|_2^2 \nonumber \\ 
&- 2\eta_t \mathbb{E}[\bm{\Delta}_{t-1}^\top(\bm{A}_t - \bm{A})\bm{\Delta}_{t-1}] - 2\eta_t \mathbb{E}[\bm{\Delta}_{t-1}^\top (\bm{A}_t \bm{\theta}^\star - \bm{b}_t)] + 2\eta_t^2 \mathbb{E}\|\bm{A}_t \bm{\theta}^\star - \bm{b}_t\|_2^2 \nonumber \\ 
&= \underset{I_1}{\underbrace{\left(1-2\lambda_0(1-\gamma)\eta_t + 2\eta_t^2(1+\gamma)^2\right) \mathbb{E}\|\bm{\Delta}_{t-1}\|_2^2 }}+ \underset{I_2}{\underbrace{2\eta_t^2 \mathbb{E}\|\bm{A}_t \bm{\theta}^\star - \bm{b}_t\|_2^2}} \nonumber \\ 
&- 2\eta_t\underset{I_3}{\underbrace{ \mathbb{E}[\bm{\Delta}_{t-1}^\top(\bm{A}_t - \bm{A})\bm{\Delta}_{t-1}]}} - 2 \eta_t \underset{I_4} {\underbrace{\mathbb{E}[\bm{\Delta}_{t-1}^\top (\bm{A}_t \bm{\theta}^\star - \bm{b}_t)]}}.
\end{align}

In this expression, $I_1$ is contractive with respect to $\mathbb{E}\|\bm{\Delta}_{t-1}\|_2^2$ as long as $\eta_t$ is sufficiently small, while $I_2$ is proportional to $\eta_t^2$ since $\mathbb{E}\|\bm{A}_t\bm{\theta}^\star - \bm{b}_t\|_2^2$ is independent of $t$. These two terms are desirable and can be left as they are; 

The difficulty of this proof lies in bounding $I_3$ and $I_4$ using Markov samples. Notice that with $i.i.d.$ sampling, both terms are actually $0$; hence, we aim to bound them by applying the mixing property of the Markov chain. 

To simplify notation, throughout the proof, we denote 
\begin{align*}
t_{\mix}:=t_{\mix}(\eta_t) + 1,
\end{align*}
so that with Markov samples, $s_{t-1}\mid \mathscr{F}_{t-t_{\mix}} \sim P^{t_{\mix}-1}(\cdot \mid s_{t-t_{\mix}})$, and that
\begin{align*}
d_{\mathsf{TV}}(P^{t_{\mix}-1}(\cdot \mid s_{t-t_{\mix}}),\mu) \leq \eta_t.
\end{align*}
Meanwhile, Assumption \ref{as:mixing} implies that

\begin{align}\label{eq:tmix-bound-L2}
t_{\mix} \leq \frac{\log(m/\eta_t)}{\log(1/\rho)} +1 = \frac{\log(m/\eta_0) + \alpha \log t}{\log(1/\rho)} + 1< \frac{\log(m/\eta_0) + \alpha \log t}{1-\rho} + 1.
\end{align}

In other words, $t_{\mix}(\eta_t)$ grows at most by $O(\log t)$; therefore, in what follows, we can assume that $t$ is large enough such that $t \geq 2t_{\mix}$. The essential idea of bounding $I_3$ and $I_4$ involves decomposing $\bm{\Delta}_{t-1}$ by
\begin{align*}
\bm{\Delta}_{t-1} = (\bm{\Delta}_{t-1} - \bm{\Delta}_{t-t_{\mix}}) + \bm{\Delta}_{t-t_{\mix}},
\end{align*}
where the norm of $(\bm{\Delta}_{t-1} - \bm{\Delta}_{t-t_{\mix}})$ is bounded by the decaying stepsizes, while the correlation between $\bm{\Delta}_{t-t_{\mix}}$ and $\bm{A}_t, \bm{b}_t$ is bounded by the mixing property of the Markov chain.

We address $I_3$ and $I_4$ respectively.

\paragraph{Bounding $I_3$} The definition of $t_{\mix}$ implies
\begin{align*}
&\left|\mathbb{E}[\bm{\Delta}_{t-t_{\mix}}^\top (\bm{A}_t - \bm{A})\bm{\Delta}_{t-t_{\mix}}]\right|\\
&= \left|\mathbb{E}[\mathbb{E}_{t-t_{\mix}}[\bm{\Delta}_{t-t_{\mix}}^\top (\bm{A}_t - \bm{A})\bm{\Delta}_{t-t_{\mix}}]]\right| \\ 
&= \left|\mathbb{E} \left[\mathbb{E}_{s_{t-1} \sim P^{t_{\mix}-1}(\cdot \mid s_{t-t_{\mix}})}[\mathbb{E}_{s_t \sim P(\cdot \mid s_{t-1})}[\bm{\Delta}_{t-t_{\mix}}^\top\bm{A}_t \bm{\Delta}_{t-t_{\mix}}]] - \mathbb{E}_{s_{t-1} \sim \mu}[\mathbb{E}_{s_t \sim P(\cdot \mid s_{t-1})}[\bm{\Delta}_{t-t_{\mix}}^\top\bm{A}_t \bm{\Delta}_{t-t_{\mix}}]  \right]] \right|\\ 
&\leq \mathbb{E}\sup_{s_{t-1}} |\mathbb{E}_{s_t \sim P(\cdot \mid s_{t-1})}\bm{\Delta}_{t-t_{\mix}}^\top\bm{A}_t \bm{\Delta}_{t-t_{\mix}}| \cdot d_{\mathsf{TV}}(P^{t_{\mix}-1}(\cdot \mid s_{t-t_{\mix}}),\mu) \\ 
&\leq \mathbb{E}[2\|\bm{\Delta}_{t - t_{\mix}}\|_2^2 ] \cdot \eta_t;
\end{align*}
notice that the inequality on the fourth line follows from the basic property of the TV distance.
As a direct consequence, $I_3$ is featured by 
\begin{align*}
I_3 &= \mathbb{E}[\bm{\Delta}_{t-t_{\mix}}^\top (\bm{A}_t - \bm{A})\bm{\Delta}_{t-t_{\mix}}] + 2\mathbb{E}[\bm{\Delta}_{t-t_{\mix}}^\top (\bm{A}_t - \bm{A})(\bm{\Delta}_{t-1} - \bm{\Delta}_{t-t_{\mix}})] \\ 
&+ \mathbb{E}[(\bm{\Delta}_{t-1} - \bm{\Delta}_{t-t_{\mix}})^\top (\bm{A}_t - \bm{A})(\bm{\Delta}_{t-1} - \bm{\Delta}_{t-t_{\mix}})] \\ 
&\geq -2\eta_t\mathbb{E}\|\bm{\Delta}_{t-t_{\mix}}\|_2^2 -4 \mathbb{E}[\|\bm{\Delta}_{t-t_{\mix}}\|_2\|\bm{\Delta}_{t-1} - \bm{\Delta}_{t-t_{\mix}}\|_2]-2 \mathbb{E}\|\bm{\Delta}_{t-1} - \bm{\Delta}_{t-t_{\mix}}\|_2^2 \\ 
&= -2\eta_t\mathbb{E}\|\bm{\Delta}_{t-t_{\mix}}\|_2^2 -4 \mathbb{E}[\|\bm{\Delta}_{t-t_{\mix}}\|_2\|\bm{\theta}_{t-1} - \bm{\theta}_{t-t_{\mix}}\|_2]-2 \mathbb{E}\|\bm{\theta}_{t-1} - \bm{\theta}_{t-t_{\mix}}\|_2^2.
\end{align*}
To lower bound the right-hand-side of the last expression, the following lemma comes in handy, with its proof postponed to Appendix \ref{app:proof-lemma-E-delta-tmix}.
\begin{lemma}\label{lemma:E-delta-tmix}
For the TD iterations \eqref{eq:TD-update-all} with Markov samples and non-increasing stepsizes $\eta_1 \geq \ldots \geq \eta_T$ it holds that, for all $t \geq t_{\mix}$, 
\begin{subequations}
\begin{align}
&\mathbb{E}\|\bm{\theta}_{t-1} - \bm{\theta}_{t-t_{\mix}}\|_2 \leq t_{\mix}\eta_{t-t_{\mix}}(2\|\bm{\theta^\star}\|_2 + 1)+ 2\eta_{t - t_{\mix}}\sum_{i=t-t_{\mix}}^{t-2}\mathbb{E}\|\bm{\Delta}_i\|_2; \label{eq:E-delta-tmix-1}\\ 
&\mathbb{E}\|\bm{\theta}_{t-1} - \bm{\theta}_{t-t_{\mix}}\|_2^2 \leq 2t_{\mix}\eta_{t-t_{\mix}}^2\left[t_{\mix}(2\|\bm{\theta}^\star\|_2+1)^2 + 4 \sum_{i=t-t_{\mix}}^{t-2} \mathbb{E}\|\bm{\Delta}_i\|_2^2\right]; \label{eq:E-delta-tmix-2}\\ 
&\mathbb{E}[\|\bm{\Delta}_{t-t_{\mix}}\|_2 \|\bm{\theta}_{t-1} - \bm{\theta}_{t-t_{\mix}}\|_2] \nonumber \\ 
&\leq t_{\mix}\eta_{t-t_{\mix}}(2\|\bm{\theta}^\star\|_2 +1)\mathbb{E}\|\bm{\Delta}_{t- t_{\mix}}\|_2+ \eta_{t - t_{\mix}}\sum_{i=t-t_{\mix}}^{t-2}\mathbb{E}\|\bm{\Delta}_{i}\|_2^2 + t_{\mix} \eta_{t - t_{\mix}} \mathbb{E}\|\bm{\Delta}_{t- t_{\mix}}\|_2^2.\label{eq:E-delta-tmix-3}
\end{align}
\end{subequations}
\end{lemma}

Lemma \ref{lemma:E-delta-tmix} directly leads to  
\begin{align}\label{eq:markov-L2-I3-bound}
I_3 &\geq -2\eta_t\mathbb{E}\|\bm{\Delta}_{t-t_{\mix}}\|_2^2 - 4t_{\mix}^2 \eta_{t-t_{\mix}}^2(2\|\bm{\theta^\star}\|_2 + 1)^2 - 16 t_{\mix} \eta_{t-t_{\mix}}^2\sum_{i=t-t_{\mix}}^{t-2}\mathbb{E}\|\bm{\Delta}_i\|_2^2 \nonumber \\ 
&- 4t_{\mix}\eta_{t-t_{\mix}}(2\|\bm{\theta}^\star\|_2 +1)\mathbb{E}\|\bm{\Delta}_{t- t_{\mix}}\|_2 - 4\eta_{t - t_{\mix}}\sum_{i=t-t_{\mix}}^{t-2}\mathbb{E}\|\bm{\Delta}_i\|_2^2  - 4t_{\mix} \eta_{t - t_{\mix}} \mathbb{E}\|\bm{\Delta}_{t- t_{\mix}}\|_2^2 \nonumber \\ 
&= -(2\eta_t +4t_{\mix}\eta_{t-t_{\mix}})\mathbb{E}\|\bm{\Delta}_{t-t_{\mix}}\|_2^2 - (16 t_{\mix} \eta_{t-t_{\mix}}^2+ 4\eta_{t - t_{\mix}})\sum_{i=t-t_{\mix}}^{t-2}\mathbb{E}\|\bm{\Delta}_i\|_2^2 \nonumber \\ 
&-4t_{\mix}^2 \eta_{t-t_{\mix}}^2(2\|\bm{\theta^\star}\|_2 + 1)^2 -4t_{\mix}\eta_{t-t_{\mix}}(2\|\bm{\theta}^\star\|_2 +1)\mathbb{E}\|\bm{\Delta}_{t- t_{\mix}}\|_2.
\end{align}

\paragraph{Bounding $I_4$} Similarly, we decompose $I_4$ as
\begin{align*}
I_4 &=  \mathbb{E}[\bm{\Delta}_{t-t_{\mix}}^\top (\bm{A}_t \bm{\theta}^\star - \bm{b}_t)] + \mathbb{E}[(\bm{\Delta}_{t-1} - \bm{\Delta}_{t-t_{\mix}})^\top (\bm{A}_t \bm{\theta}^\star - \bm{b}_t)]\\ 
&= \mathbb{E}[\bm{\Delta}_{t-t_{\mix}}^\top (\bm{A}_t \bm{\theta}^\star - \bm{b}_t)] + \mathbb{E}[(\bm{\theta}_{t-1} - \bm{\theta}_{t-t_{\mix}})^\top (\bm{A}_t \bm{\theta}^\star - \bm{b}_t)].
\end{align*}
The first term can be bounded using  the $t_{\mix}$ separation: 
\begin{align*}
&|\mathbb{E}[\bm{\Delta}_{t-t_{\mix}}^\top (\bm{A}_t \bm{\theta}^\star - \bm{b}_t)]| \\ 
&= |\mathbb{E} [\mathbb{E}_{t-t_{\mix}}[\bm{\Delta}_{t-t_{\mix}}^\top (\bm{A}_t \bm{\theta}^\star - \bm{b}_t)]]|\\ 
&= |\mathbb{E} [\mathbb{E}_{s_{t-1} \sim P^{t_{\mix}-1}(\cdot \mid s_{t-t_{\mix}})}[\mathbb{E}_{s_t \sim P(\cdot \mid s_{t-1})}[\bm{\Delta}_{t-t_{\mix}}^\top(\bm{A}_t \bm{\theta}^\star - \bm{b}_t)]] -  \mathbb{E}_{s_{t-1}\sim \mu }[\mathbb{E}_{s_t \sim P(\cdot \mid s_{t-1}) }[\bm{\Delta}_{t-t_{\mix}}^\top (\bm{A}_t \bm{\theta}^\star - \bm{b}_t)]]]|\\
&\leq \mathbb{E} \sup_{s_{t-1}} |\mathbb{E}_{s_t \sim P(\cdot \mid s_{t-1})}\bm{\Delta}_{t-t_{\mix}}^\top(\bm{A}_t \bm{\theta}^\star - \bm{b}_t)| \cdot d_{\mathsf{TV}}(P^{t_{\mix}-1}(\cdot \mid s_{t-t_{\mix}}),\mu)\\
&\leq  \mathbb{E}\|\bm{\Delta}_{t-t_{\mix}}\|_2 (2\|\bm{\theta}^\star\|_2 + 1) \cdot \eta_t,
\end{align*}
while the second term can be bounded by stepsizes:
\begin{align*}
&\mathbb{E}[(\bm{\theta}_{t-1} - \bm{\theta}_{t-t_{\mix}})^\top (\bm{A}_t \bm{\theta}^\star - \bm{b}_t)]\\ 
&\geq -(2\|\bm{\theta}^\star\|_2 + 1)\mathbb{E}\|\bm{\theta}_{t-1} - \bm{\theta}_{t-t_{\mix}}\|_2 \\ 
&\geq -(2\|\bm{\theta}^\star\|_2 + 1) \left[t_{\mix} \eta_{t-t_{\mix}} (2\|\bm{\theta}^\star\|_2 + 1) + 2\eta_{t-t_{\mix}} \sum_{i=t-t_{\mix}}^{t-2} \mathbb{E}\|\bm{\Delta}_i\|_2\right]\\ 
&= -\eta_{t-t_{\mix}}(2\|\bm{\theta}^\star\|_2 + 1) \left[t_{\mix}(2\|\bm{\theta}^\star\|_2 + 1) + 2\sum_{i=t-t_{\mix}}^{t-2} \mathbb{E}\|\bm{\Delta}_i\|_2 \right].
\end{align*}
Therefore, $I_4$ can be bounded by
\begin{align}\label{eq:markov-L2-I4-bound}
I_4 &\geq -\eta_t \mathbb{E}\|\bm{\Delta}_{t-t_{\mix}}\|_2 (2\|\bm{\theta}^\star\|_2 + 1) \nonumber \\ 
&-\eta_{t-t_{\mix}}(2\|\bm{\theta}^\star\|_2 + 1) \left[t_{\mix}(2\|\bm{\theta}^\star\|_2 + 1) + 2\sum_{i=t-t_{\mix}}^{t-2} \mathbb{E}\|\bm{\Delta}_i\|_2 \right].
\end{align}
\paragraph{Combining terms} With $I_3$ and $I_4$ bounded, we now return to Equation \eqref{eq:markov-L2-iter}. $\mathbb{E}\|\bm{\Delta}_{t}\|_2^2$ can be upper bounded by 
\begin{align}\label{eq:markov-L2-iter-bound}
\mathbb{E}\|\bm{\Delta}_{t}\|_2^2 &\leq I_1 + I_2 - 2\eta_t(I_3 + I_4) \nonumber \\ 
&\leq [1-2\lambda_0(1-\gamma)\eta_t + 2\eta_t^2(1+\gamma)^2]\mathbb{E}\|\bm{\Delta}_{t-1}\|_2^2 + 2\eta_t^2(2\|\bm{\theta}\|_2+1)^2 \nonumber \\
&+2\eta_t (2\eta_t +4t_{\mix}\eta_{t-t_{\mix}})\mathbb{E}\|\bm{\Delta}_{t-t_{\mix}}\|_2^2 + 2\eta_t (16 t_{\mix} \eta_{t-t_{\mix}}^2+ 4\eta_{t - t_{\mix}})\sum_{i=t-t_{\mix}}^{t-2}\mathbb{E}\|\bm{\Delta}_i\|_2^2 \nonumber \\ 
&+8\eta_t t_{\mix}^2 \eta_{t-t_{\mix}}^2 (2\|\bm{\theta^\star}\|_2 + 1)^2 + 8\eta_t t_{\mix}\eta_{t-t_{\mix}}(2\|\bm{\theta}^\star\|_2 +1)\mathbb{E}\|\bm{\Delta}_{t- t_{\mix}}\|_2 \nonumber \\ 
&+ 2\eta_t^2 \mathbb{E}\|\bm{\Delta}_{t-t_{\mix}}\|_2 (2\|\bm{\theta}^\star\|_2 + 1) +2\eta_t \eta_{t-t_{\mix}}(2\|\bm{\theta}^\star\|_2 + 1)t_{\mix}(2\|\bm{\theta}^\star\|_2 + 1) \nonumber \\ 
&+4\eta_t \eta_{t-t_{\mix}}(2\|\bm{\theta}^\star\|_2 + 1)\sum_{i=t-t_{\mix}}^{t-2} \mathbb{E}\|\bm{\Delta}_i\|_2 .
\end{align}
\paragraph{Specifying the polynomially-decaying stepsizes} With polynomially-decaying stepsizes, when $t$ is sufficiently large, it can be guaranteed that $t > 2t_{\mix}$, and therefore $\eta_{t-t_{\mix}} \geq 2^{-\alpha} \eta_t$. Furthermore, for sufficiently large $t$, $\eta_t (1+\gamma)^2 < \lambda_0(1-\gamma)$. Therefore, by dividing $(2\|\bm{\theta}^\star\|_2+1)^2$ on both sides and combining terms, we can simplify Equation \eqref{eq:markov-L2-iter-bound} as
\begin{align}\label{eq:markov-L2-iter-simplify}
\frac{\mathbb{E}\|\bm{\Delta}_{t}\|_2^2}{(2\|\bm{\theta}^\star\|_2+1)^2} &\leq (1-\widetilde{C}_1t^{-\alpha})\frac{\mathbb{E}\|\bm{\Delta}_{t-1}\|_2^2}{(2\|\bm{\theta}^\star\|_2+1)^2}  + \widetilde{C}_2 t^{-2\alpha}\log^2 t + \widetilde{C}_3 t^{-2\alpha} \log t \frac{\mathbb{E}\|\bm{\Delta}_{t-t_{\mix}}\|_2^2}{{(2\|\bm{\theta}^\star\|_2+1)^2} } \nonumber \\ 
&+ \widetilde{C}_4 t^{-2\alpha} \log t \sum_{i=t-t_{\mix}}^{t-2}\frac{\mathbb{E}\|\bm{\Delta}_i\|_2^2}{{(2\|\bm{\theta}^\star\|_2+1)^2} } + \widetilde{C}_5 t^{-2\alpha} \sum_{i=t-t_{\mix}}^{t-2}\frac{\mathbb{E}\|\bm{\Delta}_i\|_2}{2\|\bm{\theta}^\star\|_2+1} ,
\end{align}
where $\widetilde{C}_1$ through $\widetilde{C}_5$ are constants depending on $\alpha,\eta_0,m$ and $\rho$. Notice that the $\log t$ terms occur due to $t_{\mix} = O(\log t)$; see \eqref{eq:tmix-bound-L2}. We will use an induction argument based on the relation \eqref{eq:markov-L2-iter-simplify}. For simplicity, let
\begin{align*}
X_t = \frac{\|\bm{\Delta}_t\|_2}{2\|\bm{\theta}^\star\|_2+1};
\end{align*}
now suppose that
\begin{align}\label{eq:markov-L2-induction-assumption}
\mathbb{E}[X_t^2] \leq \widetilde{C} \cdot \frac{\log^2 t}{t^{\alpha}}, \quad \forall 1 < t \leq k,
\end{align}
for some $\widetilde{C}$. 

Our goal is to demonstrate, inductively, that
\begin{align}\label{eq:markov-L2-induction-goal}
\mathbb{E}[X_{k+1}^2] \leq \widetilde{C} \cdot \frac{\log^2 (k+1)}{(k+1)^{\alpha}}.
\end{align}
Towards this end, the iterative relation \eqref{eq:markov-L2-iter-simplify} implies that
\begin{align}\label{eq:markov-L2-induction-1}
\mathbb{E}[X_{k+1}^2] &\leq \left(1-\widetilde{C}_1(k+1)^{-\alpha}\right)\mathbb{E}[X_k^2] + \widetilde{C}_2 (k+1)^{-2\alpha}\log^2(k+1) \nonumber \\ 
&+ \widetilde{C}_3(k+1)^{-2\alpha} \log (k+1) \mathbb{E}[X_{k+1-t_{\mix}}^2] \nonumber \\ 
&+ \widetilde{C}_4(k+1)^{-2\alpha} \log(k+1) \sum_{i=k+1-t_{\mix}}^{k-1} \mathbb{E}[X_i^2] \nonumber \\ 
&+ \widetilde{C}_5(k+1)^{-2\alpha}\sum_{i=k+1-t_{\mix}}^{k-1} \mathbb{E}[X_i].
\end{align}
Here, the induction assumption guarantees that, as long as $k > 2t_{\mix}$, 
\begin{align*}
&\mathbb{E}[X_k^2] \leq \widetilde{C} k^{-\alpha} \log^2 k, \\ 
&\mathbb{E}[X_{k+1-t_{\mix}}^2] \leq \widetilde{C} (k+1-t_{\mix})^{-\alpha} \log^2(k+1-t_{\mix}) < 2^{-\alpha} \widetilde{C} (k+1)^{-\alpha} \log^2(k+1), 
\end{align*}
and that 
\begin{align*}
\sum_{i=k+1-t_{\mix}}^{k-1} \mathbb{E}[X_i] &\leq t_{\mix} \cdot \widetilde{C} (k+1-t_{\mix})^{-\frac{\alpha}{2}} \log(k+1-t_{\mix}) \\ 
&\lesssim \widetilde{C} \cdot (k+1)^{-\frac{\alpha}{2}} \log^2(k+1),\\ 
\sum_{i=k+1-t_{\mix}}^{k-1} \mathbb{E}[X_i^2] &\leq t_{\mix} \cdot \widetilde{C} (k+1-t_{\mix})^{-\alpha} \log^2(k+1-t_{\mix}) \\ 
&\lesssim \widetilde{C} \cdot (k+1)^{-\alpha} \log^3(k+1).
\end{align*}
Plugging these inequalities into the iteration relation \eqref{eq:markov-L2-induction-1}, we obtain that for sufficiently large $k$,
\begin{align*}
\mathbb{E}[X_{k+1}^2] &\leq \widetilde{C} \cdot \left[k^{-\alpha} \log^2 k - \widetilde{C}_1(k+1)^{-\alpha} k^{-\alpha} \log^2 k + \widetilde{C}_3 (k+1)^{-\frac{5}{2}\alpha}\log^2(k+1)\right]\\ 
& + \widetilde{C}_2(k+1)^{-2\alpha} \log^2(k+1).
\end{align*}
Here, $\widetilde{C}_1, \widetilde{C}_2$ and $\widetilde{C}_3$ are again constants independent of $t$, with there exact values can change from \eqref{eq:markov-L2-induction-1}. Therefore, it suffices to prove that
\begin{align}\label{eq:markov-L2-induction-2}
&\widetilde{C} \cdot \left[k^{-\alpha} \log^2 k - \widetilde{C}_1(k+1)^{-\alpha} k^{-\alpha} \log^2 k + \widetilde{C}_3 (k+1)^{-\frac{5}{2}\alpha}\log^2(k+1)\right]\nonumber \\ 
& + \widetilde{C}_2(k+1)^{-2\alpha} \log^2(k+1) \leq \widetilde{C} (k+1)^{-\alpha} \log^2(k+1).
\end{align}
Notice that when $x$ is sufficiently large, the function $f(x) = x^{-\alpha}\log^2(x)$ is monotonically decreasing; therefore, for sufficiently large $k$, it can be guaranteed that $k^{-\alpha} \log^2 k > (k+1)^{-\alpha} \log (k+1)$. Therefore, the left-hand-side of \eqref{eq:markov-L2-induction-2} is upper bounded by
\begin{align*}
&\widetilde{C} \cdot \left[k^{-\alpha} \log^2 k - \widetilde{C}_1(k+1)^{-\alpha} k^{-\alpha} \log^2 k + \widetilde{C}_3 (k+1)^{-\frac{5}{2}\alpha}\log^2(k+1)\right]\\ 
 &+\widetilde{C}_2(k+1)^{-2\alpha} \log^2(k+1) \\ 
&\leq \widetilde{C} \cdot \left[k^{-\alpha} \log^2 (k+1) - \widetilde{C}_1(k+1)^{-\alpha} (k+1)^{-\alpha} \log^2 (k+1) + \widetilde{C}_3 (k+1)^{-\frac{5}{2}\alpha}\log^2(k+1)\right]\\ 
 &+\widetilde{C}_2(k+1)^{-2\alpha} \log^2(k+1) \\ 
 &= \widetilde{C} \log^2(k+1) \cdot \left[k^{-\alpha} +\left(\frac{\widetilde{C}_2}{\widetilde{C}}- \widetilde{C}_1\right)(k+1)^{-2\alpha} + \widetilde{C}_3 (k+1)^{-\frac{5}{2}\alpha} \right].
\end{align*}
Hence, in order to prove \eqref{eq:markov-L2-induction-2}, it suffices to show that
\begin{align*}
k^{-\alpha} +\left(\frac{\widetilde{C}_2}{\widetilde{C}}- \widetilde{C}_1\right)(k+1)^{-2\alpha} + \widetilde{C}_3 (k+1)^{-\frac{5}{2}\alpha} \leq (k+1)^{-\alpha},
\end{align*}
which is equivalent to
\begin{align*}
(k+1)^{2\alpha}\left[k^{-\alpha} - (k+1)^{-\alpha}\right] + \widetilde{C}_3 (k+1)^{-\frac{\alpha}{2}} \leq \widetilde{C}_1 - \frac{\widetilde{C}_2}{\widetilde{C}}.
\end{align*}
Here, we further notice that the function $f(x) = x^{-\alpha}$ is monotonically decreasing and convex, so $k^{-\alpha}-(k+1)^{-\alpha} = f(k) -f(k+1) \leq -f'(k+1) = \alpha(k+1)^{-\alpha-1}$. Hence, the proof boils down to showing
\begin{align*}
\widetilde{C}_1 - \frac{\widetilde{C}_2}{\widetilde{C}} &\geq (k+1)^{2\alpha}\cdot \alpha(k+1)^{-\alpha-1} + \widetilde{C}_3 (k+1)^{-\frac{\alpha}{2}} \\ 
&= \alpha(k+1)^{\alpha-1} + \widetilde{C}_3 (k+1)^{-\frac{\alpha}{2}}
\end{align*}
for an appropriate $\widetilde{C}$ that is independent of $t$ and satisfies the induction assumption \eqref{eq:markov-L2-induction-assumption}. 
Towards this end, we define a function $f(\widetilde{C},k)$ as
\begin{align*}
f(\widetilde{C}, k):= \widetilde{C}_1 - \frac{\widetilde{C}_2}{\widetilde{C}} - \alpha(k+1)^{\alpha-1}- \widetilde{C}_3 (k+1)^{-\frac{\alpha}{2}}
\end{align*}
It is easy to verify that for any $\widetilde{C}$,
\begin{align*}
\lim_{k \to \infty} f(\widetilde{C}, k) = \widetilde{C}_1 - \frac{\widetilde{C}_2}{\widetilde{C}}.
\end{align*}
Therefore, we can take 
\begin{align*}
&k^\star = \min\left\{k:f\left(\max_{1 \leq t \leq k} \frac{t^\alpha}{\log^2 t} \mathbb{E}[X_t^2], k\right) \geq 0 \right\}, \quad \text{and} \\
&\widetilde{C} = \max_{1 \leq t \leq k^\star} \frac{t^\alpha}{\log^2 t} \mathbb{E}[X_t^2].
\end{align*}
On one hand, if $k^\star$ does not exist, then from our analysis we can conclude that 
\begin{align*}
\mathbb{E}[X_t^2] \leq \frac{\widetilde{C}_2}{\widetilde{C}_1} \frac{\log^2 t}{t^{\alpha}}
\end{align*}
for all $t \geq 1$; on the other hand, if $k^\star$ does exist, then an induction argument guarantees that
\begin{align*}
\mathbb{E}[X_t^2] \leq \widetilde{C} \frac{\log^2 t}{t^{\alpha}}
\end{align*}
for all $t \geq 1$. In both cases, \eqref{eq:markov-L2-induction-goal} holds true. 
\paragraph{Specification of the coefficient} We next try to specify the coefficient corresponding to the leading term of the upper bound. In the previous paragraph, we have essentially proved that, there exists a $t^\star \in \mathbb{N}$ depending on $\alpha,\eta_0,\lambda_0,m$ and $\rho$ such that
\begin{align*}
\mathbb{E}[X_t^2] \leq 1, \quad \text{for all }\quad t \geq t^\star.
\end{align*}
Hence, when $t > t^\star$, a closer examination of \eqref{eq:markov-L2-iter-bound} yields
\begin{align*}
\mathbb{E}[X_t^2] \leq (1-\lambda_0(1-\gamma)\eta_t) \mathbb{E}[X_{t-1}^2] + C  \frac{\eta_0^2}{(1-\rho)^2} t^{-2\alpha}\log^2 t
\end{align*}
for a \emph{universal constant} $C$. Hence by iteration, it can be guaranteed that when $t > t^\star$,
\begin{align*}
\mathbb{E}[X_t^2] \leq \underset{I_1}{\underbrace{\prod_{i=t^\star+1}^t (1-\beta i^{-\alpha})X_{t^\star}}} + C  \underset{I_2}{\underbrace{\frac{\eta_0^2}{(1-\rho)^2}\sum_{i=t^\star}^t (i^{-2\alpha} \log^2 i)\prod_{k=i+1}^t (1-\beta k^{-\alpha})}}.
\end{align*}
Here, it is easy to verify that $I_1$ converges exponentially with respect to $t$, and that $I_2$ is upper bounded by
\begin{align*}
I_2 &\leq \frac{\eta_0^2}{(1-\rho)^2} \log^2 t \sum_{i=t^\star}^t i^{-2\alpha} \prod_{k=i+1}^t (1-\beta k^{-\alpha}) \\ 
&\leq \frac{\eta_0^2}{(1-\rho)^2} \log^2 t \sum_{i=1}^t i^{-2\alpha} \prod_{k=i+1}^t (1-\beta k^{-\alpha}) \\ 
&\overset{(i)}{=} \frac{\eta_0^2}{(1-\rho)^2} \log^2 t \left(\frac{1}{\beta} t^{-\alpha} + O(t^{-1})\right)\\ 
&= \frac{\eta_0}{(1-\rho)^2 \lambda_0(1-\gamma)} t^{-\alpha}\log^2 t + o(t^{-\alpha}\log^2 t).
\end{align*}
The theorem follows immediately.

\end{proof}

\subsection{High-probability convergence guarantee for the original TD estimation error}\label{app:proof-TD-original}
Similar to the case with $i.i.d.$ samples, we firstly state the following theorem for the high-probability convergence rate for the original TD estimation error $\bm{\Delta}_t$ with Markov samples.
\begin{theorem}
\label{thm:markov-deltat-convergence}
Consider TD with Polyak-Ruppert averaging~\eqref{eq:TD-update-all} with Markov samples and decaying stepsizes $\eta_t = \eta_0 t^{-\alpha}$ for $\alpha \in (\frac{1}{2},1)$. Suppose that the Markov transition kernel has a unique stationary distribution $\mu$, a strictly positive spectral gap $1-\lambda > 0$, and Assumption \ref{as:mixing} holds true. Then for any $\delta \in (0,1)$, there exists $\eta_0 > 0$, such that with probability at least $1-{\delta}$,
\begin{align*}
\|\bm{\Delta}_t\|_2 &\leq \frac{13}{2}\|\bm{\theta}^\star\|_2 + \frac{5}{4} \quad \text{and}\\
\|\bm{\Delta}_t\|_2 &\lesssim \eta_0 \sqrt{\frac{2\tmix(t^{-\frac{\alpha}{2}})}{(2\alpha-1)}\log \frac{9T\tmix(t^{-\frac{\alpha}{2}})}{\delta}}(2\|\bm{\theta}^\star\|_2  + 1)\left(\frac{(1-\gamma)\lambda_0 \eta_0}{4\alpha}\right)^{-\frac{\alpha}{2(1-\alpha)}}t^{-\frac{\alpha}{2}}
\end{align*}
hold simultaneously for all $t \in [T]$.
\end{theorem}

\begin{proof}
Recalling the TD update rule \eqref{eq:TD-update-all}, we represent $\bm{\Delta}_t$ as
\begin{align*}
\bm{\Delta}_t = \bm{\theta}_t - \bm{\theta}^\star
&= (\bm{\theta}_{t-1}-\eta_t(\bm{A}_{t}\bm{\theta}_{t-1}-\bm{b}_{t})) - \bm{\theta}^\star\\
&= \bm{\Delta}_{t-1} - \eta_t (\bm{A}\bm{\theta}_{t-1}-\bm{b} + \bm{\zeta}_{t})\\
&= (\bm{I}-\eta_t \bm{A})\bm{\Delta}_{t-1} -\eta_t \bm{\zeta}_{t},
\end{align*}
where $\bm{\zeta}_t$ is defined as
\begin{align}\label{eq:defn-zetat}
\bm{\zeta}_t := (\bm{A}_t -\bm{A})\bm{\theta}_{t-1} - (\bm{b}_t-\bm{b}).
\end{align}
Therefore by induction, $\bm{\Delta}_t$ can be expressed as a weighted sum of $\{\bm{\zeta}_i\}_{0 \leq i < t}$, namely
\begin{align}
& \bm{\Delta}_t = \prod_{k=1}^{t} (\bm{I}-\eta_k \bm{A}) \bm{\Delta}_0 -\sum_{i=0}^{t-1} \bm{R}_i^t \bm{\zeta}_i,  \label{eq:delta-t-markov} \\ 
&  \text{where} \quad \bm{R}_i^t = \eta_i \prod_{k=i+1}^{t-1} (\bm{I}-\eta_k \bm{A}).
\end{align}

The difficulty in bounding the second term of \eqref{eq:delta-t-markov} lies in the fact that with Markov samples, $\{\bm{\zeta_i}\}_{i > 0}$ is no longer a martingale difference process. Therefore, we further decompose $\bm{\zeta}_i$ into three parts, namely
\begin{align}\label{eq:markov-zetai-decompose}
\bm{\zeta}_i = \mathbb{E}_{i_{\mix}}[\bm{\zeta}_{i,\mix}] + (\bm{\zeta}_{i,\mix} - \mathbb{E}_{i_{\mix}}[\bm{\zeta}_{i,\mix}]) + (\bm{\zeta}_i - \bm{\zeta}_{i,\mix}).
\end{align}
In order to simplify notation, throughout this proof we denote
\begin{align*}
t_{\mix} = t_{\mix}(\varepsilon) + 1,
\end{align*}
where $\varepsilon \in (0,1)$ is to be specified later. 
Furthermore, for every $t > 0$, we denote
\begin{align} 
&i_{\mix} = \max\{0,i - t_{\mix}(\varepsilon)\}, \quad \text{and} \label{eq:defn-t-imix}\\ 
&\bm{\zeta}_{i,\mix} = (\bm{A}_i - \bm{A})\bm{\theta}_{i_{\mix}} - (\bm{b}_i - \bm{b}),\label{eq:defn-zeta-imix}
\end{align}

The intuition behind the construction of $\bm{\zeta}_{i,\mix}$ is to guarantee that the samples $(\bm{A}_i,\bm{b}_i)$ and the iterated estimator $\bm{\theta}_{i_{\mix}}$ are separated in the Markov chain by at least $t_{\mix}$ samples, so that their distributions are close to independent. Recall from \eqref{eq:tmix.bound} that the mixing property of the Markov chain featured by Assumption \ref{as:mixing} guarantees
\begin{align}
\tmix \leq \frac{ \log (m/\varepsilon)}{\log (1/\rho)} + 1.
\end{align}

Furthermore, the difference bewteen $\bm{\zeta}_i $ and $\bm{\zeta}_{i,\mix}$ is 
\begin{align}
\bm{\zeta}_i - \bm{\zeta}_{i,\mix} = (\bm{A}_i - \bm{A}) (\bm{\theta}_{i-1} - \bm{\theta}_{i,\mix}).
\end{align}
Therefore, with the decomposition \eqref{eq:markov-zetai-decompose}, $\bm{\Delta}_t$ can be characterized as
\begin{align}\label{eq:markov-deltat-decompose}
\bm{\Delta}_t &= \underset{I_1}{\underbrace{\prod_{k=1}^t (\bm{I}-\eta_k \bm{A}) \bm{\Delta}_0}} - \underset{I_2}{\underbrace{\sum_{i=1}^t \bm{R}_i^t \mathbb{E}_{i_{\mix}}[\bm{\zeta}_{i,\mix}]}} - \underset{I_3}{\underbrace{\sum_{i=1}^t \bm{R}_i^t (\bm{A}_i - \bm{A}) (\bm{\theta}_{i-1} - \bm{\theta}_{i,\mix})}} \nonumber \\ 
&- \underset{I_4}{\underbrace{\sum_{i=1}^t \bm{R}_i^t (\bm{\zeta}_{i,\mix} - \mathbb{E}_{i_{\mix}}[\bm{\zeta}_{i,\mix}])}}.
\end{align}

In what follows, we denote
\begin{align*}
& \beta=\frac{1-\gamma}{2}\lambda_0\eta_0,\\ 
& R = \frac{13}{2}\|\bm{\theta}^\star\|_2 + \frac{5}{4}, \quad \text{and}\\
& \mathcal{H}_t = \left\{\max_{1 \leq i \leq t}\|\bm{\Delta}_i\|_2 \leq R\right\}.
\end{align*}
Furthermore, for any given $\delta$, let\footnote{Notice that the existence of $t^\star(\delta)$ is guaranteed by a similar reasoning to that in the proof of Theorem B.1 in \cite{wu2024statistical}.}
\begin{align}\label{eq:defn-tstar-markov}
t^\star = t^\star(\delta):= \inf\Bigg\{t \in \mathbb{N}^+: &\int_t^{\infty} \exp \left(-\frac{2\alpha-1}{2^{11}\eta_0^2} \frac{\log(1/\rho)}{\log(8m\eta_0/\beta)}\left(\frac{\beta}{2\alpha}\right)^{\frac{\alpha}{1-\alpha}}x^{-\alpha}\right)\mathrm{d}x \nonumber \\ 
&\leq \frac{\delta}{27} \cdot \frac{\log(8m\eta_0/\beta)}{\log(1/\rho)}\Bigg\},
\end{align}
and assume that 
\begin{align}\label{eq:deltat-condition-markov}
\eta_0 \sqrt{\frac{2}{2\alpha-1} \frac{\log(8m\eta_0/\beta)}{\log(1/\rho)}} \max\left\{\frac{1}{4\sqrt{1-\alpha}},\log \frac{9\log(8m\eta_0/\beta)t^\star}{\log(1/\rho)\delta}\right\}\leq \frac{1}{32}.
\end{align}
We break down the proof of the theorem into a sequence of steps:
\begin{enumerate}
\item We obtain convergence rates for the four terms on the right-hand-side of \eqref{eq:markov-deltat-decompose};
\item We lower bound  the probability of $\mathcal{H}_{t^\star}$ by 
 $1-\frac{\delta}{3}$;
\item We lower bound the probability of $\mathcal{H}_{\infty}$ by 
  $1-\frac{2\delta}{3}$;
\item Using the results from the first steps, we arrive at a final bound on $\|\bm{\Delta}_t\|_2$.
\end{enumerate}

\paragraph{Step 1: Basic convergence properties of the four terms on the right-hand-side of \eqref{eq:markov-deltat-decompose}} As is shown in the proof of Theorem B.1 in \cite{wu2024statistical}, the norm of $I_1$ is bounded by
\begin{align}\label{eq:markov-deltat-I1-bound}
\left\|\prod_{k=0}^{t-1} (\bm{I}-\eta_k \bm{A}) \bm{\Delta}_0\right\| \leq \left(\frac{1-\gamma}{2}\lambda_0 \eta_0\right)^{-\frac{\alpha}{1-\alpha}} t^{-\alpha} \|\bm{\Delta}_0\|_2.
\end{align}
For the term $I_2$, we observe that for $i \leq t_{\mix}$, since $i_{\mix} = 0$, 
\begin{align*}
\mathbb{E}_{i_{\mix}}\left[\bm{\zeta}_{i,\mix}\right]= \mathbb{E}_0 [(\bm{A}_i-\bm{A})\bm{\theta}_0 - (\bm{b}_i - \bm{b})] = \bm{0};
\end{align*}
otherwise when $i > t_{\mix}$, since $i_{\mix} = i-t_{\mix}$, 
\begin{align}\label{eq:E-zeta-imix}
\left\|\mathbb{E}_{i_\mix}[\bm{\zeta}_{i,\mix}]\right\|_2&\leq \left\| (\mathbb{E}_{i_{\mix}}[\bm{A}_i] - \mathbb{E}[\bm{A}_i]) \bm{\theta}_{i_{\mix}}\right\|_2  + \left\|\mathbb{E}_{i_{\mix}}[\bm{b}_i] - \mathbb{E}[\bm{b}_i]\right\|_2\nonumber \\
&\leq d_{\text{TV}}(P^{t_{\mix}}(\cdot|s_{i_{\mix}}),\mu) \left(\sup_{s_{i-1},s_i \in \mathcal{S}} \| \bm{A}_i \bm{\theta}_{i_{\mix}} \|_2 + \sup_{s_i \in \mathcal{S}} \|\bm{b}_i\|_2\right)\nonumber \\
&\leq (2\max_{1 \leq i < t}\|\bm{\Delta}_i\|_2+ 2\|\bm{\theta}^\star\|_2 + 1) d_{\text{TV}}(P^{t_{\mix}}(\cdot|s_{i_{\mix}}),\mu) \nonumber \\ 
&\leq \varepsilon (2\max_{1 \leq i < t}\|\bm{\Delta}_i\|_2+ 2\|\bm{\theta}^\star\|_2 + 1).
\end{align}
Meanwhile, we observe that the sum of $\|\bm{R}_i^t\|$ can be bounded by 
\begin{align*}
\sum_{i=1}^t \|\bm{R}_i^t\| &\leq \sum_{i=1}^t \eta_i \prod_{k=i+1}^t \|\bm{I}-\eta_k \bm{A}\|\\ 
&\leq \sum_{i=1}^t \eta_0 i^{-\alpha} \prod_{k=i+1}^t \left(1-\frac{1-\gamma}{2}\lambda_0 \eta_0 k^{-\alpha}\right)\\ 
&= \sum_{i=1}^t \eta_0 i^{-\alpha} \prod_{k=i+1}^t (1-\beta k^{-\alpha})\\ 
&= \frac{\eta_0}{\beta} \sum_{i=1}^t \left(\prod_{k=i+1}^t - \prod_{k=i}^t\right)(1-\beta k^{-\alpha}) \\ 
&= \frac{\eta_0}{\beta}\left(1-\prod_{k=1}^t (1-\beta k^{-\alpha})\right) < \frac{\eta_0}{\beta},
\end{align*}
where the second inequality follows from \eqref{eq:lemma-A-5} in Lemma \ref{lemma:A}.
Therefore, the norm of $I_2$ is bounded by 
\begin{align}\label{eq:markov-deltat-I2-bound}
\left\| \sum_{i=1}^t \bm{R}_i^t \mathbb{E}_{i_{\mix}}[\bm{\zeta}_{i,\mix}] \right\|&\leq \sum_{i=1}^t \left\|\bm{R}_i^t\right\| \left\|\mathbb{E}_{i_\mix}[\bm{\zeta}_{i,\mix}]\right\|_2 \nonumber \\ 
&\leq \frac{\varepsilon \eta_0}{\beta} (2\max_{1 \leq i \leq t}\|\bm{\Delta}_i\|_2 + 2\|\bm{\theta}^\star\|_2  + 1).
\end{align}
For the term $I_3$, we firstly bound the difference betweeen $\bm{\theta}_i$ and $\bm{\theta}_{i,\mix}$ by $R$ and the initial stepsize $\eta_0$. Notice that given the induction assumption, 
\begin{align*}
\left\|\bm{\theta}_{i-1} - \bm{\theta}_{i,\mix}\right\|_2 = \left\|\sum_{j=i_{\mix}+1}^{i-1} (\bm{\theta}_{j} - \bm{\theta}_{j-1}) \right\|_2 &= \left\| \sum_{j=i_{\mix}+1}^{i-1} \eta_j (\bm{A}_j \bm{\theta}_{j-1} - \bm{b}_j)\right\|_2 \\ 
&\leq \sum_{j=i_{\mix}}^{i-1} \eta_j \|\bm{A}_j \bm{\theta}_{j-1} - \bm{b}_j\|_2 \\ 
&\leq \frac{t_{\mix}}{1-\alpha} \eta_i (2 \max_{1 \leq i < t} \|\bm{\Delta}_i\|_2 + 2\|\bm{\theta}^\star\|_2  + 1).
\end{align*}
Hence, the norm of $I_{3}$ can be bounded by
\begin{align}\label{eq:markov-deltat-I3-bound}
&\left\|\sum_{i=1}^t \bm{R}_i^t (\bm{A}_i - \bm{A}) (\bm{\theta}_i - \bm{\theta}_{i,\mix})\right\|_2  \nonumber \\
& \leq \sum_{i=1}^t \|\bm{R}_i^t\| \|\bm{A}_i - \bm{A}\| \left\|\bm{\theta}_{i-1} - \bm{\theta}_{i,\mix}\right\|_2 \nonumber \\ 
&\leq \sum_{i=1}^t \left(\eta_i \prod_{k=i+1}^t \left(1-\beta k^{-\alpha} \right)\right) \cdot 4 \cdot \frac{t_{\mix}}{1-\alpha} \eta_i (2 \max_{1 \leq i < t} \|\bm{\Delta}_i\|_2 + 2\|\bm{\theta}^\star\|_2  + 1)\nonumber \\ 
&\lesssim \sum_{i=1}^t \eta_i^2 \prod_{k=i+1}^t \left(1-\beta k^{-\alpha} \right) \cdot  \frac{t_{\mix}}{1-\alpha} (2\max_{1 \leq i < t} \|\bm{\Delta}_i\|_2 + 2\|\bm{\theta}^\star\|_2  + 1)\nonumber \\ 
&\leq \frac{\eta_0^2t_{\mix}}{(1-\alpha)} (2\max_{1 \leq i < t} \|\bm{\Delta}_i\|_2 + 2\|\bm{\theta}^\star\|_2  + 1) \cdot \sum_{i=1}^t i^{-2\alpha} \prod_{k=i+1}^t (1-\beta k^{-\alpha})\nonumber \\
&\leq \frac{\eta_0^2t_{\mix}}{1-\alpha} (2\max_{1 \leq i < t} \|\bm{\Delta}_i\|_2 + 2\|\bm{\theta}^\star\|_2  + 1) \cdot \frac{1}{2\alpha-1}\left(\frac{\beta}{2\alpha}\right)^{\frac{\alpha}{1-\alpha}}t^{-\alpha}.
\end{align}
Notice that the last inequality follows by taking $\nu = 2\alpha$ in \eqref{eq:lemma-R-2} in Lemma \ref{lemma:R}. 

For the term $I_4$, we again invoke the vector Azuma's inequality (Theorem \ref{thm:vector-Azuma}). 
For simplicity, we define
\begin{align}
	\bm{X}_i := \bm{\zeta}_{i,\text{mix}}-\mathbb{E}_{i_{\text{mix}}}[\bm{\zeta}_{i,\text{mix}}].
\end{align}
By definition, we can see that for every integer $r \in [\tmix]$, the sequence $\{ \bm{X}_{r + i t_{\text{mix}}} \}_{i=0,1,\ldots}$
form a martingale difference process. Throughout this part, we assume, without loss of generality that $t = t' \cdot \tmix$ for a positive integer $t'$. Hence, we can first consider the norm of the summation
\begin{align}
\sum_{i'=0}^{t'} \bm{R}_{r+i'\tmix}^t \bm{X}_{r+i'\tmix},
\end{align}
which we will bound using vector Azuma's inequality. Towards that goal, we define 
\begin{align*}
W_{\max}^r :=  \sum_{i'=0}^{t'} \sup\left\|\bm{R}_{r+i'\tmix}^t \bm{X}_{r+i'\tmix}\right\|_2^2.
\end{align*}
By definition, this term is bounded by
\begin{align*}
W_{\max}^r &\leq \sum_{i'=0}^{t'}\|\bm{R}_{r+i'\tmix}^t\|^2 \sup \|\bm{X}_{r+i'\tmix}\|_2^2 \leq  \sum_{i'=0}^{t'}\|\bm{R}_{r+i'\tmix}^t\|^2 \cdot \max_{1 \leq i \leq t} \sup \|\bm{X}_i\|_2^2;
\end{align*}
by summing over $r = 1$ through $r = \tmix$, we obtain
\begin{align*}
\sum_{r=1}^{\tmix} W_{\max}^r &\leq \sum_{r=1}^{\tmix} \sum_{i'=0}^{t'}\|\bm{R}_{r+i'\tmix}^t\|^2 \cdot  \max_{1 \leq i \leq t} \sup \|\bm{X}_i\|_2^2 \\ 
&\leq  \left(\max_{1 \leq i \leq t} \sup \|\bm{X}_i\|_2^2\right)\cdot \sum_{i=1}^t \|\bm{R}_i^t\|^2 .
\end{align*}
Meanwhile, the triangle inequality directly implies
\begin{align*}
\max_{1 \leq i \leq t} \sup \|\bm{X}_i\|_2^2 &= \max_{1 \leq i \leq t} \sup \|\bm{\zeta}_{i,\text{mix}}-\mathbb{E}_{i_{\text{mix}}}[\bm{\zeta}_{i,\text{mix}}]\| \\ 
&\leq 4\max_{1 \leq i < t} \|\bm{\Delta}_i\|_2 + 2\|\bm{\theta}^\star\|_2  + 1.
\end{align*}
In combination, we have that
\begin{align*}
\sum_{r=1}^{\tmix} \sqrt{W_{\max}^r} &\leq \sqrt{\tmix} \sqrt{\sum_{r=1}^{\tmix} W_{\max}^r} \\ 
&\leq \sqrt{\tmix} \sup \|\bm{X}_i\|_2 \sqrt{\sum_{i=1}^t \|\bm{R}_i^t\|^2} \\ 
&\leq \sqrt{\tmix} (4\max_{1 \leq i < t} \|\bm{\Delta}_i\|_2 + 2\|\bm{\theta}^\star\|_2  + 1)\cdot \sqrt{\sum_{i=1}^t \|\bm{R}_i^t\|^2} \\ 
&\leq \sqrt{\tmix} (4\max_{1 \leq i < t} \|\bm{\Delta}_i\|_2 + 2\|\bm{\theta}^\star\|_2  + 1)\cdot \sqrt{\sum_{i=1}^t \eta_i^2 \prod_{k=i+1}^t (1-\frac{1-\gamma}{2}\lambda_0\eta_k)^2}\\
&\leq \sqrt{\tmix} (4\max_{1 \leq i < t} \|\bm{\Delta}_i\|_2 + 2\|\bm{\theta}^\star\|_2  + 1)\cdot \eta_0 \cdot \sqrt{\sum_{i=1}^t i^{-2\alpha} \prod_{k=i+1}^t (1-\beta k^{-\alpha})}\\
&\leq \sqrt{\tmix} (4\max_{1 \leq i < t} \|\bm{\Delta}_i\|_2 + 2\|\bm{\theta}^\star\|_2  + 1)\cdot \eta_0 \cdot \sqrt{\frac{1}{2\alpha-1}}\left(\frac{\beta}{2\alpha}\right)^{\frac{\alpha}{2(1-\alpha)}}t^{-\frac{\alpha}{2}},
\end{align*}
where we invoked Lemma \ref{lemma:R} in the last inequality, following the same logic as in the last line of \eqref{eq:markov-deltat-I3-bound}. 
Consequently, the vector Azuma's inequality (Theorem \ref{thm:vector-Azuma}) 
, combined with a union bound argument, yields the bound on the norm of $I_4$
\begin{align}\label{eq:markov-deltat-I4-bound}
&\left\|\sum_{i=1}^t \bm{R}_i^t (\bm{\zeta}_{i,\mix} - \mathbb{E}_{i_{\mix}}[\bm{\zeta}_{i,\mix}])\right\|_2\nonumber \leq 2\sqrt{2\log \frac{9t_{\mix}}{\delta_t}} \sum_{r=1}^{\tmix} \sqrt{W_{\max}^r} \nonumber \\ 
&\leq 2 \eta_0 \sqrt{\frac{2\tmix}{2\alpha-1}\log \frac{3\tmix}{\delta_t}}(4\max_{1 \leq i < t} \|\bm{\Delta}_i\|_2 + 2\|\bm{\theta}^\star\|_2  + 1)\cdot \left(\frac{\beta}{2\alpha}\right)^{\frac{\alpha}{2(1-\alpha)}}t^{-\frac{\alpha}{2}},
\end{align}
with probability at least $1-\delta_t/3$. 
\paragraph{Step 2: Bounding $\mathbb{P}(\mathcal{H}_{t^\star})$} By definition, $\mathbb{P}(\mathcal{H}_0) = 1$. We will show via induction that, for all other $1 \leq t \leq t^\star$
\begin{align*}
\mathbb{P}(\mathcal{H}_t) \geq 1-\frac{t}{3t^\star} \delta.
\end{align*}
By taking $\varepsilon = \frac{\beta}{8\eta_0}$ in \eqref{eq:markov-deltat-I2-bound}, we obtain that
\begin{align*}
\left\| \sum_{i=1}^t \bm{R}_i^t \mathbb{E}_{i_{\mix}}[\bm{\zeta}_{i,\mix}] \right\|_2 \leq \frac{1}{8}\left(2\max_{1 \leq i < t} \|\bm{\Delta}_i\|_2 + 2\|\bm{\theta}^\star\|_2  + 1\right).
\end{align*}
Next, putting together \eqref{eq:markov-deltat-I3-bound}, condition \eqref{eq:deltat-condition-markov} and the bound on $\tmix$ as specified by \eqref{eq:tmix-bound} guarantees that
\begin{align*}
\left\|\sum_{i=1}^t \bm{R}_i^t (\bm{A}_i - \bm{A}) (\bm{\theta}_i - \bm{\theta}_{i,\mix})\right\|_2 \leq \frac{1}{8}\left(2\max_{1 \leq i < t} \|\bm{\Delta}_i\|_2 + 2\|\bm{\theta}^\star\|_2  + 1\right).
\end{align*}
Similarly, setting $\delta_t = \frac{\delta}{3t^\star}$ in \eqref{eq:markov-deltat-I4-bound} and using the condition \eqref{eq:deltat-condition-markov}, we have that, with probability at least $1-\frac{\delta}{3t^\star}$, 
\begin{align*}
\left\|\sum_{i=1}^t \bm{R}_i^t (\bm{\zeta}_{i,\mix} - \mathbb{E}_{i_{\mix}}[\bm{\zeta}_{i,\mix}])\right\|_2 &\leq 2\tmix \eta_0 \sqrt{2\log \frac{9\tmix t^\star}{\delta}}(4\max_{1 \leq i < t} \|\bm{\Delta}_i\|_2 + 2\|\bm{\theta}^\star\|_2  + 1) \\ 
&\leq \frac{1}{16}(4\max_{1 \leq i < t} \|\bm{\Delta}_i\|_2 + 2\|\bm{\theta}^\star\|_2  + 1).
\end{align*}
Therefore, when $\mathcal{H}_{t-1}$ holds true, the norm of $\bm{\Delta}_t$ can be bounded using the triangle inequality as
\begin{align}\label{eq:markov-deltat-induction}
\|\bm{\Delta}_t\|_2 &\leq \|\bm{\theta}^\star\|_2 + \frac{1}{8}\left(2R + 2\|\bm{\theta}^\star\|_2  + 1\right) + \frac{1}{8}\left(2R + 2\|\bm{\theta}^\star\|_2  + 1\right) + \frac{1}{16}(4R + 2\|\bm{\theta}^\star\|_2  + 1) \nonumber \\ 
&= \frac{3}{4}R + \frac{13}{8}\|\bm{\theta}^\star\|_2 + \frac{5}{16} = R,
\end{align}
with probability of at least $1-\frac{\delta}{3t^\star}$. It then follows that
\begin{align*}
\mathbb{P}(\mathcal{H}_{t-1} \setminus \mathcal{H}_t) \leq \frac{\delta}{3t^\star},
\end{align*}
and thus that
\begin{align*}
\mathbb{P}(\mathcal{H}_{t^\star}) \geq 1-\frac{\delta}{3}.
\end{align*}
\paragraph{Step 3: Bounding $\mathbb{P}(\mathcal{H}_{\infty})$} For $t > t^\star$, we sharpen the induction argument by a more refined choice of $\delta_t$. In detail, let 
\begin{align*}
\delta_t = 3\tmix(\beta/8\eta_0)\exp\left\{-\frac{2\alpha-1}{2^{11} \eta_0^2} \left(\frac{1}{\tmix(\beta/8\eta_0)}\right)\left(\frac{\beta}{2\alpha}\right)^{\frac{\alpha}{1-\alpha}} t^{-\alpha}\right\}.
\end{align*}
Then the norm of $I_4$ is bounded by
\begin{align*}
&\left\|\sum_{i=1}^t \bm{R}_i^t (\bm{\zeta}_{i,\mix} - \mathbb{E}_{i_{\mix}}[\bm{\zeta}_{i,\mix}])\right\|_2\\ 
&\leq 2 \eta_0 \sqrt{\frac{2\tmix}{2\alpha-1}\log \frac{3\tmix}{\delta_t}}(4\max_{1 \leq i < t} \|\bm{\Delta}_i\|_2 + 2\|\bm{\theta}^\star\|_2  + 1)\cdot \left(\frac{\beta}{2\alpha}\right)^{\frac{\alpha}{2(1-\alpha)}}t^{-\frac{\alpha}{2}} \\ 
&\leq \frac{1}{16}(4\max_{1 \leq i < t} \|\bm{\Delta}_i\|_2 + 2\|\bm{\theta}^\star\|_2  + 1).
\end{align*}
with probability at least $1-\delta_t$. Hence, using induction, when $\mathcal{H}_{t-1}$ holds true,  the bound in \eqref{eq:markov-deltat-induction} implies that $\mathcal{H}_t$ also holds true with probability at least $1-\delta_t$. In other words,
\begin{align*}
\mathbb{P}(\mathcal{H}_{t-1}) - \mathbb{P}(\mathcal{H}_t) = \mathcal{P}(\mathcal{H}_{t-1} \setminus \mathcal{H}_t) \leq \delta_t.
\end{align*}
Consequently, the definition of $t^\star$ \eqref{eq:defn-tstar-markov} guarantees
\begin{align*}
\mathbb{P}(\mathcal{H}_{\infty}) &= \mathbb{P}(\mathcal{H}_{t^\star}) - \sum_{t=t^\star+1}^{\infty} \mathbb{P}(\mathcal{H}_{t-1}) - \mathbb{P}(\mathcal{H}_t) \\ 
&\geq \left(1-\frac{\delta}{3}\right) - \sum_{t=t^\star+1}^{\infty} \delta_t \\ 
&\geq \left(1-\frac{\delta}{3}\right) - \int_{t^\star}^{\infty} 3\tmix(\beta/8\eta_0)\exp\left\{-\frac{2\alpha-1}{2^{11} \eta_0^2} \left(\frac{1}{\tmix(\beta/8\eta_0)}\right)\left(\frac{\beta}{2\alpha}\right)^{\frac{\alpha}{1-\alpha}} x^{-\alpha}\right\}\mathrm{d}x \\ 
&\geq \left(1-\frac{\delta}{3}\right) - \frac{\delta}{3} = 1-\frac{2\delta}{3}.
\end{align*}
\paragraph{Step 4: Refining the bound on $\|\bm{\Delta}_t\|_2$} In order to bound the norm of $\bm{\Delta}_t$ by $O(t^{-\frac{\alpha}{2}})$ and thus conclude the proof, we take $\varepsilon = t^{-\frac{\alpha}{2}}$. Then,
\begin{align*}
\tmix(\varepsilon) \leq \frac{\log m + \frac{\alpha}{2}\log t}{\log(1/\rho)}.
\end{align*}
With this bound, \eqref{eq:markov-deltat-I2-bound}, \eqref{eq:markov-deltat-I3-bound} and \eqref{eq:markov-deltat-I4-bound} yield that, with probability at least $1-\frac{\delta}{3T}$,
\begin{align*}
&\left\| \sum_{i=1}^t \bm{R}_i^t \mathbb{E}_{i_{\mix}}[\bm{\zeta}_{i,\mix}] \right\|\leq \frac{\eta_0}{\beta}(2\max_{1 \leq i < t} \|\bm{\Delta}_i\|_2 + 2\|\bm{\theta}^\star\|_2  + 1)t^{-\frac{\alpha}{2}}, \\ 
&\left\|\sum_{i=1}^t \bm{R}_i^t (\bm{A}_i - \bm{A}) (\bm{\theta}_i - \bm{\theta}_{i,\mix})\right\|_2 \leq \frac{\eta_0^2}{(1-\alpha)(2\alpha-1)} \cdot \tmix(2\max_{1 \leq i < t} \|\bm{\Delta}_i\|_2 + 2\|\bm{\theta}^\star\|_2  + 1) \left(\frac{\beta}{2\alpha}\right)^{-\frac{\alpha}{1-\alpha}}t^{-\alpha}, \\ 
&\left\|\sum_{i=1}^t \bm{R}_i^t (\bm{\zeta}_{i,\mix} - \mathbb{E}_{i_{\mix}}[\bm{\zeta}_{i,\mix}])\right\|_2 \leq 2 \eta_0 \sqrt{\frac{2\tmix}{(2\alpha-1)}\log \frac{9T\tmix}{\delta}}(4\max_{1 \leq i < t} \|\bm{\Delta}_i\|_2 + 2\|\bm{\theta}^\star\|_2  + 1)\left(\frac{\beta}{2\alpha}\right)^{-\frac{\alpha}{2(1-\alpha)}}t^{-\frac{\alpha}{2}}.
\end{align*}
The final result follows from the triangle inequality and the union bound. 
\end{proof}


\subsection{Proof of Theorem \ref{thm:TD-whp}}\label{app:proof-markov-deltat-convergence}
Recall from Theorem \ref{thm:Lambda} that
\begin{align*}
\mathsf{Tr}(\tilde{\bm{\Lambda}}_T - \bm{\Lambda}^\star) = T^{\alpha-1}\mathsf{Tr}(\bm{X}(\tilde{\bm{\Lambda}}^\star)) + O(T^{2\alpha-2}),
\end{align*}
where $\bm{X}(\tilde{\bm{\Lambda}}^\star)$ is the solution to the Lyapunov equation
\begin{align*}
\eta_0(\bm{AX+XA}^\top) = \tilde{\bm{\Lambda}}^\star.
\end{align*}
By combining Lemma \ref{lemma:A}, Lemma \ref{lemma:Lyapunov} and Lemma \ref{lemma:Gamma}, we obtain
\begin{align*}
\mathsf{Tr}(\bm{X}(\tilde{\bm{\Lambda}}^\star)) &\leq \frac{\mathsf{Tr}(\tilde{\bm{\Lambda}}^\star)}{\eta_0\lambda_0(1-\gamma)} \leq \frac{\|\bm{A}^{-1}\|^2 \mathsf{Tr}(\tilde{\bm{\Gamma}})}{\eta_0\lambda_0(1-\gamma)} 
\leq \frac{\mathsf{Tr}(\tilde{\bm{\Gamma}})}{\eta_0\lambda_0^3(1-\gamma)^3} \\ 
&\lesssim \frac{m}{1-\rho} \cdot \frac{1}{\eta_0\lambda_0^5(1-\gamma)^5}.
\end{align*}
Hence, the difference between $\tilde{\bm{\Lambda}}_T$ and $\bm{\Lambda}^\star$ is given by
\begin{align*}
\mathsf{Tr}(\tilde{\bm{\Lambda}}_T - \bm{\Lambda}^\star)  \leq \widetilde{C}T^{\alpha-1},
\end{align*}
where $\widetilde{C}$ can be represented by $\lambda_0,\eta_0$ and $\gamma$.
Therefore, it suffices to show that with probability at least $1-\delta$, the averaged TD error can be bounded by
\begin{align*}
\|\bar{\bm{\Delta}}_T\|_2 &\lesssim 2\sqrt{\frac{2\mathsf{Tr}(\widetilde{\bm{\Lambda}}_T)}{T} \log \frac{6d}{\delta}} + o\left(T^{-\frac{1}{2}}\log^{\frac{3}{2}}\frac{dT}{\delta}\right).
\end{align*}
As a direct implication of \eqref{eq:delta-t-markov}, $\bar{\bm{\Delta}}_T$ can be decomposed as 
\begin{align*}
\bar{\bm{\Delta}}_T &= \frac{1}{T}\sum_{t=1}^T \bm{\Delta}_t \\ 
 &= \frac{1}{T}\sum_{t=1}^T \left(\prod_{k=1}^t (\bm{I}-\eta_k \bm{A})\Delta_0 - \sum_{i=1}^{t} \bm{R}_{i}^t (\bm{A}_i\bm{\theta}^\star -\bm{b}_i) - \sum_{i=1}^{t} \bm{R}_{i}^t (\bm{A}_i - \bm{A})\bm{\Delta}_{i-1}\right) \\ 
&= \frac{1}{T}\sum_{t=1}^T\prod_{k=1}^t (\bm{I}-\eta_k \bm{A})\Delta_0 - \frac{1}{T}\sum_{t=1}^T\sum_{i=1}^{t} \bm{R}_{i}^t (\bm{A}_i\bm{\theta}^\star -\bm{b}_i) - \frac{1}{T}\sum_{t=1}^T\sum_{i=1}^{t}\bm{R}_{i}^t (\bm{A}_i - \bm{A})\bm{\Delta}_{i-1} \\ 
&= \frac{1}{T}\sum_{t=1}^T\prod_{k=1}^t (\bm{I}-\eta_k \bm{A})\Delta_0 - \frac{1}{T} \sum_{i=1}^T \sum_{t=i}^T \bm{R}_{i}^t (\bm{A}_i\bm{\theta}^\star -\bm{b}_i) - \frac{1}{T} \sum_{i=1}^T \sum_{t=i}^T \bm{R}_{i}^t (\bm{A}_i - \bm{A})\bm{\Delta}_{i-1},
\end{align*}
where we have switched the order of summation in the last equation. The definition of $\bm{Q}_t$ \eqref{eq:defn-Qt} implies 
\begin{align*}
\bar{\bm{\Delta}}_T = \underset{I_1}{\underbrace{\frac{1}{T\eta_0} \bm{Q}_0 \bm{\Delta}_0}} - \underset{I_2}{\underbrace{\frac{1}{T} \sum_{i=1}^T \bm{Q}_i(\bm{A}_i \bm{\theta}^\star-\bm{b}_i)}} - \underset{I_3}{\underbrace{\frac{1}{T}\sum_{i=1}^T \bm{Q}_i (\bm{A}_i-\bm{A})\bm{\Delta}_{i-1}}},
\end{align*}
where $I_1$ can be bounded by
\begin{align*}
\left\|\frac{1}{T\eta_0} \bm{Q}_0 \bm{\Delta}_0\right\|_2 \leq \frac{1}{T\eta_0} \|\bm{Q}_0\| \|\bm{\Delta}_0\|\leq \frac{3}{T}\left(\frac{2}{\beta}\right)^{\frac{1}{1-\alpha}} \|\bm{\Delta}_0\|_2. 
\end{align*}

We now proceed to bounding $I_2$ and $I_3$respectively. 

Throughout the proof, we let
\begin{align*}
&\beta = \frac{1-\gamma}{2}\lambda_0 \eta_0,\\ 
& R = \frac{13}{2}\|\bm{\theta}^\star\|_2 + \frac{5}{4}, \\ 
&t_{\mix} = \tmix(T^{-\frac{\alpha}{2}}) \leq \frac{\log m + (\alpha/2)\log T}{\log(1/\rho)} \quad \text{and}\\ 
&R' = \eta_0 \sqrt{\frac{2\tmix}{2\alpha-1}\log \frac{27T\tmix}{\delta}} (2\|\bm{\theta}^\star\|_2+1)\left(\frac{\beta}{2\alpha}\right)^{-\frac{\alpha}{2(1-\alpha)}}, 
\end{align*}
and, for each $1 \leq t \leq T$,
\[
\mathcal{H}_t = \Big\{ \|\bm{\Delta}_j\|_2 \leq \min\{R'j^{-\frac{\alpha}{2}},R\}, \forall j \leq t \Big\} \quad \text{and} \quad \tilde{\bm{\Delta}}_t = \bm{\Delta}_t \mathds{1}(\mathcal{H}_t).
\]
Theorem \ref{thm:markov-deltat-convergence} shows that $\mathbb{P}(\mathcal{H}_T) \geq 1-\frac{\delta}{3}$.
\paragraph{Bounding $I_2$} In order to invoke the matrix Freedman's inequality on the term $I_2$, we firstly relate it to a martingale. Specifically, for every $i \in [T]$, we define $\bm{U}_i$ as
\begin{align}\label{eq:defn-Ui}
\bm{U}_i = \mathbb{E}_i \left[\sum_{j=i}^{\infty} (\bm{A}_j \bm{\theta}^\star - \bm{b}_j)\right].
\end{align}
It is then easy to verify that on one hand, the norm of $\bm{U}_i$ is uniformly bounded due to the exponential convergence of the Markov chain. Specifically, since for any positive integers $i<j$, it can be guaranteed that
\begin{align*}
&\left\|\mathbb{E}_i[\bm{A}_j \bm{\theta}^\star - \bm{b}_j]\right\|_2 \\ 
&=\left\|\mathbb{E}_{s_{j-1} \sim P^{j-i-1}(\cdot \mid s_i),s_j \sim P(\cdot \mid s_{j-1})}[\bm{A}_j \bm{\theta}^\star - \bm{b}_j]\right\|_2 \\ 
&= \left\|\mathbb{E}_{s_{j-1} \sim P^{j-i-1}(\cdot \mid s_i),s_j \sim P(\cdot \mid s_{j-1})}[\bm{A}_j \bm{\theta}^\star - \bm{b}_j]-\mathbb{E}_{s_{j-1} \sim \mu,s_j \sim P(\cdot \mid s_{j-1})}[\bm{A}_j \bm{\theta}^\star - \bm{b}_j]\right\|_2 \\ 
&\leq d_{\mathsf{TV}}(P^{j-i-1}(\cdot\mid s_i),\mu) \cdot \sup_{s_{j-1},s_j}\|\bm{A}_j \bm{\theta}^\star - \bm{b}_j\|_2 \\ 
&\leq m\rho^{j-i-1} (2\|\bm{\theta}^\star\|_2+1).
\end{align*}
Therefore, the norm of $\bm{U}_i$ is bounded by
\begin{align}\label{eq:U-bound}
\|\bm{U}_i\|_2 &\leq \|\bm{A}_i \bm{\theta}^\star - \bm{b}_i\|_2 + \sum_{j=i+1}^{\infty}\mathbb{E}_i \|\bm{A}_j \bm{\theta}^\star - \bm{b}_j\|_2 \nonumber \\ 
&\leq (2\|\bm{\theta}^\star\|_2+1) \left(1+\sum_{j=i+1}^{\infty} m\rho^{j-i-1}\right)\nonumber \\ 
& \lesssim \frac{1}{1-\rho} (2\|\bm{\theta}^\star\|_2+1);
\end{align}
On the other hand, $\bm{A}_i \bm{\theta}^\star - \bm{b}_i$ can be represented as
\begin{align}\label{eq:Ui-telescope}
\bm{A}_i \bm{\theta}^\star - \bm{b}_i &= \bm{U}_i - \mathbb{E}_i[\bm{U}_{i+1}]\nonumber \\
&= (\bm{U}_i - \mathbb{E}_{i-1}[\bm{U}_i]) + (\mathbb{E}_{i-1}[\bm{U}_i] - \mathbb{E}_i[\bm{U}_{i+1}])\nonumber \\ 
&=: \bm{m}_i + (\mathbb{E}_{i-1}[\bm{U}_i] - \mathbb{E}_i[\bm{U}_{i+1}])
\end{align}
Here, the first term $\bm{U}_i -\bm{U}_{i+1}$ can be analyzed by the telescoping technique, while 
\begin{align}\label{eq:defn-mi}
\bm{m}_i:= \bm{U}_{i} - \mathbb{E}_{i-1}[\bm{U}_{i}]
\end{align}
is a martingale difference process. Furthermore, we observe that when $s_0$ is drawn from the stationary distribution $\mu$, the covariance matrix $\mathbb{E}_{0}\left[\bm{m}_i \bm{m}_i^\top\right]$ is time-invariant, and can be expressed as
\begin{align}\label{eq:var-mi}
\mathbb{E}[\bm{m}_i \bm{m}_i^\top ] &= \mathbb{E}[\bm{m}_1 \bm{m}_1^\top]\nonumber \\ 
&= \mathbb{E}[(\bm{U}_1 - \mathbb{E}_0[\bm{U}_1])(\bm{U}_1 - \mathbb{E}_0[\bm{U}_1])^\top]\nonumber\\
&= \mathbb{E}[\bm{U}_1 \bm{U}_1^\top] - \mathbb{E}[\mathbb{E}_0[\bm{U}_1]\mathbb{E}_0[\bm{U}_1^\top]]\nonumber\\
&\overset{(i)}{=} \mathbb{E}[\bm{U}_1 \bm{U}_1^\top] - \mathbb{E}[\mathbb{E}_1[\bm{U}_2]\mathbb{E}_1[\bm{U}_2^\top]]\nonumber\\
&= \mathbb{E}[(\bm{A}_1 \bm{\theta}^\star -\bm{b}_1 + \mathbb{E}_1[\bm{U}_2])(\bm{A}_1 \bm{\theta}^\star -\bm{b}_1 + \mathbb{E}_1[\bm{U}_2])^\top] - \mathbb{E}[\mathbb{E}_1[\bm{U}_2]\mathbb{E}_1[\bm{U}_2^\top]]\nonumber\\
&= \mathbb{E}[(\bm{A}_1 \bm{\theta}^\star -\bm{b}_1)(\bm{A}_1 \bm{\theta}^\star -\bm{b}_1)^\top] + \mathbb{E}[(\bm{A}_1 \bm{\theta}^\star -\bm{b}_1)\mathbb{E}_1[\bm{U}_2]^\top] + \mathbb{E}[\mathbb{E}_1[\bm{U}_2](\bm{A}_1 \bm{\theta}^\star -\bm{b}_1)^\top]\nonumber\\
&= \mathbb{E}[(\bm{A}_1 \bm{\theta}^\star -\bm{b}_1)(\bm{A}_1 \bm{\theta}^\star -\bm{b}_1)^\top]\nonumber \\ 
&+ \sum_{j=2}^{\infty} \mathbb{E}[(\bm{A}_1 \bm{\theta}^\star -\bm{b}_1)(\bm{A}_j \bm{\theta}^\star - \bm{b}_j)^\top + (\bm{A}_j \bm{\theta}^\star - \bm{b}_j)(\bm{A}_1 \bm{\theta}^\star -\bm{b}_1)^\top]\nonumber\\
&= \widetilde{\bm{\Gamma}},
\end{align}
according to the definition of $\widetilde{\bm{\Gamma}}$ (as in eq.~\eqref{eq:defn-tilde-Gamma}). Notice here that we applied the rule of total expectation throughout this deduction, and took advantage of the time-invariant property of the distribution of $\{\bm{U}_i\}_{1 \leq i \leq T}$ in (i).

In order to relate $I_2$ to the martingale difference process $\bm{m}_i$, we invoke the  relation \eqref{eq:Ui-telescope} to obtain
\begin{align*}
\frac{1}{T}\sum_{i=1}^T \bm{Q}_i (\bm{A}_i \bm{\theta}^\star - \bm{b}_i)
&= \frac{1}{T}\sum_{i=1}^T \bm{Q}_i \bm{m}_i + \frac{1}{T}\sum_{i=1}^T \bm{Q}_i (\mathbb{E}_{i-1}[\bm{U}_i] - \mathbb{E}_i[\bm{U}_{i+1}]) \\ 
&= \frac{1}{T}\sum_{i=1}^T \bm{Q}_i \bm{m}_i + \frac{1}{T}\sum_{i=1}^T (\bm{Q}_{i-1}\mathbb{E}_{i-1}[\bm{U}_i] - \bm{Q}_{i}\mathbb{E}_i[\bm{U}_{i+1}]) + \frac{1}{T}\sum_{i=1}^T (\bm{Q}_i - \bm{Q}_{i-1})\mathbb{E}_{i-1}[\bm{U}_i] \\ 
&= \underset{I_{21}}{\underbrace{\frac{1}{T}\sum_{i=1}^T \bm{Q}_i \bm{m}_i}} + \underset{I_{22}}{\underbrace{\frac{1}{T}(\bm{Q}_0 \mathbb{E}_0[\bm{U}_1] - \bm{Q}_T \mathbb{E}_T [\bm{U}_{T+1}])}}+ \underset{I_{23}}{\underbrace{\frac{1}{T}\sum_{i=1}^T (\bm{Q}_i - \bm{Q}_{i-1})\mathbb{E}_{i-1}[\bm{U}_i]}}
\end{align*}
where we applied the telescoping technique in the last equation. The uniform boundedness of $\|\bm{Q}_t\|$, as indicated by Lemma \ref{lemma:Q-bound}, and the uniform boundedness of $\|\bm{U}_i\|_2$, as indicated by \eqref{eq:U-bound}, guarantee that
\begin{align}
\label{eq:markov-bar-deltat-I22-bound}
\left\|\frac{1}{T}(\bm{Q}_0 \mathbb{E}_0[\bm{U}_1] - \bm{Q}_T \mathbb{E}_T [\bm{U}_{T+1}])\right\|_2  
&\leq \frac{1}{T} (\|\bm{Q}_0\| \sup \|\bm{U}_1\|_2 + \|\bm{Q}_T\| \sup \|\bm{U}_T\|_2)\notag\\ 
&\lesssim  \left(\frac{2}{\beta}\right)^{\frac{1}{1-\alpha}}\left(\frac{2m}{1-\rho}\right)(2\|\bm{\theta}^\star\|_2+1)
\end{align}
deterministically.
Meanwhile, the norm of $I_{23}$ is bounded by invoking Lemma \ref{lemma:delta-Q}:
\begin{align}
\label{eq:markov-bar-deltat-I23-bound}
\left\|\frac{1}{T}\sum_{i=1}^T (\bm{Q}_i - \bm{Q}_{i-1})\mathbb{E}_{i-1}[\bm{U}_i]\right\| \lesssim \eta_0 \left[\eta_0 \Gamma\left(\frac{1}{1-\alpha}\right)+\alpha\right]\left(\frac{1}{\beta}\right)^{\frac{1}{1-\alpha}} \left(\frac{2m}{1-\rho}\right)(2\|\bm{\theta}^\star\|_2+1)\frac{\log T}{T}
\end{align}
almost surely. 

It now boils down to bounding the norm of $I_{21}$. Towards this end, we firstly observe that
\begin{align*}
\frac{1}{T}\sum_{i=1}^T \mathbb{E}_{s_{i-1}\sim\mu,s_i \sim P(\cdot \mid s_{i-1})}\|\bm{Q}_i\bm{m}_i\|_2^2 = \mathsf{Tr}(\tilde{\bm{\Lambda}}_T),
\end{align*}
and that 
\begin{align*}
\frac{1}{T}\|\bm{Q}_i\bm{m}_i\|_2 \leq \frac{1}{T} \eta_0 \left(\frac{2}{\beta}\right)^{\frac{1}{1-\alpha}}\left(\frac{2m}{1-\rho}\right)(2\|\bm{\theta}^\star\|_2+1)
\end{align*}
almost surely for all $i \in [T]$, according to Lemma \ref{lemma:Q-bound} and Equation \eqref{eq:U-bound}. Now consider a sequence of matrix-valued functions $\bm{F}_i:\mathcal{S} \times \mathcal{S} \to \mathbb{R}^{(d+1) \times (d+1)}$, defined as
\begin{align*}
\bm{F}_i(s,s') = \begin{pmatrix}
0 & (\bm{Q}_i\bm{m}_i(s,s'))^\top \\ 
\bm{Q}_i\bm{m}_i(s,s') & \bm{0}_{d\times d}.
\end{pmatrix}, \quad \forall i \in [T].
\end{align*}
It can then be verified that 
\begin{align*}
\left\|\mathbb{E}_{s\sim \mu, s' \sim P(\cdot \mid s)}[\bm{F}_i^2(s,s')] \right\|= \mathbb{E}_{s\sim \mu, s' \sim P(\cdot \mid s)}\|\bm{Q}_i\bm{m}_i(s,s')\|_2^2,
\end{align*}
and that
\begin{align*}
\|\bm{F}_i(s,s')\| = \|\bm{Q}_i\bm{m}_i(s,s')\|_2, \quad \forall s,s' \in \mathcal{S}.
\end{align*}
Therefore, a direct application of Corollary \ref{thm:matrix-bernstein-mtg} yields
\begin{align}\label{eq:markov-bar-deltat-I21-bound}
\left\|\frac{1}{T}\sum_{i=1}^T \bm{Q}_i \bm{m}_i\right\|_2 &\lesssim 2\sqrt{\frac{2\mathsf{Tr}(\tilde{\bm{\Lambda}}_T)}{T}\log \frac{12d}{\delta}} \nonumber \\ 
&+ \eta_0 \left(\frac{2}{\beta}\right)^{\frac{1}{1-\alpha}}\left(\frac{2m}{1-\rho}\right)(2\|\bm{\theta}^\star\|_2+1)(1-\lambda)^{-\frac{1}{4}} \frac{1}{T}\log^{\frac{3}{2}}\frac{6d}{\delta},
\end{align}
with probability at least $1-\frac{\delta}{3}$.

\paragraph{Bounding $I_3$} Applying a similar technique as in the proof of Theorem \ref{thm:markov-deltat-convergence}, we decompose the term $I_3$ as
\begin{align}\label{eq:markov-bar-deltat-I3-decompose}
&\frac{1}{T}\sum_{i=1}^T \bm{Q}_i (\bm{A}_i-\bm{A})\bm{\Delta}_{i-1} \nonumber \\ 
&=\frac{1}{T} \sum_{i=1}^T \bm{Q}_i(\bm{A}_i - \bm{A}) (\bm{\Delta}_{i-1} -\bm{\Delta}_{i_{\mix}}) + \frac{1}{T}\sum_{i=1}^T \bm{Q}_i(\bm{A}_i - \bm{A}) \bm{\Delta}_{i_{\mix}} \nonumber \\
&= \underset{I_{31}}{\underbrace{\frac{1}{T} \sum_{i=1}^T \bm{Q}_i(\bm{A}_i - \bm{A}) (\bm{\Delta}_{i-1} -\bm{\Delta}_{i_{\mix}})}} + \underset{I_{32}}{\underbrace{\frac{1}{T}\sum_{i=1}^T \bm{Q}_i(\bm{A}_i - \mathbb{E}_{i_{\mix}}[\bm{A}_i])\bm{\Delta}_{i_{\mix}} }}+ \underset{I_{33}}{\underbrace{\frac{1}{T}\sum_{i=1}^T \bm{Q}_i(\mathbb{E}_{i_{\mix}}[\bm{A}_i] - \bm{A}) \bm{\Delta}_{i_{\mix}} }}.
\end{align}
Recall from the proof of Theorem \ref{thm:markov-deltat-convergence} that 
\begin{align*}
\left\|\bm{\Delta}_{i-1} -\bm{\Delta}_{i_{\mix}}\right\|_2&= \left\|\bm{\theta}_{i-1} -\bm{\theta}_{i_{\mix}}\right\|_2 \\ 
&\leq \frac{t_{\mix}}{1-\alpha} \eta_i (2 \max_{1 \leq j < i} \|\bm{\Delta}_j\|_2 + 2\|\bm{\theta}^\star\|_2  + 1);
\end{align*}
hence the norm of $I_{31}$ can be bounded by
\begin{align}
\label{eq:markov-bar-deltat-I31-bound}
\left\|\frac{1}{T} \sum_{i=1}^T \bm{Q}_i(\bm{A}_i - \bm{A}) (\bm{\Delta}_{i-1} -\bm{\Delta}_{i_{\mix}})\right\|_2
&\lesssim \frac{1}{T} \sum_{i=1}^T \|\bm{Q}_i\| \cdot \frac{t_{\mix}}{1-\alpha} \eta_i (2 \max_{1 \leq j < i} \|\bm{\Delta}_j\|_2 + 2\|\bm{\theta}^\star\|_2  + 1) \nonumber \\ 
&\lesssim \frac{\tmix}{(1-\alpha)T}\left(\frac{2}{\beta}\right)^{\frac{1}{1-\alpha}}\sum_{i=1}^T \eta_i (2 \max_{1 \leq j < T} \|\bm{\Delta}_j\|_2 + 2\|\bm{\theta}^\star\|_2  + 1) \nonumber \\ 
&\lesssim \frac{\tmix \eta_0 }{(1-\alpha)^2}  \left(\frac{2}{\beta}\right)^{\frac{1}{1-\alpha}} (2 \max_{1 \leq j < T} \|\bm{\Delta}_j\|_2 + 2\|\bm{\theta}^\star\|_2  + 1) T^{-\alpha}.
\end{align}
The term $I_{32}$ can be decomposed into $\tmix$ martingales and bounded by the vector Azuma's inequality, invoking a similar technique to the tackling of the term $I_4$ in the proof of Theorem \ref{thm:markov-deltat-convergence}. With details omitted, we obtain with probability at least $1-\frac{\delta}{3}$ that
\begin{align}\label{eq:markov-bar-deltat-I32-bound}
&\left\|\frac{1}{T}\sum_{i=1}^T \bm{Q}_i(\bm{A}_i - \mathbb{E}_{i_{\mix}}[\bm{A}_i])\tilde{\bm{\Delta}}_{i_{\mix}}\right\|_2^2 \lesssim \left(\frac{2}{\beta}\right)^{\frac{1}{1-\alpha}}  \sqrt{\frac{\tmix}{1-\alpha}\log \frac{9\tmix}{\delta}}R' T^{-\frac{\alpha+1}{2}}.
\end{align}
The term $I_{33}$ is bounded by the mixing property of the Markov chain, specifically
\begin{align}\label{eq:markov-bar-deltat-I33-bound}
\left\|\frac{1}{T}\sum_{i=1}^T \bm{Q}_i(\mathbb{E}_{i_{\mix}}[\bm{A}_i] - \bm{A}) \tilde{\bm{\Delta}}_{i_{\mix}}\right\|_2 \lesssim \frac{1}{T} \left(\frac{2}{\beta}\right)^{\frac{1}{1-\alpha}} T^{-\frac{\alpha}{2}} \sum_{i=1}^T R' (i_{\mix})^{-\frac{\alpha}{2}} \lesssim \left(\frac{2}{\beta}\right)^{\frac{1}{1-\alpha}} R' T^{-\alpha}.
\end{align}
\paragraph{Completing the proof} Combining \eqref{eq:markov-bar-deltat-I21-bound}, \eqref{eq:markov-bar-deltat-I22-bound}, \eqref{eq:markov-bar-deltat-I23-bound}, \eqref{eq:markov-bar-deltat-I31-bound}, \eqref{eq:markov-bar-deltat-I32-bound}, and \eqref{eq:markov-bar-deltat-I33-bound} by a union bound argument and plugging in the definition of $R'$, we obtain
\begin{align*}
\|\bar{\bm{\Delta}}_T\|_2 &\lesssim 2\sqrt{\frac{2\mathsf{Tr}(\widetilde{\bm{\Lambda}}_T)}{T} \log \frac{6d}{\delta}} \\ 
&+ \frac{\eta_0\tmix(T^{-\frac{\alpha}{2}})}{(1-\alpha)^2}  (2\|\bm{\theta}^\star\|_2+1)\left(\frac{1-\gamma}{2}\lambda_0\eta_0\right)^{-\frac{1}{1-\alpha}}T^{-\alpha} \\ 
&+ \eta_0 \sqrt{\frac{2\tmix(T^{-\frac{\alpha}{2}})}{2\alpha-1}\log \frac{27T \tmix(T^{-\frac{\alpha}{2}})}{\delta}}(2\|\bm{\theta}^\star\|_2+1)\left(\frac{1-\gamma}{2}\lambda_0\eta_0\right)^{-\frac{2+\alpha}{2(1-\alpha)}}T^{-\alpha} \\ 
&+ \frac{\eta_0\tmix(T^{-\frac{\alpha}{2}})}{\sqrt{(1-\alpha)(2\alpha-1)}} \log \frac{27T \tmix(T^{-\frac{\alpha}{2}})}{\delta}(2\|\bm{\theta}^\star\|_2+1)\left(\frac{1-\gamma}{2}\lambda_0\eta_0\right)^{-\frac{2+\alpha}{2(1-\alpha)}}T^{-\frac{\alpha+1}{2}} \\ 
&+ \eta_0 \frac{m}{1-\rho} (2\|\bm{\theta}^\star\|_2+1)\left(\frac{1-\gamma}{4}\lambda_0\eta_0\right)^{-\frac{1}{1-\alpha}}T^{-1} \\ 
&\cdot \left[(1-\lambda)^{-\frac{1}{4}}\log^{\frac{3}{2}}\frac{6d}{\delta} + \left(\eta_0 \Gamma\left(\frac{1}{1-\alpha}\right)+\alpha\right) \log T\right].
\end{align*}
Notice that all the terms beginning from the second line can all be bounded by
\begin{align*}
\widetilde{C}T^{-\alpha}\log^{\frac{3}{2}} \frac{dT}{\delta},
\end{align*}
where $\widetilde{C}$ is a problem-related quantity depending on $\alpha,\eta_0,\lambda_0, \gamma, m,\rho$ and $\lambda$. The theorem follows immediately.

\subsection{Proof of Theorem \ref{thm:TD-Berry--Esseen}} \label{app:proof-TD-Berry--Esseen}
Following the precedent of \cite{wu2024statistical}, we approach this Berry--Esseen bound by introducing a Gaussian comparison term. Specifically, the triangle inequality indicates
\begin{align}\label{eq:TD-Berry--Esseen-decompose}
d_{\mathsf{C}}(\sqrt{T} \bar{\bm{\Delta}}_T,\mathcal{N}(\bm{0},\widetilde{\bm{\Lambda}}^{\star}))\leq d_{\mathsf{C}}(\sqrt{T} \bar{\bm{\Delta}}_T,\mathcal{N}(\bm{0},\widetilde{\bm{\Lambda}}_T)) + d_{\mathsf{C}}(\mathcal{N}(\bm{0},\widetilde{\bm{\Lambda}}_T), \mathcal{N}(\bm{0},\widetilde{\bm{\Lambda}}^{\star}))
\end{align}
where the second term on the right-hand-side can be bounded by the following proposition.
\begin{lemma}\label{lemma:Gaussian-comparison}
With $\tilde{\bm{\Lambda}}^\star$ and $\tilde{\bm{\Lambda}}_T$ defined as in \eqref{eq:defn-tilde-Lambdastar} and \eqref{eq:defn-tilde-LambdaT} respectively, it can be guaranteed for any $\eta_0 \leq \frac{1}{2\lambda_{\Sigma}}$ that
\begin{align*}
d_{\mathsf{C}}(\mathcal{N}(\bm{0},\widetilde{\bm{\Lambda}}_T), \mathcal{N}(\bm{0},\widetilde{\bm{\Lambda}}^{\star})) \lesssim \frac{\sqrt{d\mathsf{cond}(\bm{\tilde{\Gamma}})}}{(1-\gamma)\lambda_0\eta_0} T^{\alpha-1} + O(T^{2\alpha-2}).
\end{align*}
\end{lemma}
\begin{proof}
This lemma is a direct generalization of Theorem 3.3 in \cite{wu2024statistical}, where $\bar{\bm{\Lambda}}_T$ is replaced by $\tilde{\bm{\Lambda}}_T$ and $\bm{\Lambda}^\star$ is replaced by $\tilde{\bm{\Lambda}}^\star$. 
\end{proof}

We next focus on the first term on the right-hand-side of \eqref{eq:TD-Berry--Esseen-decompose}. For this, we notice that 
\begin{align*}
d_{\mathsf{C}}(\sqrt{T} \bar{\bm{\Delta}}_T,\mathcal{N}(\bm{0},\widetilde{\bm{\Lambda}}_T)) = d_{\mathsf{C}}(\sqrt{T} \bm{A}\bar{\bm{\Delta}}_T,\mathcal{N}(\bm{0},\bm{A}\widetilde{\bm{\Lambda}}_T\bm{A}^\top));
\end{align*}
and we will focus on bounding the latter.

Recall that $\sqrt{T}\bm{A}\bar{\bm{\Delta}}_T$ can be decomposed as
\begin{align}\label{eq:delta-decomposition-markov}
\sqrt{T}\bm{A}\bar{\bm{\Delta}}_T&= \underset{I_1}{\underbrace{\frac{\bm{A}}{\sqrt{T} \eta_0} \bm{Q}_0 \bm{\Delta}_0}} - \underset{I_2}{\underbrace{\frac{\bm{A}}{\sqrt{T}} \sum_{i=1}^T \bm{Q}_i (\bm{A}_i  - \bm{A})\bm{\Delta}_{i-1}}} -\underset{I_3} {\underbrace{\frac{\bm{A}}{\sqrt{T}} \sum_{i=1}^T \bm{Q}_i (\bm{A}_i \bm{\theta}^\star - \bm{b}_i) }}.
\end{align}
In order to derive the non-asymptotic rate at which $\sqrt{T}\bm{A}\bar{\bm{\Delta}}_T$ converges to its Gaussian distribution, we derive the convergence of $I_1$, $I_2$ and $I_3$ accordingly in the following paragraphs. For readability concerns, we will only keep track of dependence on $T$ and $d$ in this proof, and use $\widetilde{C}$ to denote any problem-related parameters that are related to $\alpha,\gamma,\eta_0,\lambda_0,m,\rho$.

\paragraph{The $a.s.$ convergence of $I_1$} 
Lemma \ref{lemma:Q-bound} directly implies that as $T \to \infty$, $I_1$ is bounded by
\begin{align}\label{eq:markov-CLT-I1-converge}
\left\|\frac{\bm{A}}{\sqrt{T} \eta_0} \bm{Q}_0 \bm{\Delta}_0\right\|_2 \leq \frac{1}{\sqrt{T} \eta_0} \|\bm{A}\bm{Q}_0\| \|\bm{\Delta}_0\|_2 \lesssim \lambda_{\Sigma} \left(\frac{2}{\beta}\right)^{\frac{1}{1-\alpha}}\|\bm{\theta}^\star\|_2 T^{-\frac{1}{2}}.
\end{align}
almost surely.

\paragraph{Bounding $I_2$ with high probability} The convergence of $I_2$ is result of the uniform boundedness of  $\bm{Q}_i$, the convergence of $\{\bm{\Delta}_t\}$, and the mixing property of the Markov chain. Specifically, we again apply the technique in the proof of Theorem \ref{thm:markov-deltat-convergence} and define
\begin{align}\label{eq:defn-markov-CLT-tmix}
&\tmix = \tmix(T^{-\frac{1}{2}}), \quad \text{and} \quad i_{\mix} = \max\left\{i-\tmix, 0\right\}.
\end{align}
Assumption \ref{as:mixing} implies that (see \ref{eq:tmix-bound-L2}) 
\begin{align}
\label{eq:tmix.bound}
\tmix \leq \frac{\log m + \frac{1}{2}\log T}{\log(1/\rho)} \lesssim \frac{\log T}{1-\rho}.
\end{align}
The term $I_2$ can be decomposed as
\begin{align*}
&\frac{1}{\sqrt{T}} \sum_{i=1}^T \bm{AQ}_i(\bm{A}_i - \bm{A}) (\bm{\Delta}_i -\bm{\Delta}_{i_{\mix}}) + \frac{1}{\sqrt{T}}\sum_{i=1}^T \bm{AQ}_i(\bm{A}_i - \bm{A}) \bm{\Delta}_{i_{\mix}} \\
&= \underset{I_{21}}{\underbrace{\frac{1}{\sqrt{T}} \sum_{i=1}^T \bm{AQ}_i(\bm{A}_i - \bm{A}) (\bm{\Delta}_{i-1} -\bm{\Delta}_{i_{\mix}})}} + \underset{I_{22}}{\underbrace{\frac{1}{\sqrt{T}}\sum_{i=1}^T \bm{AQ}_i(\bm{A}_i - \mathbb{E}_{i_{\mix}}[\bm{A}_i])\bm{\Delta}_{i_{\mix}} }}\\ 
&+ \underset{I_{23}}{\underbrace{\frac{1}{\sqrt{T}}\sum_{i=1}^T \bm{AQ}_i(\mathbb{E}_{i_{\mix}}[\bm{A}_i] - \bm{A}) \bm{\Delta}_{i_{\mix}} }};
\end{align*}
for the term $I_{21}$, recall that the difference between $\bm{\Delta}_{i-1}$ and $\bm{\Delta}_{i_{\mix}}$ can be further decomposed into
\begin{align*}
\bm{\Delta}_{i-1}-\bm{\Delta}_{i_{\mix}} = \bm{\theta}_{i-1} - \bm{\theta}_{i_{\mix}} &= \sum_{j=i_{\mix}+1}^{i-1} \eta_j (\bm{\theta}_{j} - \bm{\theta}_{j-1})\\  
&= -\sum_{j=i_{\mix}+1}^{i-1} \eta_j(\bm{A}_j \bm{\theta}_{j-1} - \bm{b}_j) \\ 
&= -\sum_{j=i_{\mix}+1}^{i-1} \eta_j(\bm{A}_j \bm{\theta}^\star - \bm{b}_j) - \sum_{j=i_{\mix}+1}^{i-1} \eta_j\bm{A}_j \bm{\Delta}_{j-1}.
\end{align*}

\color{black}

Hence, the decomposition of $I_2$ can be expressed as
\begin{align}\label{eq:markov-Berry--Esseen-I2-decompose}
I_2 &= -\underset{I_{20}}{\underbrace{\frac{1}{\sqrt{T}}\sum_{i=1}^T \left[\bm{AQ}_i (\bm{A}_i - \bm{A})\sum_{j=i_{\mix}+1}^{i-1} \eta_j (\bm{A}_j \bm{\theta}^\star - \bm{b}_j)\right]}} \nonumber \\ 
&- \underset{I_{21}'}{\underbrace{\frac{1}{\sqrt{T}}\sum_{i=1}^T \left[\bm{AQ}_i (\bm{A}_i - \bm{A})\sum_{j=i_{\mix}+1}^{i-1} \eta_j \bm{A}_j \bm{\Delta}_{j-1}\right]}} + I_{22} + I_{23},
\end{align}
where the norm of $I_{20}$ is bounded almost surely by
\begin{align}\label{eq:markov-Berry--Esseen-I20}
\|I_{20}\|_2&\leq \frac{1}{\sqrt{T}}\sum_{i=1}^T \|\bm{AQ}_i\| \|\bm{A}_i - \bm{A}\| \cdot \sum_{j=i_{\mix}+1}^{i-1} \eta_j (2\|\bm{\theta}^\star\|_2 + 1)\nonumber \\ 
&\lesssim \frac{1}{\sqrt{T}}\sum_{i=1}^T (2+\widetilde{C}i^{\alpha-1})\cdot \tmix \eta_i (2\|\bm{\theta}^\star\|_2 + 1)\nonumber \\ 
&\lesssim (2\|\bm{\theta}^\star\|_2 + 1) \left[\frac{\eta_0 }{1-\rho}T^{\frac{1}{2}-\alpha}\log T + \widetilde{C}T^{-\frac{1}{2}}\log^2 T\right]\nonumber \\
&= \frac{\eta_0 }{1-\rho}(2\|\bm{\theta}^\star\|_2 + 1) T^{\frac{1}{2}-\alpha}\log T + o(T^{\frac{1}{2}-\alpha}).
\end{align}
Here, the second line follows from Lemma \ref{lemma:Q-bound}, and the third line uses the bound \eqref{eq:tmix.bound} on $\tmix$.


For the term $I_{21}'$, we invoke the fact that for any vectors $\bm{x}_1,\bm{x}_2,\ldots,\bm{x}_n \in \mathbb{R}^d$, it can be guaranteed that
\begin{align*}
\left\|\sum_{i=1}^n \bm{x}_i\right\|_2^2 \leq n \sum_{i=1}^n \|\bm{x}_i\|_2^2;
\end{align*}
Consequently, the norm of $I_{21}'$ is bounded, in expectation, by
\begin{align}\label{eq:markov-Berry--Esseen-I21}
\mathbb{E}\|I_{21}'\|_2^2 &= \frac{1}{T}  \mathbb{E}\left\|\sum_{i=1}^T \sum_{j=i_{\mix}+1}^{i-1}\bm{AQ}_i (\bm{A}_i - \bm{A}) \eta_j \bm{A}_j \bm{\Delta}_{j-1}\right\|_2^2 \nonumber \\ 
&\leq \frac{1}{T} \cdot (T\tmix) \sum_{i=1}^T \sum_{j=i_{\mix}+1}^{i-1} \mathbb{E}\|\bm{AQ}_i (\bm{A}_i - \bm{A}) \eta_j \bm{A}_j \bm{\Delta}_{j-1}\|_2^2 \nonumber 
\\ 
&\leq \tmix \sum_{i=1}^T \left\|\bm{AQ}_i\right\|^2 \|\bm{A}_i - \bm{A}\|^2 \cdot \left(4  \sum_{j=i_{\mix}+1}^{i-1} \eta_j^2 \mathbb{E}\|\bm{\Delta}_{j-1}\|_2^2\right)\nonumber \\ 
&\lesssim \eta_0^2 \tmix^2 (2\|\bm{\theta}^\star\|_2+1)^2 \sum_{i=1}^T (2+\widetilde{C}i^{\alpha-1})i^{-2\alpha} \left(\frac{\eta_0}{\lambda_0(1-\gamma)}\frac{1}{(1-\rho)^2}i^{-\alpha} \log^2 i + \widetilde{C}'i^{-1}\log^2 i\right)\nonumber \\ 
& \lesssim \widetilde{C}(2\|\bm{\theta}^\star\|_2+1)^2 T^{1-3\alpha} \log^4 T = o(T^{\frac{3}{2}-3\alpha}),
\end{align}
where we invoke Theorem \ref{thm:markov-L2-convergence} in the fourth line.

The term $I_{22}$ can be decomposed into $\tmix$ martingales:
\begin{align*}
&\left\|\frac{1}{\sqrt{T}}\sum_{i=1}^T \bm{AQ}_i(\bm{A}_i - \mathbb{E}_{i_{\mix}}[\bm{A}_i])\bm{\Delta}_{i_{\mix}}\right\|_2^2   \\
&= \frac{1}{T}\left\|\sum_{r=1}^{\tmix} \sum_{i=0}^{T'-1} \bm{AQ}_{i\tmix + r}(\bm{A}_{i\tmix + r} - \mathbb{E}_{(i-1)\tmix + r}[\bm{A}_{i\tmix + r}])\bm{\Delta}_{(i-1)\tmix + r}\right\|_2^2 \\ 
&\leq \frac{\tmix}{T} \sum_{r=1}^{\tmix} \left\|\sum_{i=0}^{T'-1} \bm{AQ}_{i\tmix + r}(\bm{A}_{i\tmix + r} - \mathbb{E}_{(i-1)\tmix + r}[\bm{A}_{i\tmix + r}])\bm{\Delta}_{(i-1)\tmix + r}\right\|.
\end{align*}
Notice here that for any $r \in [\tmix]$, the sequence
\begin{align*}
\left\{\bm{AQ}_{i\tmix + r}(\bm{A}_{i\tmix + r} - \mathbb{E}_{(i-1)\tmix + r}[\bm{A}_{i\tmix + r}])\bm{\Delta}_{(i-1)\tmix + r}\right\}_{i=0}^{T'-1}
\end{align*}
is a martingale difference. Therefore, its expected norm can be bounded by
\begin{align*}
&\mathbb{E}\left\|\sum_{i=0}^{T'-1}\bm{AQ}_{i\tmix + r}(\bm{A}_{i\tmix + r} - \mathbb{E}_{(i-1)\tmix + r}[\bm{A}_{i\tmix + r}])\bm{\Delta}_{(i-1)\tmix + r}\right\|_2^2 \\ 
&= \sum_{i=0}^{T'-1} \mathbb{E}\left\|\bm{AQ}_{i\tmix + r}(\bm{A}_{i\tmix + r} - \mathbb{E}_{(i-1)\tmix + r}[\bm{A}_{i\tmix + r}])\bm{\Delta}_{(i-1)\tmix + r} \right\|_2^2 \\ 
&\lesssim  \sum_{i=0}^{T'-1} \|\bm{AQ}_{i\tmix + r}\|^2 \mathbb{E}\|\bm{\Delta}_{(i-1)\tmix + r}\|_2^2 
\end{align*}
Therefore, the norm of $I_{22}$ is bounded by
\begin{align}
\label{eq:markov-CLT-I22-bound}
&\mathbb{E}\left\|\frac{1}{\sqrt{T}}\sum_{i=1}^T \bm{AQ}_i(\bm{A}_i - \mathbb{E}_{i_{\mix}}[\bm{A}_i])\bm{\Delta}_{i_{\mix}}\right\|_2^2 \nonumber \\
&\leq \frac{\tmix}{T}  \sum_{i=1}^T \|\bm{AQ}_i\|^2\mathbb{E}\|\bm{\Delta}_{i_{\mix}}\|_2^2 \nonumber \\ 
&\leq \frac{\log T}{(1-\rho) T} \sum_{i=1}^T (2+O(i^{\alpha-1}))^2  \left[\frac{\eta_0}{\lambda_0(1-\gamma)} \frac{1}{(1-\rho)^2} (2\|\bm{\theta}^\star\|_2+1)^2 i_{\mix}^{-\alpha} \log^2 i + O(i_{\mix}^{-1} \log^2 i)\right] \nonumber \\ 
&\leq \frac{\eta_0}{\lambda_0(1-\gamma)} \frac{1}{(1-\rho)^3}  (2\|\bm{\theta}^\star\|_2+1)^2 T^{-\alpha} \log^3 T + O(T^{-1}\log^3 T)\nonumber \\ 
&= \frac{\eta_0}{\lambda_0(1-\gamma)} \frac{1}{(1-\rho)^3}  (2\|\bm{\theta}^\star\|_2+1)^2 T^{-\alpha} \log^3 T + o(T^{-\alpha}).
\end{align}

For $I_{23}$, we make use of the fact that since 
\begin{align*}
\max_{s \in \mathcal{S}} d_{\mathsf{TV}}(P^{\tmix}(\cdot \mid s), \mu) \leq T^{-1/2},
\end{align*}
the difference between $\mathbb{E}_{i_{\mix}}[\bm{A}_i]$ and $\bm{A} = \mathbb{E}_{\mu}[\bm{A}_i]$ is bounded by
\begin{align*}
\left\|\mathbb{E}_{i_{\mix}}[\bm{A}_i] - \mathbb{E}_{\mu}[\bm{A}_i] \right\|\leq \max_{s \in \mathcal{S}} d_{\mathsf{TV}}(P^{\tmix}(\cdot \mid s), \mu) \cdot \sup_{s_{i-1},s_{i}} \|\bm{A}_i\| \leq 2T^{-1/2}.
\end{align*}
Hence, by AM-GM inequality, the expected norm of $I_{23}$ is bounded by
\begin{align}\label{eq:markov-CLT-I23-bound}
&\mathbb{E}\left\|\frac{1}{\sqrt{T}}\sum_{i=1}^T \bm{AQ}_i(\mathbb{E}_{i_{\mix}}[\bm{A}_i] - \bm{A}) \bm{\Delta}_{i_{\mix}}\right\|_2^2 \nonumber \\
&\leq \sum_{i=1}^T \mathbb{E}\left\|\bm{AQ}_i(\mathbb{E}_{i_{\mix}}[\bm{A}_i] - \bm{A}) \bm{\Delta}_{i_{\mix}} \right\|_2^2 \nonumber \\ 
&\leq \sum_{i=1}^T \mathbb{E} \left\{ \|\bm{AQ}_i\|^2 \|\mathbb{E}_{i_{\mix}}[\bm{A}_i] - \bm{A}\|^2 \|\bm{\Delta}_{i_{\mix}} \|_2^2 \right\}\nonumber \\ 
&\lesssim \sum_{i=1}^T (2+O(i^{\alpha-1})) (T^{-1}) \left[\frac{\eta_0}{\lambda_0(1-\gamma)} \frac{1}{(1-\rho)^2} (2\|\bm{\theta}^\star\|_2+1)^2 i_{\mix}^{-\alpha} \log^2 i + O(i_{\mix}^{-1} \log^2 i)\right]\nonumber \\ 
&\lesssim \frac{\eta_0}{\lambda_0(1-\gamma)} \frac{1}{(1-\rho)^2} (2\|\bm{\theta}^\star\|_2+1)^2 T^{-\alpha}\log^2 T + O(T^{-1}\log^2 T) \nonumber \\ 
&=  \frac{\eta_0}{\lambda_0(1-\gamma)} \frac{1}{(1-\rho)^2} (2\|\bm{\theta}^\star\|_2+1)^2 T^{-\alpha}\log^2 T +o(T^{-\alpha}). 
\end{align}
Combining \eqref{eq:markov-Berry--Esseen-I2-decompose}, \eqref{eq:markov-Berry--Esseen-I20}, \eqref{eq:markov-Berry--Esseen-I21}, \eqref{eq:markov-CLT-I22-bound} and \eqref{eq:markov-CLT-I23-bound}, we obtain
\begin{align*}
\mathbb{E}\left\|I_2 - \widetilde{C}_1'T^{\frac{1-2\alpha}{2}} \log T + o(T^{\frac{1}{2}-\alpha})\right\|_2^2 
&= \mathbb{E}\left\|I_2 - I_{20}\right\|_2^2 \\ 
&\leq 3\left(\mathbb{E}\|I_{21}'\|_2^2 + \mathbb{E}\|I_{22}\|_2^2 + \mathbb{E}\|I_{23}\|_2^2\right)\\ 
&\lesssim (\widetilde{C}_2')^3 T^{-\alpha}\log^3 T + o(T^{\frac{3}{2}-3\alpha} + T^{-\alpha}),
\end{align*}
where we use $\widetilde{C}_1'$ and $\widetilde{C}_2'$ to denote problem-related quantities
\begin{align}
&\widetilde{C}_1' = \frac{\eta_0}{1-\rho}(2\|\bm{\theta}^\star\|_2+1), \quad \text{and} \label{eq:Berry--Esseen-C10}\\ 
&\widetilde{C}_2' = \frac{1}{1-\rho} \left(\frac{\eta_0(2\|\bm{\theta}^\star\|_2+1)^2}{\lambda_0(1-\gamma)}\right)^{\frac{1}{3}}. \label{eq:Berry--Esseen-C20}
\end{align}
Therefore, the Chebyshev's inequality  directly implies that
\begin{align*}
&\mathbb{P}\left(\left\|I_2 - \widetilde{C}_1'T^{\frac{1}{2}-\alpha}\log T -o(T^{\frac{1}{2}-\alpha}) \right\|_2 \gtrsim \widetilde{C}_2'T^{-\frac{\alpha}{3}}\log T + o(T^{\frac{1}{2}-\alpha} + T^{-\frac{\alpha}{3}})\right)\\
 &\lesssim \widetilde{C}_2'T^{-\frac{\alpha}{3}}\log T + o(T^{\frac{1}{2}-\alpha} + T^{-\frac{\alpha}{3}}).
\end{align*}
Applying the triangle inequality, we obtain the bound on $I_2$ with high probability by triangle inequality:
\begin{align}\label{eq:markov-Berry--Esseen-I2}
\mathbb{P}\left(\left\|I_2 \right\|_2 \gtrsim \widetilde{C}_1'T^{\frac{1}{2}-\alpha}\log T + \widetilde{C}_2'T^{-\frac{\alpha}{3}}\log T + o(T^{\frac{1}{2}-\alpha} + T^{-\frac{\alpha}{3}})\right)
 \lesssim \widetilde{C}_2'T^{-\frac{\alpha}{3}}\log T + o(T^{\frac{1}{2}-\alpha} + T^{-\frac{\alpha}{3}}).
\end{align}

\paragraph{A Berry--Esseen bound for $I_3$}

Following the decomposition of the term $I_2$ in the proof of Theorem \ref{thm:TD-whp}, we represent $I_3$ as
\begin{align*}
&\frac{1}{\sqrt{T}}\sum_{i=1}^T \bm{AQ}_i(\bm{A}_i \bm{\theta}^\star - \bm{b}_i )\\
 &= \underset{I_{31}}{\underbrace{\frac{1}{\sqrt{T}}\sum_{i=1}^T \bm{AQ}_i \bm{m}_i}} + \underset{I_{32}}{\underbrace{\frac{\bm{A}}{\sqrt{T}}(\bm{Q}_0 \mathbb{E}_0[\bm{U}_1] - \bm{Q}_T \mathbb{E}_T [\bm{U}_{T+1}])}}+ \underset{I_{33}}{\underbrace{\frac{\bm{A}}{\sqrt{T}}\sum_{i=1}^T (\bm{Q}_i - \bm{Q}_{i-1})\mathbb{E}_{i-1}[\bm{U}_i]}}
\end{align*}
where $\bm{U}_i$ is defined as in \eqref{eq:defn-Ui} and $\bm{m}_i$ is defined as in \eqref{eq:defn-mi}. Here, the norm of $I_{32}$ and $I_{33}$ can be bounded by $O(T^{-\frac{1}{2}})$ almost surely; it now boils down to the term $I_{31}$, for which we aim to apply Corollary \ref{thm:Berry--Esseen-mtg}. Specifically, let 
\begin{align*}
\bm{f}_i(s_i,s_{i-1}) = \bm{AQ}_i \bm{m}_i,
\end{align*}
it is easy to verify that for all $i \in [T]$,
\begin{align*}
\|\bm{f}_i(s_i,s_{i-1})\|_2 &\leq \|\bm{A}\bm{Q}_i \bm{m}_i\|_2  \\ 
&\leq \|\bm{AQ}_i\|\|\bm{m}_i\|_2\\ 
&\leq (2+O(i^{\alpha-1})) \cdot \frac{m}{1-\rho}(2\|\bm{\theta}^\star\|_2+1)\\ 
&\lesssim \eta_0 \left(\frac{1}{\beta}\right)^{\frac{1}{1-\alpha}}\frac{m}{1-\rho}(2\|\bm{\theta}^\star\|_2+1), \quad \text{a.s.}
\end{align*}
Meanwhile,
\begin{align*}
&\frac{1}{T}\sum_{i=1}^T \mathbb{E}[\bm{f}_i\bm{f}_i^\top] = \bm{A}\tilde{\bm{\Lambda}}_T \bm{A}^\top, \quad \text{with} \\ 
&\|\bm{A}\tilde{\bm{\Lambda}}_T \bm{A}^\top - \tilde{\bm{\Gamma}}\| \leq O(T^{\alpha-1}).
\end{align*} 
and when $T$ satisfies \eqref{eq:Lambda-T-condition}, it can be guaranteed that $\lambda_{\min}(\bm{A}\tilde{\bm{\Lambda}}_T \bm{A}^\top) \geq \frac{1}{2}\lambda_{\min}(\tilde{\bm{\Gamma}})$. Hence, a direct application of Corollary \ref{thm:Berry--Esseen-mtg} reveals that
\begin{align}\label{eq:markov-Berry--Esseen-I31}
d_{\mathsf{C}}\left(\frac{1}{\sqrt{T}}\sum_{i=1}^T \bm{AQ}_i \bm{m}_i,\mathcal{N}(\bm{0},\bm{A}\tilde{\bm{\Lambda}}_T \bm{A}^\top)\right) \leq \widetilde{C}_3 T^{-\frac{1}{4}}\log T + o(T^{-\frac{1}{4}}),
\end{align}
where $\widetilde{C}_3$ is a problem-related quantity
\begin{align}
\widetilde{C}_3 &= \Bigg\{ \left(\frac{p}{(p-1)(1-\lambda)}\log\left(d\left\|\frac{\mathrm{d}\nu}{\mathrm{d}\mu}\right\|_{\mu,p}\right)\right)^{\frac{1}{4}}\cdot \frac{m}{1-\rho}(2\|\bm{\theta}^\star\|_2+1) \nonumber \\ 
&+ \sqrt{\frac{m}{1-\rho}(2\|\bm{\theta}^\star\|_2+1)} \cdot \eta_0 \left(\frac{1}{(1-\gamma)\lambda_0\eta_0}\right)^{\frac{1}{2(1-\alpha)}} \log^{\frac{1}{4}}(d\|\tilde{\bm{\Gamma}}\|) \Bigg\} \cdot \sqrt{d}\|\tilde{\bm{\Gamma}}\|_{\mathsf{F}}^{\frac{1}{2}}.\label{eq:Berry--Esseen-C3}
\end{align}

\paragraph{Completing the proof} 
The proof now boils down to combining the convergence rate of $I_1,I_2$ and the Berry--Esseen bound on $I_3$. For simplicity, we denote
\begin{align*}
\bm{\delta}_T := \sqrt{T} \bm{A}\bar{\bm{\Delta}}_T - \frac{1}{\sqrt{T}}\sum_{i=1}^T\bm{AQ}_i{\bm{m}}_i = I_1 - I_2 - I_{32} - I_{33}.
\end{align*}
From the previous calculations, we have shown that
\begin{align*}
\mathbb{P}\left(\left\|\bm{\delta_T} \right\|_2 \gtrsim \widetilde{C}_1'T^{\frac{1}{2}-\alpha}\log T + \widetilde{C}_2'T^{-\frac{\alpha}{3}}\log T + o(T^{\frac{1}{2}-\alpha} + T^{-\frac{\alpha}{3}})\right)
 \lesssim \widetilde{C}_2'T^{-\frac{\alpha}{3}}\log T + o(T^{\frac{1}{2}-\alpha} + T^{-\frac{\alpha}{3}}),
\end{align*}
and that
\begin{align*}
\sup_{\mathcal{A} \in \mathscr{C}}\left|\mathbb{P}\left(\frac{1}{\sqrt{T}}\sum_{i=1}^T{\bm{AQ}_i\bm{m}}_i\in \mathcal{A}\right) - \mathbb{P}(\bm{A}\widetilde{\bm{\Lambda}}_T^{ \frac{1}{2}}\bm{z} \in \mathcal{A})\right| \leq \widetilde{C}_3 T^{-\frac{1}{4}}{\log T} + o(T^{-\frac{1}{4}}).
\end{align*}
We now combine these two results to bound the difference between the distributions of $\sqrt{T} \bm{A}\bar{\bm{\Delta}}_T$ and $\mathcal{N}(\bm{0},\bm{A}\widetilde{\bm{\Lambda}}_T\bm{A}^\top)$. Considering any convex set $\mathcal{A} \subset \mathbb{R}^d$, define
\begin{align*}
&\mathcal{A}^{\varepsilon} := \{\bm{x} \in \mathbb{R}^d: \inf_{y \in \mathcal{A}} \|\bm{x} - \bm{y}\|_2 \leq \varepsilon\}, \quad \text{and}  \qquad \mathcal{A}^{-\varepsilon}:=\{\bm{x} \in \mathbb{R}^d: B(\bm{x},\varepsilon) \subset \mathcal{A}\}. 
\end{align*}
Direct calculation yields
\begin{align*}
\mathbb{P}(\sqrt{T}\bm{A}\bar{\bm{\Delta}}_T \in \mathcal{A}) &= \mathbb{P}(\sqrt{T}\bm{A}\bar{\bm{\Delta}}_T \in \mathcal{A},\|\bm{\delta}_T\|_2 > \varepsilon) + \mathbb{P}(\sqrt{T}\bm{A}\bar{\bm{\Delta}}_T \in \mathcal{A},\|\bm{\delta}_T\|_2 \leq \varepsilon) \\ 
&\leq \mathbb{P}(\|\bm{\delta}_T\|_2 > \varepsilon) + \mathbb{P}(\sqrt{T}\bm{A}\bar{\bm{\Delta}}_T \in \mathcal{A},\|\bm{\delta}_T\|_2 \leq \varepsilon).
\end{align*}
Here, the triangle inequality implies
\begin{align*}
\left(\sqrt{T}\bm{A}\bar{\bm{\Delta}}_T \in \mathcal{A},\|\bm{\delta}_T\|_2 \leq \varepsilon \right) \Rightarrow \frac{1}{\sqrt{T}}\sum_{i=1}^T\bm{AQ}_i{\bm{m}}_i \in \mathcal{A}^{\varepsilon}.
\end{align*}
Hence, $\mathbb{P}(\sqrt{T}\bm{A}\bar{\bm{\Delta}}_T \in \mathcal{A})$ is upper bounded by
\begin{align*}
\mathbb{P}(\sqrt{T}\bm{A}\bar{\bm{\Delta}}_T \in \mathcal{A}) 
&\leq \mathbb{P}(\|\bm{\delta}_T\|_2 > \varepsilon) + \mathbb{P}\left(\frac{1}{\sqrt{T}}\sum_{i=1}^T{\bm{AQ}_i\bm{m}}_i \in \mathcal{A}^{\varepsilon}\right)\\
&\leq \mathbb{P}(\|\bm{\delta}_T\|_2 > \varepsilon) + \mathbb{P}\left(\bm{A}\widetilde{\bm{\Lambda}}_T^{\frac{1}{2}}\bm{z} \in \mathcal{A}^{\varepsilon}\right)+\widetilde{C}_3 T^{-\frac{1}{4}}{\log T}  + o(T^{-\frac{1}{4}})\\ 
&\leq \mathbb{P}(\|\bm{\delta}_T\|_2 > \varepsilon) + \mathbb{P}\left(\bm{A}\widetilde{\bm{\Lambda}}_T^{\frac{1}{2}}\bm{z} \in \mathcal{A}\right) + \|\tilde{\bm{\Gamma}}\|_{\mathsf{F}}^{\frac{1}{2}}\varepsilon +\widetilde{C}_3 T^{-\frac{1}{4}}{\log T}  + o(T^{-\frac{1}{4}}),
\end{align*}
where we invoked Theorem \ref{thm:Gaussian-reminder} in the last inequality. By letting 
\begin{align*}
\varepsilon \asymp \widetilde{C}_1'T^{\frac{1}{2}-\alpha}\log T + \widetilde{C}_2'T^{-\frac{\alpha}{3}}\log T + o(T^{\frac{1}{2}-\alpha} + T^{-\frac{\alpha}{3}}),
\end{align*}
we obtain
\begin{align*}
\mathbb{P}(\sqrt{T}\bm{A}\bar{\bm{\Delta}}_T \in \mathcal{A}) 
&\leq \widetilde{C}_2'T^{-\frac{\alpha}{3}}\log T + o(T^{\frac{1}{2}-\alpha} + T^{-\frac{\alpha}{3}})  + \mathbb{P}\left(\bm{A}\widetilde{\bm{\Lambda}}_T^{\frac{1}{2}}\bm{z} \in \mathcal{A}\right) \\ 
&+ \|\tilde{\bm{\Gamma}}\|_{\mathsf{F}}^{\frac{1}{2}} \left(\widetilde{C}_1'T^{\frac{1}{2}-\alpha}\log T + \widetilde{C}_2'T^{-\frac{\alpha}{3}}\log T + o(T^{\frac{1}{2}-\alpha} + T^{-\frac{\alpha}{3}})\right) +\widetilde{C}_3 T^{-\frac{1}{4}}{\log T}  + o(T^{-\frac{1}{4}}) \\ 
&= \mathbb{P}\left(\bm{A}\widetilde{\bm{\Lambda}}_T^{\frac{1}{2}}\bm{z} \in \mathcal{A}\right) + (\widetilde{C}_1 T^{\frac{1}{2}-\alpha} + \widetilde{C}_2 T^{-\frac{\alpha}{3}}+ \widetilde{C}_3 T^{-\frac{1}{4}}){\log T}  + o(T^{\frac{1}{2}-\alpha} + T^{-\frac{\alpha}{3}} + T^{-\frac{1}{4}}).
\end{align*}
Here, in the last equality, we denote, for simplicity,
\begin{align}
&\widetilde{C}_1 = \|\tilde{\bm{\Gamma}}\|_{\mathsf{F}}^{\frac{1}{2}}\widetilde{C}_1' = \|\tilde{\bm{\Gamma}}\|_{\mathsf{F}}^{\frac{1}{2}}\frac{\eta_0}{1-\rho}(2\|\bm{\theta}^\star\|_2+1), \quad \text{and} \label{eq:Berry--Esseen-C1} \\ 
&\widetilde{C}_2 = (\|\tilde{\bm{\Gamma}}\|_{\mathsf{F}}^{\frac{1}{2}} + 1) \widetilde{C}_2' = (\|\tilde{\bm{\Gamma}}\|_{\mathsf{F}}^{\frac{1}{2}} + 1)\frac{1}{1-\rho} \left(\frac{\eta_0(2\|\bm{\theta}^\star\|_2+1)^2}{\lambda_0(1-\gamma)}\right)^{\frac{1}{3}} \label{eq:Berry--Esseen-C2}.
\end{align}
Using the same technique as in the proof of Corollary \ref{thm:Berry--Esseen-mtg}, a lower bound can be derived symmetrically. By taking a supremum over $\mathcal{A} \in \mathscr{C}$, it can be guaranteed that
\begin{align}\label{eq:TD-Berry--Esseen-intermediate}
d_\mathsf{C}(\sqrt{T}\bm{A}\bar{\bm{\Delta}}_T,\mathcal{N}(\bm{0},\bm{A}\tilde{\bm{\Lambda}}_T \bm{A}^\top)) &\lesssim (\widetilde{C}_1 T^{\frac{1}{2}-\alpha} + \widetilde{C}_2 T^{-\frac{\alpha}{3}}+ \widetilde{C}_3 T^{-\frac{1}{4}}){\log T} + o(T^{\frac{1}{2}-\alpha} + T^{-\frac{\alpha}{3}} + T^{-\frac{1}{4}}).
\end{align}
Further combining \eqref{eq:TD-Berry--Esseen-intermediate} with \eqref{eq:TD-Berry--Esseen-decompose} and Lemma \ref{lemma:Gaussian-comparison}, we obtain the Berry--Esseen bound
\begin{align*}
d_{\mathsf{C}}(\sqrt{T}\bar{\bm{\Delta}}_T,\mathcal{N}(\bm{0},\tilde{\bm{\Lambda}}^\star)) &\lesssim (\widetilde{C}_1 T^{\frac{1}{2}-\alpha} + \widetilde{C}_2 T^{-\frac{\alpha}{3}}+ \widetilde{C}_3 T^{-\frac{1}{4}} + \widetilde{C}_4 T^{\alpha-1}){\log T} \\ 
& + o(T^{\frac{1}{2}-\alpha} + T^{-\frac{\alpha}{3}} + T^{-\frac{1}{4}} + T^{\alpha-1})
\end{align*}
with $\widetilde{C}_1$, $\widetilde{C}_2$, $\widetilde{C}_3$ defined as in \eqref{eq:Berry--Esseen-C1}, \eqref{eq:Berry--Esseen-C2}, \eqref{eq:Berry--Esseen-C3} respectively, and
\begin{align*}
\widetilde{C}_4 = \frac{\sqrt{d\mathsf{cond}(\bm{\tilde{\Gamma}})}}{(1-\gamma)\lambda_0\eta_0}.
\end{align*}
Finally, when $\alpha = \frac{3}{4}$, we have, coincidentally,
\begin{align*}
\frac{1}{2}-\alpha = -\frac{\alpha}{3} = -\frac{1}{4} = \alpha-1.
\end{align*} 
Hence, Theorem \ref{thm:TD-Berry--Esseen} follows from taking $\tilde{C} = \max\{\tilde{C}_1,\tilde{C}_2,\tilde{C}_3,\tilde{C}_4\}$.

\subsection{Proof of Relation~\eqref{eq:TD-Berry--Esseen-tight}}\label{app:proof-Berry--Esseen-tight} 
Following the same logic as Appendix B.4.1 in \cite{wu2024statistical}, we can obtain a lower bound on the difference between $\mathcal{N}(\bm{0},\tilde{\bm{\Lambda}}_T)$ and $\mathcal{N}(\bm{0},\tilde{\bm{\Lambda}}^\star)$. Specifically, when $T$ is sufficiently large,
\begin{align*}
d_{\mathsf{C}}(\mathcal{N}(\bm{0},\tilde{\bm{\Lambda}}_T),\mathcal{N}(\bm{0},\tilde{\bm{\Lambda}}^\star)) = d_{\mathsf{TV}}(\mathcal{N}(\bm{0},\tilde{\bm{\Lambda}}_T),\mathcal{N}(\bm{0},\tilde{\bm{\Lambda}}^\star)) \gtrsim O(T^{\alpha-1}).
\end{align*}
Meanwhile, when $\alpha > \frac{3}{4}$, both $\frac{1}{2}-\alpha$ and $-\frac{\alpha}{3}$ are less than $-\frac{1}{4}$. Therefore, the upper bound \eqref{eq:TD-Berry--Esseen-intermediate} is transformed as
\begin{align*}
d_\mathsf{C}(\sqrt{T}\bar{\bm{\Delta}}_T,\mathcal{N}(\bm{0},\tilde{\bm{\Lambda}}_T )) \leq O(T^{-\frac{1}{4}}).
\end{align*}
In combination, the triangle inequality reveals
\begin{align*}
d_{\mathsf{C}}(\sqrt{T}\bar{\bm{\Delta}}_T,\mathcal{N}(\bm{0},\tilde{\bm{\Lambda}}^\star)) &\geq d_{\mathsf{C}}(\mathcal{N}(\bm{0},\tilde{\bm{\Lambda}}_T),\mathcal{N}(\bm{0},\tilde{\bm{\Lambda}}^\star)) - d_\mathsf{C}(\sqrt{T}\bar{\bm{\Delta}}_T,\mathcal{N}(\bm{0},\tilde{\bm{\Lambda}}_T ))\\ 
&\gtrsim O(T^{\alpha-1}) - O(T^{-\frac{1}{4}})\\ 
&\gtrsim O(T^{\alpha-1}) \gtrsim O(T^{-\frac{1}{4}}).
\end{align*}
Here, in the last line, we applied the fact that since $\alpha > \frac{3}{4}$, $\alpha - 1 \geq -\frac{1}{4}$.

\paragraph{Choice of stepsizes} We conclude by noting that choice of the stepsize in Theorem \ref{thm:TD-whp} and Theorem \ref{thm:TD-Berry--Esseen} are different: in Theorem \ref{thm:TD-whp}, $\eta_0$ depends on $\delta$ and other problem-related quantities like $\lambda_0$ and $\gamma$, and $\alpha$ can take any value between $\frac{1}{2}$ and $1$; however, in Theorem \ref{thm:TD-Berry--Esseen}, the initial stepsize $\eta_0$ can take any value less than $1/2\lambda_{\Sigma}$, while $\alpha$ is set to the specific value of $\frac{3}{4}$. In fact, our proof of Theorem~\ref{thm:TD-Berry--Esseen} allows for a general choice of $\alpha$; however, using other values of $\alpha$ other than $3/4$ appears to be suboptimal.

\section{Proof of supportive lemmas and propositions}

\subsection{Proof of Proposition \ref{prop:PE}}\label{proof-prop-PE}
We address these recursive relations in order.
\paragraph{Proof of \eqref{eq:PE1}.} By definition and according to Proposition \ref{prop:P-top}, the left-hand-side is featured by
\begin{align*}
\left\|\left(\mathcal{P}^* \exp(t\bm{H})\bm{g}^{\parallel}\right)^{\parallel}\right\|_{\mu} &= \left\|\left(\exp(t\bm{H})\bm{g}^{\parallel}\right)^{\parallel}\right\|_{\mu} \\ 
&= \mathbb{E}_{\mu}\left[\exp(t\bm{H})\bm{g}^{\parallel}\right] = \left\|\mathbb{E}_{\mu} [\exp(t\bm{H})]\right\| \cdot \|\bm{g}^{\parallel}\|.
\end{align*}
Since $\mu(\bm{H}) = \bm{0}$, the expectation of $\exp(t\bm{H})$ can be bounded by
\begin{align*}
\left\|\mathbb{E}_{\mu}\exp(t\bm{H})\right\| = \left\|\sum_{k=0}^{\infty} \mathbb{E}_{\mu} \frac{1}{k!}(t\bm{H})^k\right\| &\leq \|\bm{I}\| + \sum_{k=2}^{\infty} \frac{1}{k!} \left\|\mathbb{E}_{\mu}(t\bm{H})^2 \right\| \\ 
&\leq 1 + \sum_{k=2}^{\infty} \frac{1}{k!} (tM)^k = \exp(tM) - tM.
\end{align*}
\paragraph{Proof of \eqref{eq:PE2}.} Since $\bm{g}^{\parallel}(x) = \mu(\bm{g}) = \|\bm{g}^{\parallel}\|_{\mu}$ is a constant function, the left hand side can be bounded by
\begin{align*}
\left\|\left(\mathcal{P}^* \exp(t\bm{H})\bm{g}^{\parallel}\right)^{\perp}\right\|_{\mu}&= \|\bm{g}^{\parallel}\|_{\mu} \cdot \left\|\left(\mathcal{P}^* \exp(t\bm{H})\right)^{\perp}\right\|_{\mu} = \|\bm{g}^{\parallel}\|_{\mu} \cdot \left\|\mathcal{P}^* \left(\exp(t\bm{H})\right)^{\perp}\right\|_{\mu} \\ 
&\leq \|\bm{g}^{\parallel}\|_{\mu} \cdot\lambda \left\| \left(\exp(t\bm{H})\right)^{\perp}\right\|_{\mu} \\ 
&= \|\bm{g}^{\parallel}\|_{\mu} \cdot\lambda \left\| \left(\exp(t\bm{H})-\bm{I}\right)^{\perp}\right\|_{\mu} \\ 
&\leq \|\bm{g}^{\parallel}\|_{\mu} \cdot\lambda \sup_{x \in \mathcal{S}} \left\| \exp(t\bm{H}(x))-\bm{I}\right\| \\ 
&\leq \|\bm{g}^{\parallel}\|_{\mu} \cdot\lambda (\exp(tM)-1).
\end{align*}
\paragraph{Proof of \eqref{eq:PE3}.} By definition, the left-hand-side can be represented as
\begin{align*}
\left\|\left(\mathcal{P}^* \exp(t\bm{H})\bm{g}^{\perp}\right)^{\parallel}\right\|_{\mu} = \left\|\mu\left(\mathcal{P}^* \exp(t\bm{H})\bm{g}^{\perp}\right)\right\| &= \left\|\mu(\exp(t\bm{H})\bm{g}^{\perp})\right\| \\ 
&= \left\|\mu((\exp(t\bm{H}) - \bm{I}) \bm{g}^{\perp})\right\| \\ 
&\leq \left\|(\exp(t\bm{H}) - \bm{I}) \bm{g}^{\perp}\right\|_{\mu} \\ 
&\leq \sup_{x \in \mathcal{S}} \left\|\exp(t\bm{H}) - \bm{I}\right\| \cdot \left\|\bm{g}^{\perp}\right\|_{\mu} \\ 
&\leq (\exp(tM)-1)  \left\|\bm{g}^{\perp}\right\|_{\mu}.
\end{align*}
\paragraph{Proof of \eqref{eq:PE4}.} As a direct implication of Proposition \ref{prop:P-top}, the left-hand-side is featured by
\begin{align*}
\left\|\left(\mathcal{P}^* \exp(t\bm{H})\bm{g}^{\perp}\right)^{\perp}\right\|_{\mu} = \left\|\mathcal{P}^* \left(\exp(t\bm{H})\bm{g}^{\perp}\right)^{\perp}\right\|_{\mu} 
&\leq \lambda \left\|\left(\exp(t\bm{H})\bm{g}^{\perp}\right)^{\perp}\right\|_{\mu} \\ 
&\leq \lambda \left\|\exp(t\bm{H})\bm{g}^{\perp}\right\|_{\mu} \\ 
&\leq \lambda \sup_{x \in \mathcal{S}}\|\exp(t\bm{H}(x))\| \|\bm{g}^{\perp}\|_{\mu} \\ 
&\leq \lambda \exp(tM) \|\bm{g}^{\perp}\|_{\mu}.
\end{align*}

\subsection{Proof of Lemma \ref{lemma:Uk}}\label{app:proof-lemma-Uk}
Direct calculation reveals
\begin{align}\label{eq:lemma-Uk-1}
\frac{\alpha_1 + \alpha_4}{2} + \frac{\sqrt{(\alpha_1 - \alpha_4)^2 + 4\alpha_3^2}}{2} \nonumber &= \alpha_1 + \frac{\sqrt{(\alpha_1 - \alpha_4)^2 + 4\alpha_3^2}}{2} - \frac{\alpha_1 -\alpha_4}{2} \nonumber \\ 
&= \alpha_1 + \frac{2\alpha_3^2}{\sqrt{(\alpha_1 - \alpha_4)^2 + 4\alpha_3^2} + (\alpha_1 -\alpha_4)}.
\end{align}
In what follows, we firstly illustrate that
\begin{align}\label{eq:lemma-Uk-2}
&\sqrt{(\alpha_1 - \alpha_4)^2 + 4\alpha_3^2} + (\alpha_1 -\alpha_4) \nonumber \\ 
&=\sqrt{((1-\lambda)e^x - x)^2 + 4(e^x-1)^2} + (1-\lambda)e^x - x\nonumber \\ 
&=: f(x) \geq 2(1-\lambda).
\end{align}
In fact, since $f(0) = 2(1-\lambda)$, it suffices to show that $f'(x) \geq 0$ for all $x \geq 0$. Towards this end, observe that
\begin{align*}
f'(x) &= \frac{((1-\lambda)e^x - x) ((1-\lambda) e^x-1) + 4(e^x-1)e^x}{\sqrt{((1-\lambda)e^x - x)^2 + 4(e^x-1)^2}} + (1-\lambda) e^x-1 \\ 
&= \frac{[(1-\lambda)e^x-1]\left[\sqrt{((1-\lambda)e^x - x)^2 + 4(e^x-1)^2} - ((1-\lambda)e^x - x)\right] + 4(e^x-1)e^x}{\sqrt{((1-\lambda)e^x - x)^2 + 4(e^x-1)^2}}
\end{align*}
Since the denominator is always positive, we now focus on showing that the numerator. Specifically, we discuss the following three cases:
\begin{enumerate}
\item If $(1-\lambda)e^x - 1 \geq 0$, then since $\sqrt{((1-\lambda)e^x - x)^2 + 4(e^x-1)^2} > (1-\lambda)e^x - x$, the numerator is positive;
\item If $(1-\lambda)e^x - 1 < 0$ and $(1-\lambda)e^x - x \geq 0$, then by triangle inequality,
\begin{align*}
\sqrt{((1-\lambda)e^x - x)^2 + 4(e^x-1)^2} - ((1-\lambda)e^x - x) \leq 2(e^x-1).
\end{align*}
Meanwhile, since $(1-\lambda)e^x - 1 > -1 > -e^x$, it can be guaranteed that
\begin{align*}
&[(1-\lambda)e^x-1]\left[\sqrt{((1-\lambda)e^x - x)^2 + 4(e^x-1)^2} - ((1-\lambda)e^x - x)\right] + 4(e^x-1)e^x \\ 
&> -e^x [2(e^x-1)]+ 4(e^x-1)e^x > 0.
\end{align*}
\item If $(1-\lambda)e^x - 1 < 0$ and $(1-\lambda)e^x - x < 0$, then also by triangle inequality,
\begin{align*}
\sqrt{((1-\lambda)e^x - x)^2 + 4(e^x-1)^2} - ((1-\lambda)e^x - x) &\leq 2(e^x-1) + 2(x-(1-\lambda)e^x) \\ 
& < 2(e^x-1 + x) < 4(e^x-1).
\end{align*}
Therefore, again because $(1-\lambda)e^x - 1 > -1 > -e^x$, the numerator is bounded below by
\begin{align*}
&[(1-\lambda)e^x-1]\left[\sqrt{((1-\lambda)e^x - x)^2 + 4(e^x-1)^2} - ((1-\lambda)e^x - x)\right] + 4(e^x-1)e^x \\ 
&> e^{-x} \cdot 4(e^x-1) + 4(e^x-1)e^x > 0.
\end{align*}
\end{enumerate}
In all three cases, we have proved that $f'(x) > 0$. This complete the proof of \eqref{eq:lemma-Uk-2}. As a direct consequence of \eqref{eq:lemma-Uk-1} and \eqref{eq:lemma-Uk-2}, we obtain
\begin{align*}
&\frac{\alpha_1 + \alpha_4}{2} + \frac{\sqrt{(\alpha_1 - \alpha_4)^2 + 4\alpha_3^2}}{2} \\ 
& \leq \alpha_1 + \frac{\alpha_3^2}{1-\lambda} = (e^x-x) + \frac{(e^x-1)^2}{1-\lambda}.
\end{align*}
We now proceed to further bound this upper bound. On one hand, when $x \in (0,1)$, It can be guaranteed that $e^x - x < 1+x^2$ and $e^x-1 < 2x$. Therefore,
\begin{align*}
(e^x-x) + \frac{(e^x-1)^2}{1-\lambda} &\leq 1 + x^2 + \frac{(2x)^2}{1-\lambda} \\ 
&< 1 + \frac{5x^2}{1-\lambda} < \exp\left(\frac{5x^2}{1-\lambda}\right)
\end{align*}
where we invoked the fact that $1+x < e^x$ in the last inequality. On the other hand, when $x > 1$, define
\begin{align*}
g(x) = \exp\left(\frac{5x^2}{1-\lambda}\right) - \left[(e^x-x) + \frac{(e^x-1)^2}{1-\lambda}\right];
\end{align*}
it is easy to illustrate that $g'(x) > 0$ for any $x > 1$, and therefore $g(x)$ is monotonically increasing with respect to $x$. This completes the proof of the Lemma.

\subsection{Proof of Proposition \ref{prop:Stein-smooth}}\label{app:proof-Stein-smooth}
This proposition is a generalization of Proposition 2.2 in \cite{gallouët2018regularity}, and the proofs are similar to each other. Recall from \cite{gallouët2018regularity}, proof of Proposition 2.2, that for any $\bm{\alpha} \in \mathbb{R}^d$ with $\|\bm{\alpha}\| =1$,
\begin{align*}
\bm{\alpha}^\top \nabla^2 f_g(\bm{x}) \bm{\alpha} &= -\int_0^1 \frac{1}{2(1-t)}\mathbb{E}\left[((\bm{\alpha}^\top \bm{z})^2 -1)g(\sqrt{t}\bm{x}+\sqrt{1-t}\bm{z})\right] \mathrm{d}t.
\end{align*}
Hence, the difference between $\nabla^2 f_g(\bm{x})$ and $\nabla^2 f_g(\bm{y})$ can be featured by
\begin{align}\label{eq:Stein-smooth-decompose}
&\bm{\alpha}^\top (\nabla^2 f_g(\bm{x})-\nabla^2 f_g(\bm{y})) \bm{\alpha} \nonumber \\ 
&= -\int_0^1 \frac{1}{2(1-t)}\mathbb{E}\left[((\bm{\alpha}^\top \bm{z})^2 -1)\left(g(\sqrt{t}\bm{x}+\sqrt{1-t}\bm{z})-g(\sqrt{t}\bm{y}+\sqrt{1-t}\bm{z})\right)\right] \mathrm{d}t \nonumber \\ 
&= -\underset{I_1}{\underbrace{\int_0^{1-\eta} \frac{1}{2(1-t)}\mathbb{E}\left[((\bm{\alpha}^\top \bm{z})^2 -1)\left(g(\sqrt{t}\bm{x}+\sqrt{1-t}\bm{z})-g(\sqrt{t}\bm{y}+\sqrt{1-t}\bm{z})\right)\right] \mathrm{d}t}} \nonumber \\ 
&- \underset{I_2}{\underbrace{\int_{1-\eta}^1 \frac{1}{2(1-t)}\mathbb{E}\left[((\bm{\alpha}^\top \bm{z})^2 -1)g(\sqrt{t}\bm{x}+\sqrt{1-t}\bm{z})\right] \mathrm{d}t}} \nonumber \\ 
&+ \underset{I_3}{\underbrace{\int_{1-\eta}^1 \frac{1}{2(1-t)}\mathbb{E}\left[((\bm{\alpha}^\top \bm{z})^2 -1)g(\sqrt{t}\bm{y}+\sqrt{1-t}\bm{z})\right] \mathrm{d}t}},
\end{align} 
where $\eta \in (0,1]$ is a variable to be determined later. We address the terms $I_1$, $I_2$ and $I_3$ accordingly.
\paragraph{Bounding $I_1$.} Since $g(\bm{x}) = h(\bm{\Sigma}^{\frac{1}{2}}\bm{x}+\bm{\mu})$ and $h \in \mathsf{Lip}_1$, it can be guaranteed that
\begin{align*}
\left|g(\sqrt{t}\bm{x}+\sqrt{1-t}\bm{z})-g(\sqrt{t}\bm{y}+\sqrt{1-t}\bm{z})\right| \leq \sqrt{t}\left\|\bm{\Sigma}^{\frac{1}{2}}(\bm{x} - \bm{y})\right\|_2;
\end{align*}
hence, $I_1$ can be bounded by
\begin{align*}
&\left|\int_0^{1-\eta} \frac{1}{2(1-t)}\mathbb{E}\left[((\bm{\alpha}^\top \bm{z})^2 -1)\left(g(\sqrt{t}\bm{x}+\sqrt{1-t}\bm{z})-g(\sqrt{t}\bm{y}+\sqrt{1-t}\bm{z})\right)\right] \mathrm{d}t\right| \\ 
&\leq \left\|\bm{\Sigma}^{\frac{1}{2}}(\bm{x} - \bm{y})\right\|_2 \mathbb{E}\left|(\bm{\alpha}^\top \bm{z})^2-1\right| \int_0^{1-\eta} \frac{\sqrt{t}}{2(1-t)}\mathrm{d}t.
\end{align*}
Here, since $\bm{z} \sim \mathcal{N}(\bm{0},\bm{I}_d)$ and $\|\bm{\alpha}\|_2 = 1$, we have $\bm{\alpha}^\top \bm{z} \sim \mathcal{N}(0,1)$ and therefore $(\bm{\alpha}^\top \bm{z})^2 \sim \chi^2(1)$. Consequently, $\mathbb{E}\left|(\bm{\alpha}^\top \bm{z})^2-1\right|$ is the standard error of $\chi^2(1)$ distribution, thus a universal constant. Meanwhile, the integral is bounded by
\begin{align*}
\int_0^{1-\eta} \frac{\sqrt{t}}{2(1-t)}\mathrm{d}t \leq \int_0^{1-\eta} \frac{1}{2(1-t)}\mathrm{d}t=-\frac{1}{2}(\log \eta).
\end{align*}
In combination, the term $I_1$ is bounded by
\begin{align}\label{eq:Stein-smooth-I1}
&\left|\int_0^{1-\eta} \frac{1}{2(1-t)}\mathbb{E}\left[((\bm{\alpha}^\top \bm{z})^2 -1)\left(g(\sqrt{t}\bm{x}+\sqrt{1-t}\bm{z})-g(\sqrt{t}\bm{y}+\sqrt{1-t}\bm{z})\right)\right] \mathrm{d}t\right| \nonumber \\ 
&\lesssim \left\|\bm{\Sigma}^{\frac{1}{2}}(\bm{x} - \bm{y})\right\|_2 (-\log \eta).
\end{align}

\paragraph{Bounding $I_2$.} Since $(\bm{\alpha}^\top \bm{z})^2 \sim \chi^2(1)$, we naturally have $\mathbb{E}[(\bm{\alpha}^\top \bm{z})^2] = 1$. Therefore, $I_2$ can be rephrased as
\begin{align*}
&\int_{1-\eta}^1 \frac{1}{2(1-t)}\mathbb{E}\left[((\bm{\alpha}^\top \bm{z})^2 -1)g(\sqrt{t}\bm{x}+\sqrt{1-t}\bm{z})\right] \mathrm{d}t\\ 
&= \int_{1-\eta}^1 \frac{1}{2(1-t)}\mathbb{E}\left[((\bm{\alpha}^\top \bm{z})^2 -1)\left(g(\sqrt{t}\bm{x}+\sqrt{1-t}\bm{z})-g(\sqrt{t}\bm{x})\right)\right] \mathrm{d}t,
\end{align*}
and its absolute value can be bounded by
\begin{align*}
&\left|\int_{1-\eta}^1 \frac{1}{2(1-t)}\mathbb{E}\left[((\bm{\alpha}^\top \bm{z})^2 -1)g(\sqrt{t}\bm{x}+\sqrt{1-t}\bm{z})\right] \mathrm{d}t\right| \\ 
&\leq \int_{1-\eta}^1 \frac{1}{2(1-t)}\mathbb{E}\left[|(\bm{\alpha}^\top \bm{z})^2 -1|\left|g(\sqrt{t}\bm{x}+\sqrt{1-t}\bm{z})-g(\sqrt{t}\bm{x})\right|\right] \mathrm{d}t \\ 
&\leq \int_{1-\eta}^1 \frac{1}{2(1-t)}\mathbb{E}\left[|(\bm{\alpha}^\top \bm{z})^2 -1| \sqrt{1-t}\|\bm{\Sigma}^{\frac{1}{2}}\bm{z}\|_2\right] \mathrm{d}t \\ 
&\leq \|\bm{\Sigma}^{\frac{1}{2}}\| \mathbb{E}[|(\bm{\alpha}^\top \bm{z})^2 -1|\|\bm{z}\|_2] \int_{1-\eta}^1 \frac{1}{2\sqrt{1-t}}\mathrm{d}t
\end{align*}
As is illustrated by Equation (26) in \cite{gallouët2018regularity}, the expectation is bounded by
\begin{align*}
\mathbb{E}[|(\bm{\alpha}^\top \bm{z})^2 -1|\|\bm{z}\|_2] \lesssim \sqrt{d},
\end{align*}
and the integral is bounded by
\begin{align*}
\int_{1-\eta}^1 \frac{1}{2\sqrt{1-t}}\mathrm{d}t \leq \sqrt{\eta}.
\end{align*}
So in combination, the term $I_2$ is bounded by
\begin{align}\label{eq:Stein-smooth-I2}
&\left|\int_{1-\eta}^1 \frac{1}{2(1-t)}\mathbb{E}\left[((\bm{\alpha}^\top \bm{z})^2 -1)g(\sqrt{t}\bm{x}+\sqrt{1-t}\bm{z})\right] \mathrm{d}t\right| \\ 
&\lesssim \sqrt{d}\|\bm{\Sigma}^{\frac{1}{2}}\| \sqrt{\eta}.
\end{align}
The term $I_3$ can be bounded by a similar manner.
\paragraph{Completing the proof.} Combining \eqref{eq:Stein-smooth-decompose}, \eqref{eq:Stein-smooth-I1} and \eqref{eq:Stein-smooth-I2} by triangle inequality, we obtain
\begin{align*}
&\left|\bm{\alpha}^\top (\nabla^2 f_g(\bm{x})-\nabla^2 f_g(\bm{y})) \bm{\alpha}\right| \\ 
&\lesssim \left\|\bm{\Sigma}^{\frac{1}{2}}(\bm{x} - \bm{y})\right\|_2 (-\log \eta) + \sqrt{d}\|\bm{\Sigma}^{\frac{1}{2}}\| \sqrt{\eta}.
\end{align*}
In the case where $2\|\bm{\Sigma}^{\frac{1}{2}}(\bm{x} - \bm{y})\|_2 >\sqrt{d}\|\bm{\Sigma}^{\frac{1}{2}}\|_2$, we can simply take $\eta=1$, yielding the bound
\begin{align}\label{eq:Stein-smooth-case1}
\left|\bm{\alpha}^\top (\nabla^2 f_g(\bm{x})-\nabla^2 f_g(\bm{y})) \bm{\alpha}\right| \lesssim \sqrt{d} \|\bm{\Sigma}^{\frac{1}{2}}\|;
\end{align}
otherwise, when $2\|\bm{\Sigma}^{\frac{1}{2}}(\bm{x} - \bm{y})\|_2 \leq \sqrt{d}\|\bm{\Sigma}^{\frac{1}{2}}\|$, we can set
\begin{align*}
\eta = \frac{4\|\bm{\Sigma}^{\frac{1}{2}}(\bm{x} - \bm{y})\|_2^2}{d\|\bm{\Sigma}\|},
\end{align*}
yielding the bound 
\begin{align}\label{eq:Stein-smooth-case2}
\left|\bm{\alpha}^\top (\nabla^2 f_g(\bm{x})-\nabla^2 f_g(\bm{y})) \bm{\alpha}\right| \lesssim 2\|\bm{\Sigma}^{\frac{1}{2}}(\bm{x} - \bm{y})\|_2 \left(1+\log \frac{\sqrt{d}\|\bm{\Sigma}^{\frac{1}{2}}\|}{\|\bm{\Sigma}^{\frac{1}{2}}(\bm{x} - \bm{y})\|_2}\right).
\end{align}
For simplicity, we use $f(x,a)$ to denote the piecewise function
\begin{align*}
f(x,a) = \begin{cases}
&x + x\log a - x \log x, \quad \text{if } x \in [0,a]; \\ 
&a, \quad \text{if } x > a.
\end{cases}
\end{align*}
It is easy to illustrate that
\begin{align*}
f(x) \leq (1+\log a)^+ x + e^{-1}.
\end{align*}
Therefore, by combining \eqref{eq:Stein-smooth-case1} and \eqref{eq:Stein-smooth-case2}, we obtain
\begin{align*}
\left|\bm{\alpha}^\top (\nabla^2 f_g(\bm{x})-\nabla^2 f_g(\bm{y})) \bm{\alpha}\right|  &\leq f(2\|\bm{\Sigma}^{\frac{1}{2}}(\bm{x} - \bm{y})\|_2,\sqrt{d}\|\bm{\Sigma}^{\frac{1}{2}}\|) \\ 
&\leq (2+\log (d\|\bm{\Sigma}\|))^+ \cdot \|\bm{\Sigma}^{\frac{1}{2}}(\bm{x} - \bm{y})\|_2 + e^{-1}.
\end{align*}

\subsection{Proof of Equation \eqref{eq:correlated-norms}}\label{app:proof-correlated-norms}
Essentially, it suffices to show that for any fixed matrix $\bm{A} \in \mathbb{R}^{d \times d}$ and random vector $\bm{x} \in \mathbb{R}^d$,
\begin{align}
\mathbb{E}\|\bm{Ax}\|_2^2 \mathbb{E}\|\bm{x}\|_2 \leq \mathbb{E}\|\bm{Ax}\|_2^2\|\bm{x}\|_2.
\end{align}
To see this, we use
\begin{align*}
\bm{A}^\top \bm{A} = \bm{PDP}^\top
\end{align*}
to denote the eigen decomposition of $\bm{A}^\top\bm{A}$, where $\bm{P}$ is a ortho-normal matrix and $\bm{D} = \text{diag}\{\lambda_1,\ldots,\lambda_d\}$ where $\lambda_1 \geq \lambda_2 \geq \ldots \geq \lambda_d \geq 0$. Further denote $\bm{y} = \bm{Px}$, then the norms of $\bm{x}$ and $\bm{Ax}$ can be represented by
\begin{align*}
&\|\bm{Ax}\|_2^2 = \bm{x}^\top \bm{PDP}^\top \bm{x} = \bm{y}^\top \bm{Dy} = \sum_{i=1}^d \lambda_i y_i^2, \quad \text{and} \\ 
&\|\bm{x}\|_2 = \|\bm{y}\|_2 = \sqrt{\sum_{i=1}^d y_i^2}.
\end{align*}
For every $i \in [d]$, it is easy to verify that
\begin{align*}
y_i^2 \quad \text{and} \quad \|\bm{y}\|_2
\end{align*}
are positively correlated, and therefore
\begin{align*}
\mathbb{E}\|\bm{Ax}\|_2^2 \mathbb{E}\|\bm{x}\|_2 &= \mathbb{E}\left[\sum_{i=1}^d \lambda_i y_i^2\right]\mathbb{E}\|\bm{y}\|_2 = \sum_{i=1}^d \lambda_i \mathbb{E}[y_i^2] \mathbb{E}\|\bm{y}\|_2 \\ 
&\leq \sum_{i=1}^d \lambda_i \mathbb{E}[y_i^2 \|\bm{y}\|_2] = \mathbb{E}\left[\left(\sum_{i=1}^d \lambda_i y_i^2\right) \cdot \|\bm{y}\|_2 \right] = \mathbb{E}\|\bm{Ax}\|_2^2\|\bm{x}\|_2.
\end{align*}
Here, the inequality on the third line follows from the Chebyshev's association inequality.

\subsection{Proof of Lemma \ref{lemma:E-delta-tmix}}\label{app:proof-lemma-E-delta-tmix}
The TD update rule \eqref{eq:TD-update-all} directly implies that
\begin{align}\label{eq:theta-tmix-decompose}
\bm{\theta}_t - \bm{\theta}_{t-t_{\mix}}&= \sum_{i=t-t_{\mix}}^{t-1} (\bm{\theta}_{i+1} - \bm{\theta}_i) \nonumber \\ 
&= \sum_{i=t-t_{\mix}}^{t-1} \eta_i (\bm{A}_i\bm{\theta}_i-\bm{b}_i)\nonumber \\ 
&= \sum_{i=t-t_{\mix}}^{t-1} \eta_i (\bm{A}_i\bm{\theta}^\star-\bm{b}_i) + \sum_{i=t-t_{\mix}}^{t-1} \eta_i \bm{A}_i \bm{\Delta}_i.
\end{align}
We will apply this relation to prove the three bounds respectively.
\paragraph{Proof of Equation \eqref{eq:E-delta-tmix-1}.} By triangle inequality, \eqref{eq:theta-tmix-decompose} implies that
\begin{align*}
\mathbb{E}\|\bm{\theta}_t - \bm{\theta}_{t-t_{\mix}}\|_2 & \leq  \sum_{i=t-t_{\mix}}^{t-1} \eta_i \mathbb{E}\|(\bm{A}_i\bm{\theta}^\star-\bm{b}_i)\|_2 + \sum_{i=t-t_{\mix}}^{t-1} \eta_i \mathbb{E}\|\bm{A}_i \bm{\Delta}_i\| \\ 
&\leq \sum_{i=t-t_{\mix}}^{t-1} \eta_i (2\|\bm{\theta}^\star\|_2+1) + \sum_{i=t-t_{\mix}}^{t-1} \eta_i 2\mathbb{E}\|\bm{\Delta}_i\|_2\\ 
&\leq t_{\mix}\eta_{t-t_{\mix}}(2\|\bm{\theta^\star}\|_2 + 1)+ 2\eta_{t - t_{\mix}}\sum_{i=t-t_{\mix}}^{t-1}\mathbb{E}\|\bm{\Delta}_i\|_2,
\end{align*}
where the last line follows from the fact that the stepsizes $\{\eta_t\}_{t \geq 0}$ are non-increasing.
\paragraph{Proof of Equation \eqref{eq:E-delta-tmix-2}.} We firstly notice that for a set of $n$ vectors $\bm{x}_1,\bm{x}_2,\ldots,\bm{x}_n$, it always holds true that
\begin{align*}
\left\|\sum_{i=1}^n \bm{x}_i\right\|_2^2 \leq n \sum_{i=1}^n \|\bm{x}_i\|_2^2.
\end{align*}
Therefore, \eqref{eq:theta-tmix-decompose} implies the following bound for $\mathbb{E}\|\bm{\theta}_t - \bm{\theta}_{t-t_{\mix}}\|_2^2$:
\begin{align*}
\mathbb{E}\|\bm{\theta}_t - \bm{\theta}_{t-t_{\mix}}\|_2^2 &\leq 2t_{\mix} \cdot \left\{\sum_{i=t-t_{\mix}}^{t-1} \eta_i^2 \mathbb{E}\|(\bm{A}_i\bm{\theta}^\star-\bm{b}_i)\|_2^2 + \sum_{i=t-t_{\mix}}^{t-1} \eta_i^2 \mathbb{E}\|\bm{A}_i \bm{\Delta}_i\|_2^2\right\} \\ 
&\leq 2t_{\mix}\cdot \left\{t_{\mix} \eta_{t-t_{\mix}}^2 (2\|\bm{\theta}^\star\|_2+1)^2 + 4 \eta_{t-t_{\mix}}^2 \sum_{i=t-t_{\mix}}^{t-1} \mathbb{E}\|\bm{\Delta}_i\|_2^2\right\}\\ 
&= 2t_{\mix}\eta_{t-t_{\mix}}^2\left[t_{\mix}(2\|\bm{\theta}^\star\|_2+1)^2 + 4 \sum_{i=t-t_{\mix}}^{t-1} \mathbb{E}\|\bm{\Delta}_i\|_2^2\right].
\end{align*}
\paragraph{Proof of Equation \eqref{eq:E-delta-tmix-3}.} By triangle inequality, \eqref{eq:theta-tmix-decompose} implies that
\begin{align*}
&\mathbb{E}[\|\bm{\Delta}_{t-t_{\mix}}\|_2 \|\bm{\theta}_{t} - \bm{\theta}_{t-t_{\mix}}\|_2] \\ 
&\leq \sum_{i=t-t_{\mix}}^{t-1} \eta_i \mathbb{E}[\|\bm{\Delta}_{t-t_{\mix}}\|_2 \|\bm{A}_i\bm{\theta}^\star-\bm{b}_i\|_2] + \sum_{i=t-t_{\mix}}^{t-1} \eta_i \mathbb{E}[\|\bm{\Delta}_{t-t_{\mix}}\|_2 \|\bm{A}_i\bm{\Delta}_i\|_2] \\ 
&\leq \sum_{i=t-t_{\mix}}^{t-1} \eta_i (2\|\bm{\theta}^\star\|_2+1)\mathbb{E}\|\bm{\Delta}_{t-t_{\mix}}\|_2 + \sum_{i=t-t_{\mix}}^{t-1} \eta_i \mathbb{E}[\frac{1}{2}\|\bm{\Delta}_{t-t_{\mix}}\|_2^2 + \frac{1}{2}\|\bm{A}_i\bm{\Delta}_i\|_2] \\ 
&\leq t_{\mix} \eta_{t-t_{\mix}}(2\|\bm{\theta}^\star\|_2+1)\mathbb{E}\|\bm{\Delta}_{t-t_{\mix}}\|_2 + \frac{1}{2}\eta_{t-t_{\mix}} \left(\mathbb{E}\|\bm{\Delta}_{t-t_{\mix}}\|_2^2 + 2\sum_{i=t-t_{\mix}}^{t-1} \|\bm{\Delta}_i\|_2^2\right).
\end{align*}
This completes the proof of the lemma.

\end{document}